\documentclass{article}

\usepackage{bbding}

\usepackage{hyperref}
\usepackage{url}

\usepackage{framed}
\usepackage{color}
\usepackage[utf8]{inputenc} 
\usepackage[T1]{fontenc}    
\usepackage{hyperref}       
\usepackage{url}            
\usepackage{booktabs}       
\usepackage{amsfonts}       
\usepackage{nicefrac}       
\usepackage{microtype}      
\usepackage{xcolor}         
\usepackage{color, colortbl}
\definecolor{greyC}{RGB}{180,180,180}
\definecolor{greyL}{RGB}{235,235,235}
\usepackage{multicol}
\usepackage{multirow}
\usepackage{makecell}

\usepackage{amsmath}
\usepackage{amssymb}
\usepackage{mathtools}
\usepackage{amsthm}
\usepackage{algorithmic,algorithm}

\usepackage{soul}
\usepackage{microtype}
\usepackage{graphicx}
\usepackage{booktabs}
\usepackage{caption}

\usepackage{threeparttable}
\usepackage{subfigure}
\usepackage{dsfont}
\usepackage{enumerate}
\usepackage{amsmath,amsthm,amssymb}
\usepackage{natbib}
\usepackage{multirow}

\usepackage{wrapfig}

\usepackage{framed}
\usepackage{color}
\definecolor{shadecolor}{rgb}{0.92,0.92,0.92}

\usepackage[toc,page,header]{appendix}
\usepackage{minitoc}

\renewcommand{\algorithmicrequire}{\textbf{Input:}}
\renewcommand{\algorithmicensure}{\textbf{Output:}}

\usepackage{hyperref}

\usepackage[accepted]{icml2023}

\usepackage{amsmath}
\usepackage{amssymb}
\usepackage{mathtools}
\usepackage{amsthm}

\usepackage[capitalize,noabbrev]{cleveref}

\theoremstyle{plain}
\newtheorem{theorem}{Theorem}[section]

\newtheorem{lemma}[theorem]{Lemma}

\theoremstyle{definition}

\newtheorem{assumption}[theorem]{Assumption}
\theoremstyle{remark}

\usepackage[textsize=tiny]{todonotes}

\icmltitlerunning{Unleashing Mask: Explore the Intrinsic Out-of-Distribution Detection Capability}

\begin{document}

\twocolumn[
\icmltitle{Unleashing Mask: Explore the Intrinsic Out-of-Distribution Detection Capability}




\begin{icmlauthorlist}
\icmlauthor{Jianing Zhu}{hkbu}
\icmlauthor{Hengzhuang Li}{hkbu}
\icmlauthor{Jiangchao Yao}{sjtu,lab}
\icmlauthor{Tongliang Liu}{syd,sydc}
\icmlauthor{Jianliang Xu}{hkbu}
\icmlauthor{Bo Han}{hkbu}
\end{icmlauthorlist}

\icmlaffiliation{hkbu}{Department of Computer Science, Hong Kong Baptist University}
\icmlaffiliation{sjtu}{CMIC, Shanghai Jiao Tong University}
\icmlaffiliation{lab}{Shanghai AI Laboratory}
\icmlaffiliation{syd}{Mohamed bin Zayed University of Artificial Intelligence}
\icmlaffiliation{sydc}{Sydney AI Centre, The University of Sydney}

\icmlcorrespondingauthor{Bo Han}{bhanml@comp.hkbu.edu.hk}
\icmlcorrespondingauthor{Jiangchao Yao}{Sunarker@sjtu.edu.cn}

\icmlkeywords{Machine Learning, ICML}

\vskip 0.3in
]



\printAffiliationsAndNotice{}  

\begin{abstract}
Out-of-distribution (OOD) detection is an indispensable aspect of secure AI when deploying machine learning models in real-world applications. Previous paradigms either explore better scoring functions or utilize the knowledge of outliers to equip the models with the ability of OOD detection. However, few of them pay attention to the intrinsic OOD detection capability of the given model. In this work, we generally discover the existence of an intermediate stage of a model trained on in-distribution (ID) data having higher OOD detection performance than that of its final stage across different settings, and further identify one critical data-level attribution to be learning with the \textit{atypical samples}. Based on such insights, we propose a novel method, \textit{Unleashing Mask}, which aims to restore the OOD discriminative capabilities of the well-trained model with ID data. Our method utilizes a mask to figure out the memorized \textit{atypical samples}, and then finetune the model or prune it with the introduced mask to forget them. Extensive experiments and analysis demonstrate the effectiveness of our method. The code is available at: \url{https://github.com/tmlr-group/Unleashing-Mask}.
\end{abstract}

\section{Introduction}


Out-of-distribution (OOD) detection has drawn increasing attention when deploying machine learning models into the open-world scenarios~\citep{Nguyen_2015_CVPR,LeeLLS18, yang2021generalized}. Since the test samples can naturally arise from a label-different distribution, identifying OOD inputs from in-distribution (ID) data is important, especially for those safety-critical applications like autonomous driving and medical intelligence. Previous studies focus on designing a series of scoring functions~\citep{hendrycks17baseline,LiangLS18,liu2020energy,SunM0L22} for OOD uncertainty estimation or fine-tuning with auxiliary outlier data to better distinguish the OOD inputs~\citep{hendrycks2018deep,MohseniPYW20,SehwagCM21}.


Despite the promising results achieved by previous methods~\citep{hendrycks17baseline,hendrycks2018deep,liu2020energy,ming2022poem}, limited attention is paid to considering whether the given well-trained model is the most appropriate basis for OOD detection. In general, models deployed for various applications have different original targets (e.g., multi-class classification~\citep{goodfellow2016deep}) instead of OOD detection~\citep{Nguyen_2015_CVPR}. However, most representative score functions, e.g., MSP~\citep{hendrycks2018deep}, ODIN~\citep{LiangLS18}, and Energy~\citep{liu2020energy}, uniformly leverage the given models for OOD detection~\citep{yang2021generalized}. The above target-oriented discrepancy naturally motivates the following critical question: \textit{does the given well-trained model have the optimal OOD discriminative capability?} If not, \textit{how can we find a more appropriate counterpart for OOD detection?}


In this work, we start by revealing an interesting empirical observation, i.e., there always exists a historical training stage where the model has a higher OOD detection performance than the final well-trained one (as shown in Figure~\ref{fig: motivation_1}), spanning among different OOD/ID datasets~\citep{netzer2011reading_SVHN,van2018inaturalist} under different learning rate schedules~\citep{LoshchilovH17} and model structures~\citep{huang2017densely,zagoruyko2016wide}. It shows the inconsistency between gaining better OOD discriminative capability~\citep{Nguyen_2015_CVPR} and pursuing better performance on ID data during training.
Through the in-depth analysis from various perspectives (as illustrated in Figure 2), we figure out one possible attribution at the data level is memorizing the \textit{atypical samples} (compared with others at the semantic level) that are hard to generalize for the model. Seeking zero training error on those samples leads the model more confident in the unseen OOD inputs.


\begin{figure*}[t!]
\begin{center}
    \hspace{-0.10in}
    \subfigure[Curves of FPR95 based on Energy score]{
    \includegraphics[scale=0.17]{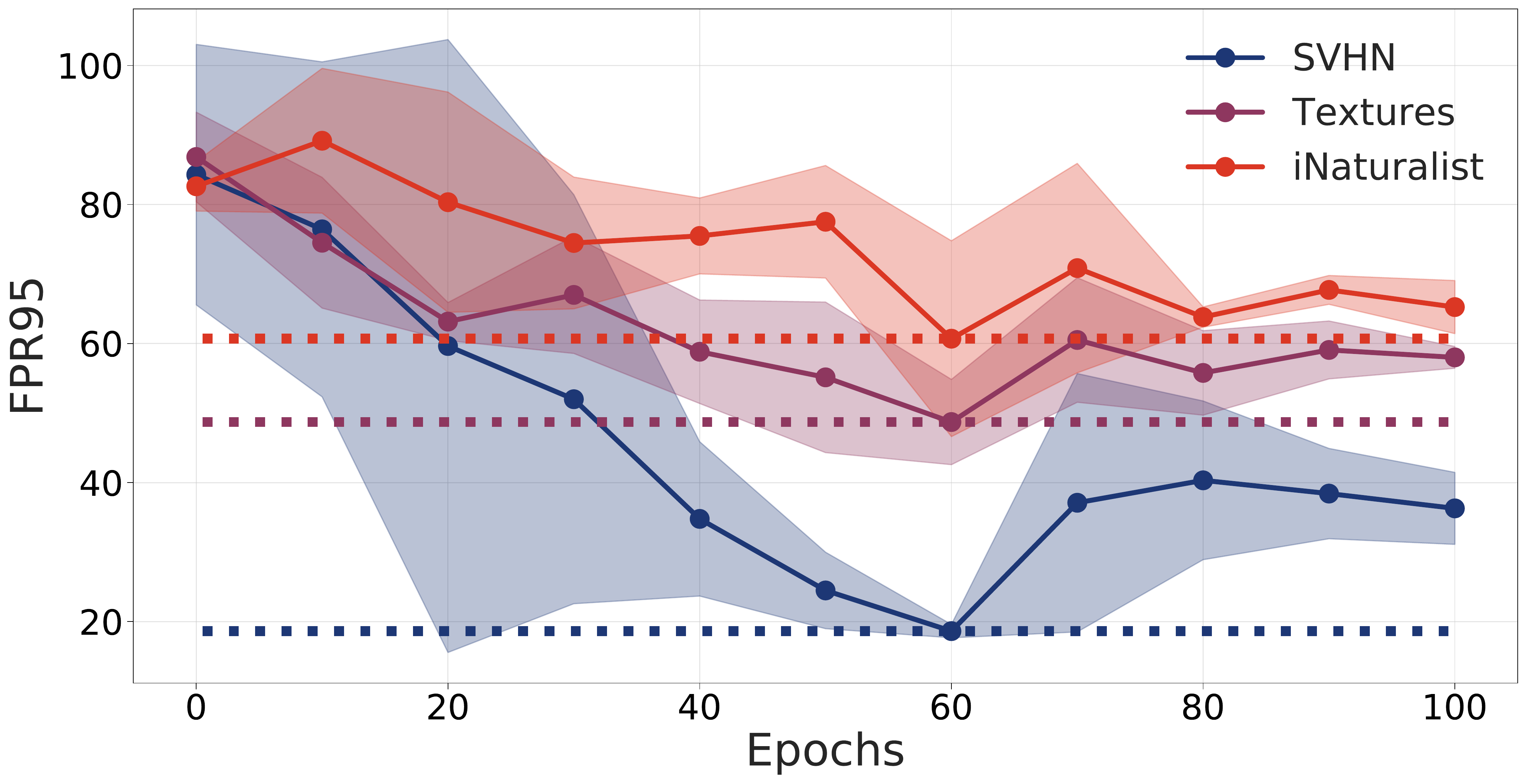}
    \label{fig1:a}
    }
   \hspace{0.15in}
    \subfigure[Diff. LR Schedules]{
    \includegraphics[scale=0.17]{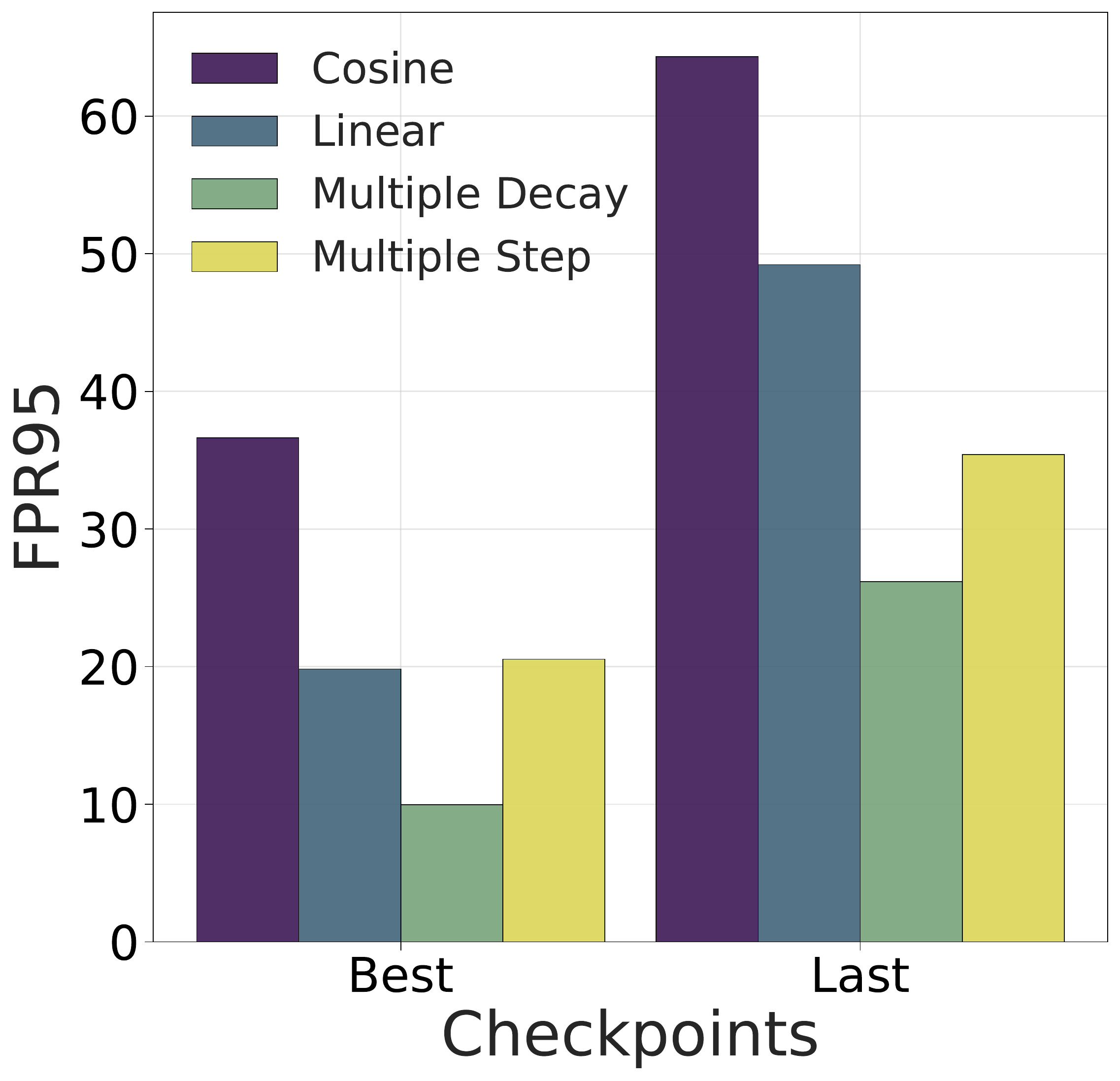}
    \label{fig1:b}
    }
    \subfigure[Diff. Model Structures]{
    \includegraphics[scale=0.17]{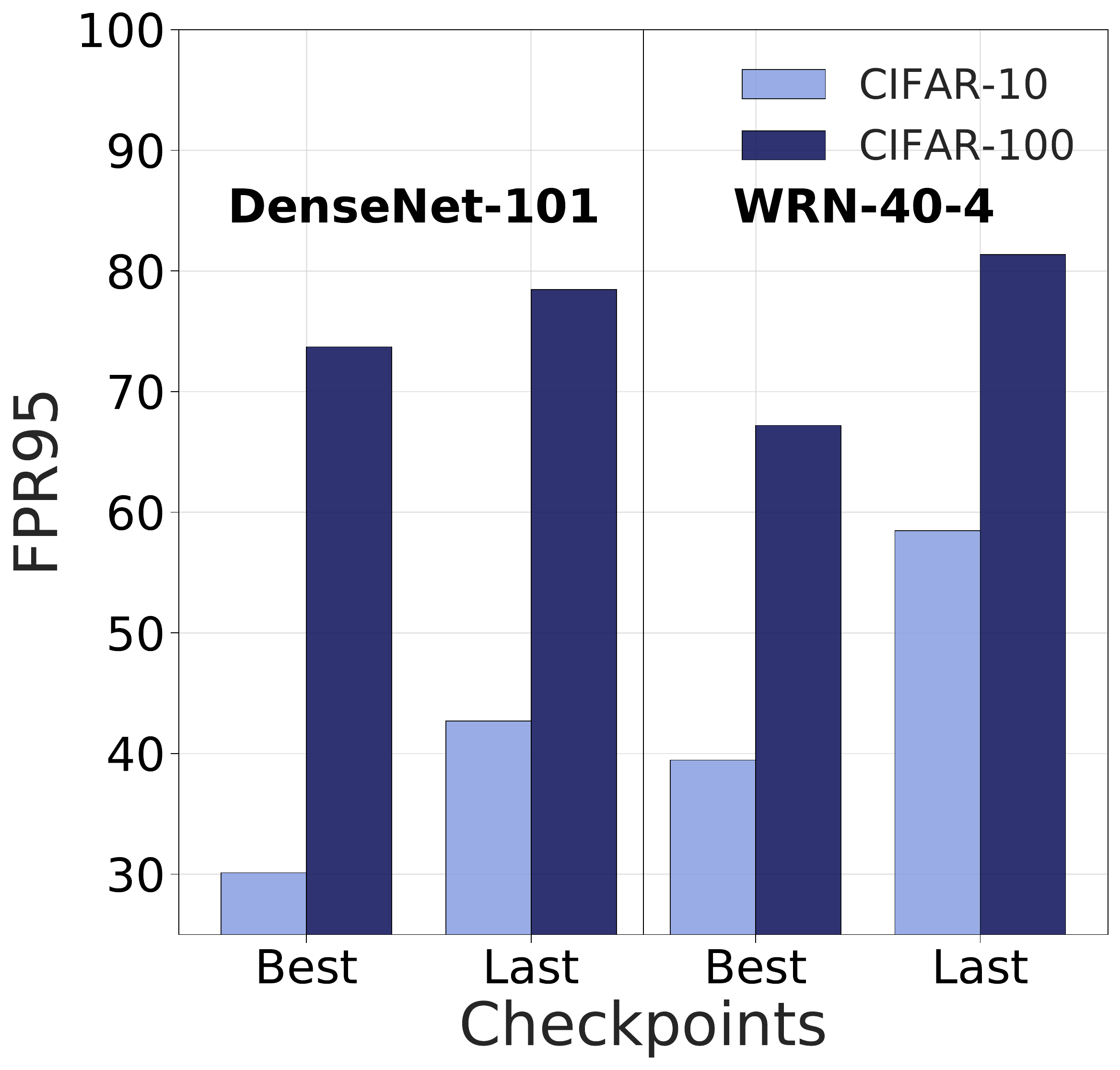}
    \label{fig1:c}
    }
\end{center}
\vspace{-5mm}
\caption{\textbf{Critical reveal of the intermediate stage with better OOD detection performance across various setups:} (a) the curves of FPR95 (false positive rate of OOD examples when the true positive rate of ID examples is at 95\%) based on Energy score~\citep{liu2020energy} across three different OOD datasets during the training on CIFAR-10 dataset; (b) comparison of best/last checkpoints for OOD detection under different lr schedules on CIFAR-10; (c) comparison of best/last checkpoints for OOD detection under different model structures on CIFAR-10/CIFAR-100. 
The results are obtained after multiple runs, and we leave other setup details to Section~\ref{sec:exp_part1} and Appendix~\ref{app:additional_exp_setup}.
}
\label{fig: motivation_1}
\vspace{-4mm}
\end{figure*}

The above analysis inspires us to propose a new method, namely, \textit{Unleashing Mask} (UM), to excavate the overlaid detection capability of a well-trained given model by alleviating the memorization of those atypical samples (as illustrated in Figure~\ref{fig:method}) of ID data. In general, we aim to backtrack its previous stage with better OOD discriminative capabilities. To achieve this target, there are two essential issues: (1) \textit{the model that is well-trained on ID data has already memorized some atypical samples}; (2) \textit{how to forget those memorized atypical samples considering the given model?} Accordingly, our proposed UM contains two parts utilizing different insights to address the two problems. First, as atypical samples are more sensitive to the change of model parameters, we initialize a mask with the specific cutting rate to mine these samples with constructed parameter discrepancy. Second, with the loss reference estimated by the mask, we conduct the constrained gradient ascent for model forgetting (i.e., Eq.~(\ref{eq:obj})). It will encourage the model to finally stabilize around the optimal stage. To avoid severe sacrifices of the original task performance on ID data, we further propose \textit{UM Adopts Pruning} (UMAP) which tunes on the introduced mask with the newly designed objective.


We conduct extensive experiments (in Section~\ref{sec:exp} and Appendixes~\ref{app:additional_exp_setup} to~\ref{app:eff_um}) to present the working mechanism of our proposed methods. We have verified the effectiveness with a series of OOD detection benchmarks mainly on two common ID datasets, i.e., CIFAR-10 and CIFAR-100. Under the various evaluations, our UM, as well as UMAP, can indeed excavate the better OOD discriminative capability of the well-trained given models and the averaged FPR95 can be reduced by a significant margin. Finally, a range of ablation studies, verification on the ImageNet pretrained model, and further discussions from both empirical and theoretical views are provided. Our main contributions are as follows,
\begin{itemize}
    \item Conceptually, we explore the OOD detection performance via a new perspective, i.e., backtracking the initial model training phase without regularizing by any auxiliary outliers, different from most previous works that start with the well-trained model on ID data.
    \item Empirically, we reveal the potential OOD discriminative capability of the well-trained model, and figure out one data-level attribution of concealing it during original training is memorizing the atypical samples.
    \item Technically, we propose a novel \textit{Unleashing Mask} (UM) and its practical variant UMAP, which utilizes the newly designed forgetting objective with ID data to excavate the intrinsic OOD detection capability. 
    \item Experimentally, we conduct extensive explorations to verify the overall effectiveness of our method in improving OOD detection performance, and perform various ablations to provide a thorough understanding. 
\end{itemize}



\section{Preliminaries}

We consider multi-class classification as the original training task~\citep{Nguyen_2015_CVPR}, where $\mathcal{X}\subset\mathbb{R}^d$ denotes the input space and $\mathcal{Y}=\{1,\ldots, C\}$ denotes the label space. In practical, a reliable classifier should be able to figure out the OOD input, which can be considered as a binary classification problem. Given $\mathcal{P}$, the distribution over $\mathcal{X}\times\mathcal{Y}$, we consider $\mathcal{D}_\text{in}$ as the marginal distribution of $\mathcal{P}$ for $\mathcal{X}$, namely, the distribution of ID data. At test time, the environment can present a distribution $\mathcal{D}_\text{out}$ over $\mathcal{X}$ of OOD data. In general, the OOD distribution $\mathcal{D}_\text{out}$ is defined as an irrelevant distribution of which the label set has no intersection with $\mathcal{Y}$~\cite{yang2021generalized} and thus should not be predicted by the model. A decision can be made with the threshold $\lambda$:
\begin{equation}
    D_{\lambda}(x;f)=\left \{
    \begin{aligned}
    &\text{ID} && S(x)\geq\lambda\\
    &\text{OOD} && S(x)<\lambda
    \end{aligned},
    \right.
\end{equation}
Building upon the model ${f}\in\mathcal{H}:\mathcal{X}\rightarrow\mathbb{R}^c$ trained on ID data with the logit outputs, the goal of decision is to utilize the scoring function $S:\mathcal{X}\rightarrow \mathbb{R}$ to distinguish the inputs of $\mathcal{D}_\text{in}$ from that of $\mathcal{D}_\text{out}$ by $S(x)$. Typically, if the score value is larger than the threshold $\lambda$, the associated input $x$ is classified as ID and vice versa. We consider several representative scoring functions designed for OOD detection, e.g., MSP~\citep{hendrycks17baseline}, ODIN~\citep{LiangLS18}, and Energy~\citep{liu2020energy}. More detailed definitions and implementation are provided in Appendix~\ref{app:baseline_info}.


To mitigate the issue of over-confident predictions for some OOD data~\citep{hendrycks17baseline,liu2020energy},  recent works~\citep{hendrycks2018deep, Tack20CSI} utilize the auxiliary unlabeled dataset to regularize the model behavior. Among them, one representative baseline is outlier exposure (OE)~\citep{hendrycks2018deep}. OE can further improve the detection performance by making the model $f(\cdot)$ finetuned from a surrogate OOD distribution $\mathcal{D}^\text{s}_\text{out}$, and its corresponding learning objective is defined as follows,
\begin{equation}
    {\mathcal{L}}_f=\mathbb{E}_{\mathcal{D}_\text{in}}\left[\ell_\text{CE}(f(x),y)\right]
     + \lambda \mathbb{E}_{\mathcal{D}^\text{s}_\text{out}}\left[\ell_\text{OE}(f(x))\right], \label{eq: oe}
\end{equation}
where $\lambda$ is the balancing parameter,  $\ell_\text{CE}(\cdot)$ is the Cross-Entropy (CE) loss, and $\ell_\text{OE}(\cdot)$ is the Kullback-Leibler divergence to the uniform distribution, which can be written as $\ell_\text{OE}(h(\boldsymbol{x}))=-\sum_k \texttt{softmax}_k~f(x) / C$, where $\texttt{softmax}_k (\cdot)$ denotes the $k$-th element of a softmax output. The OE loss $\ell_\text{OE}(\cdot)$ is designed for model regularization, making the model learn from surrogate OOD inputs to return low-confident predictions~\citep{hendrycks2018deep}. 

Although previous works show promising results via designing scoring functions or regularizing models with different auxiliary outlier data, few of them investigated or excavated the original discriminative capability of the well-trained model using ID data. In this work, we introduce the layer-wise mask $m$~\citep{han2015deep,ramanujan2020s} to mine the atypical samples that are memorized by the model. Accordingly, the decision can be rewritten as $D(x;m\odot f)$, and the output of a masked model is defined as $m\odot f(x)$. 








\begin{figure*}[t!]
\begin{center}
    \subfigure[Training/Testing Loss and Accuracy]{
    \includegraphics[scale=0.124]{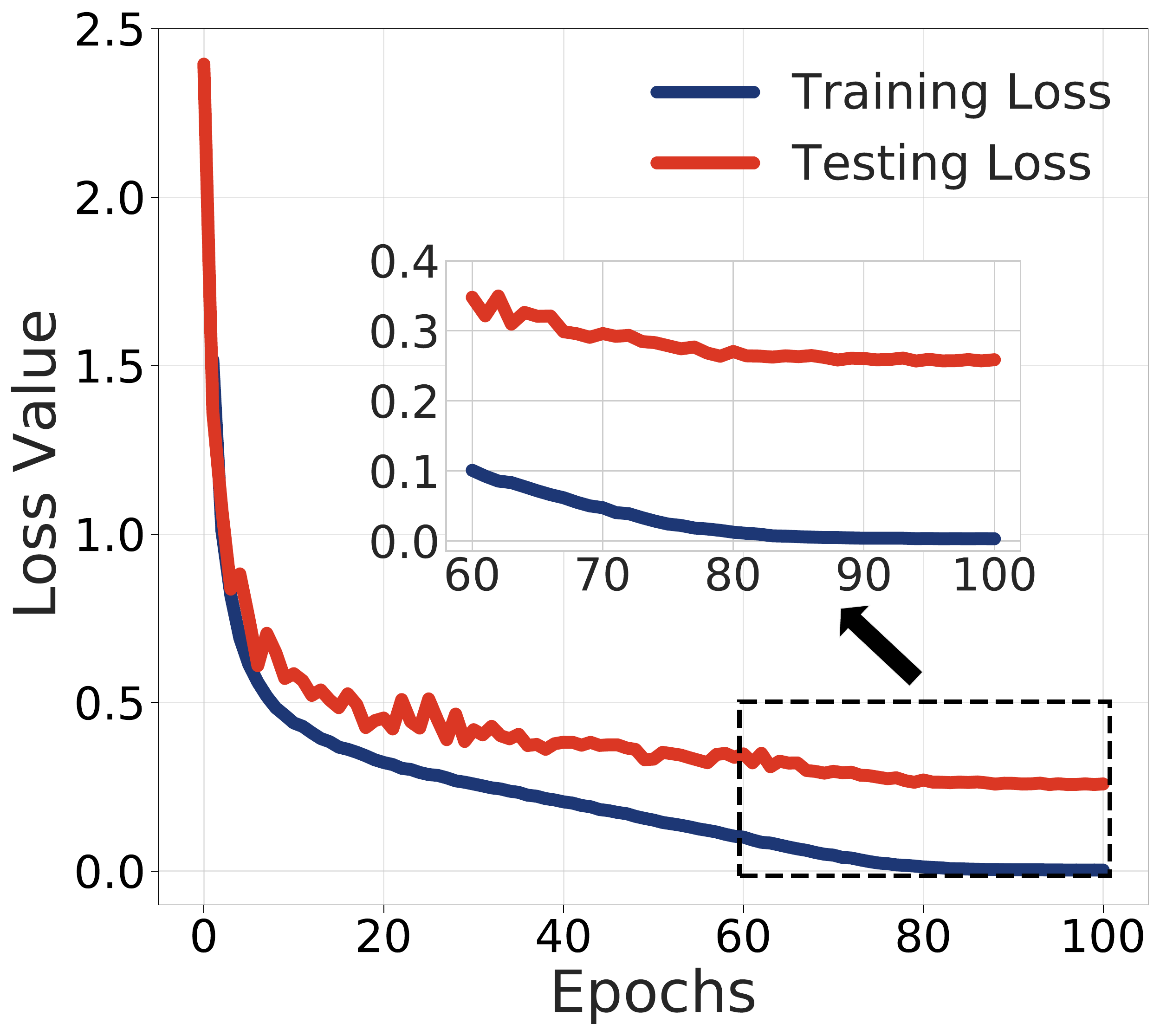}
    \includegraphics[scale=0.124]{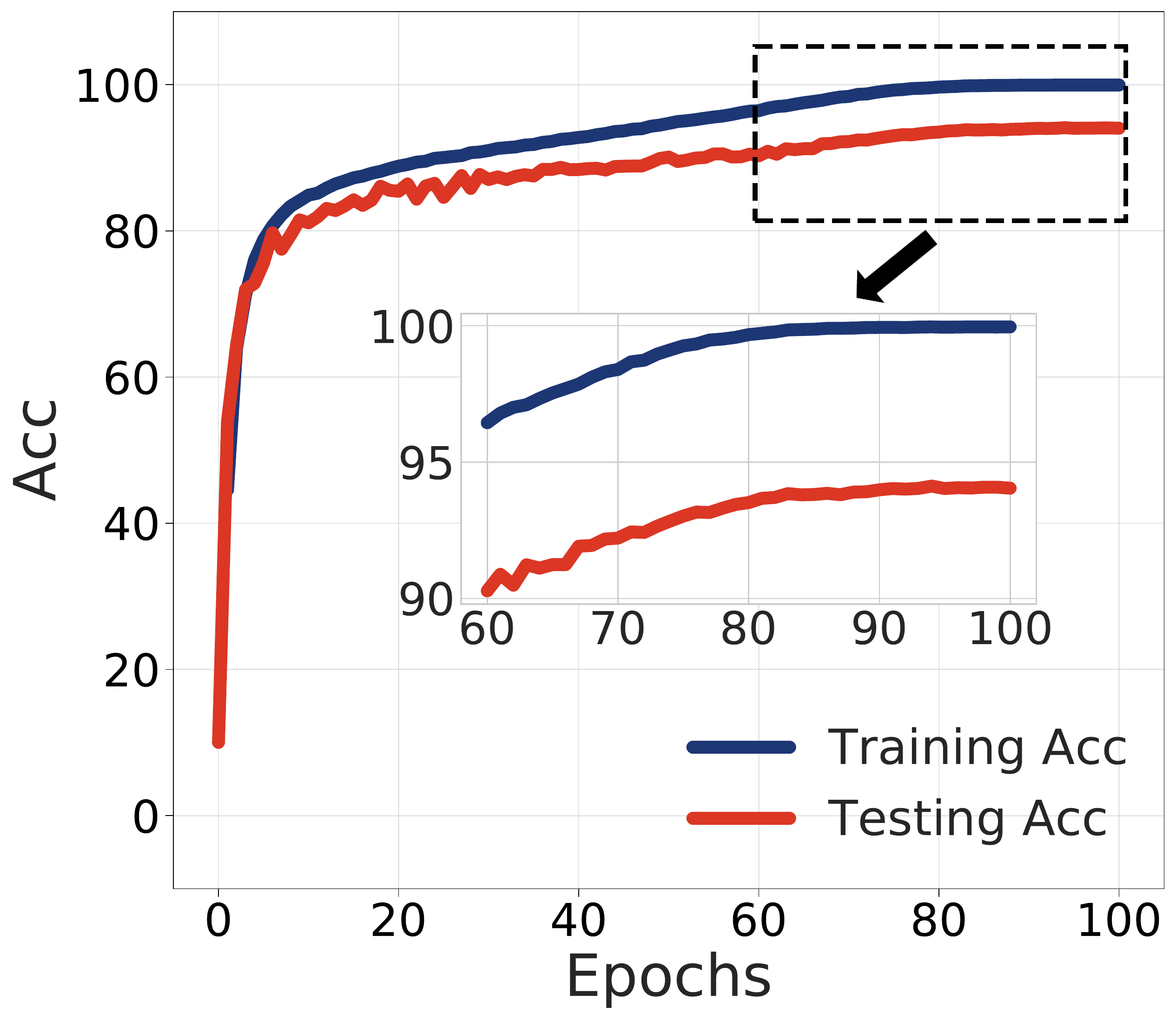}
    \label{fig2:a}
    }
    \subfigure[ID and OOD Distributions at Epoch 60/100]{
    \includegraphics[scale=0.124]{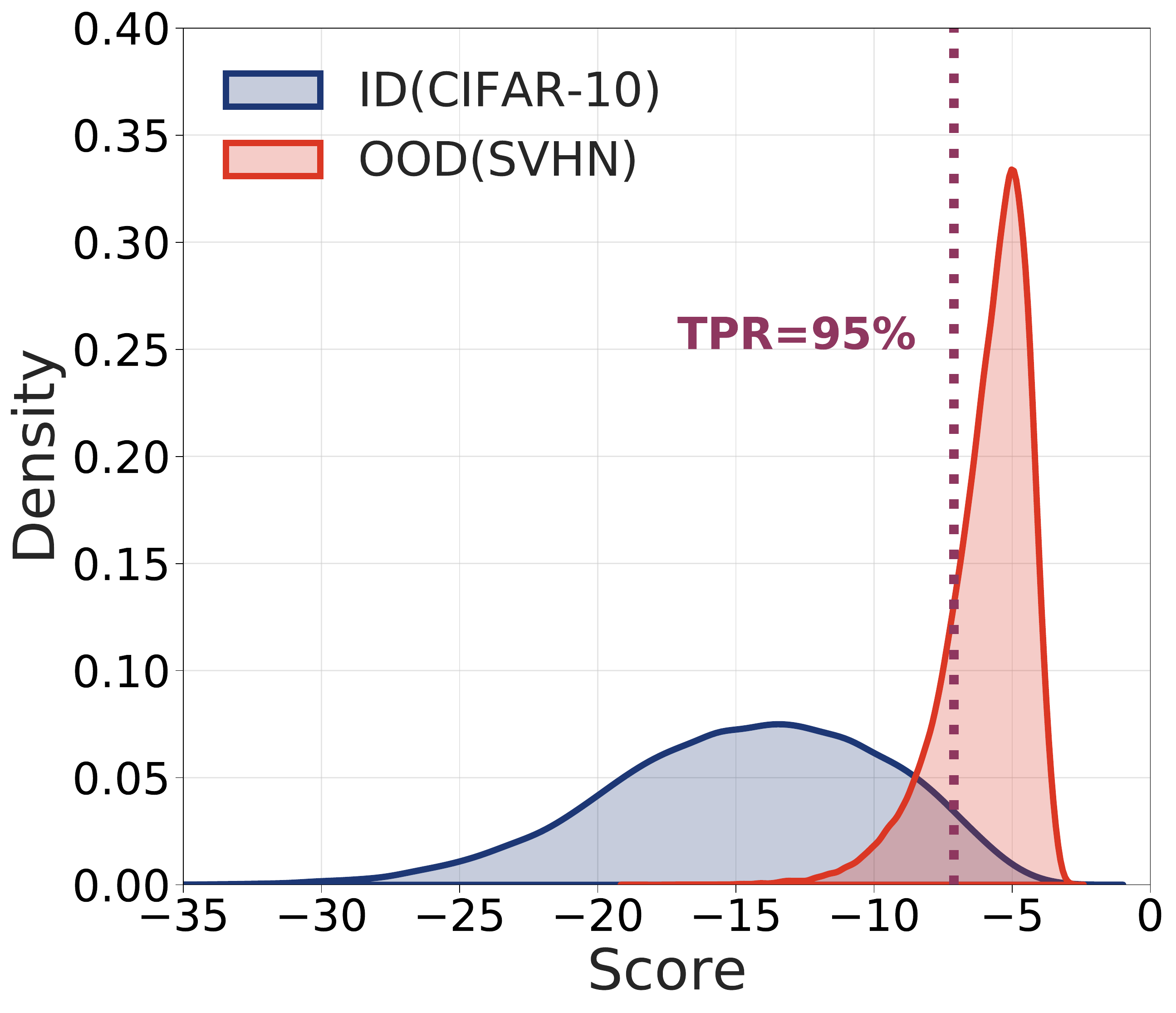}
    \includegraphics[scale=0.124]{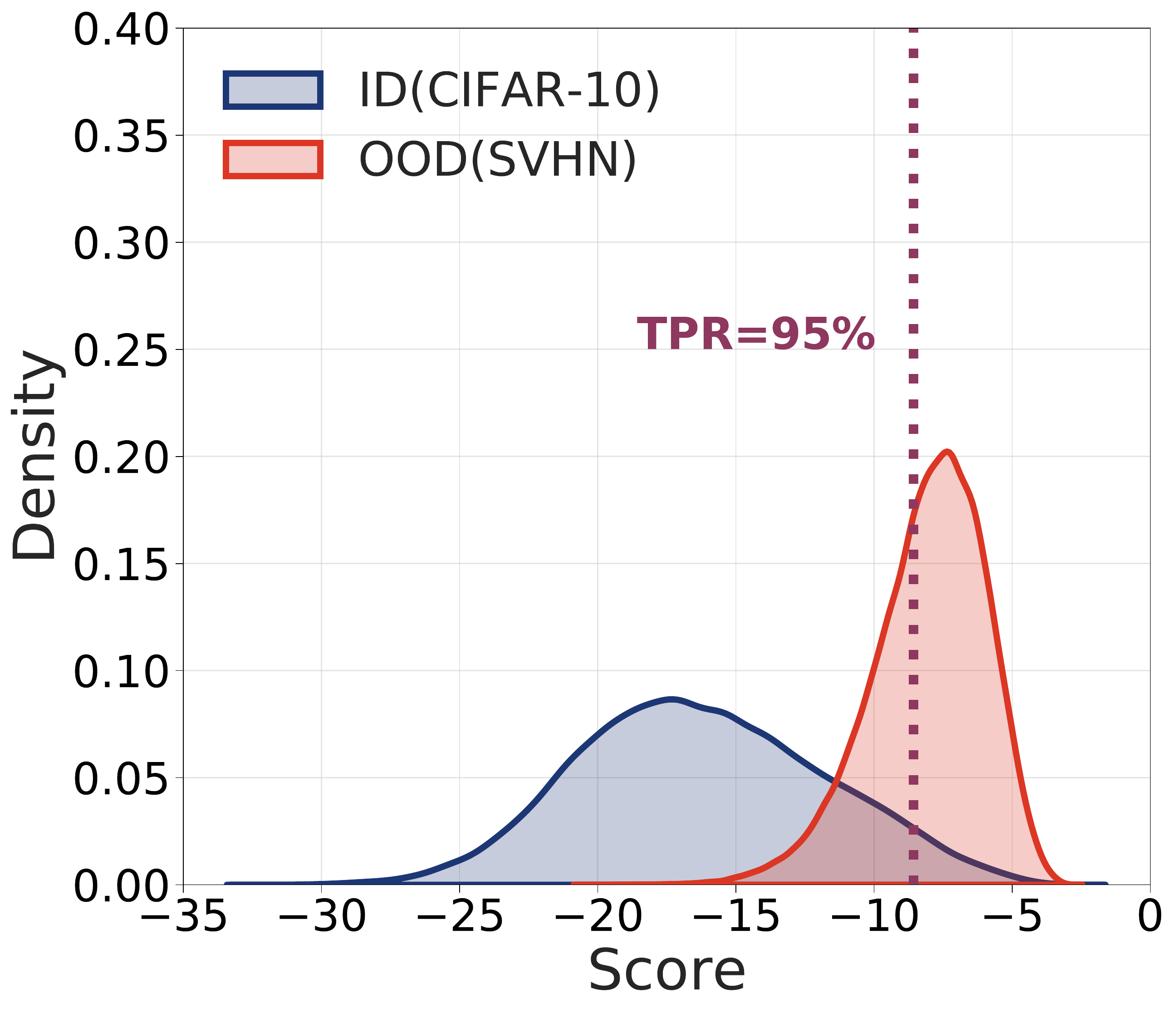}
    \label{fig2:b}
    }
    \subfigure[ID Distributions]{
    \includegraphics[scale=0.124]{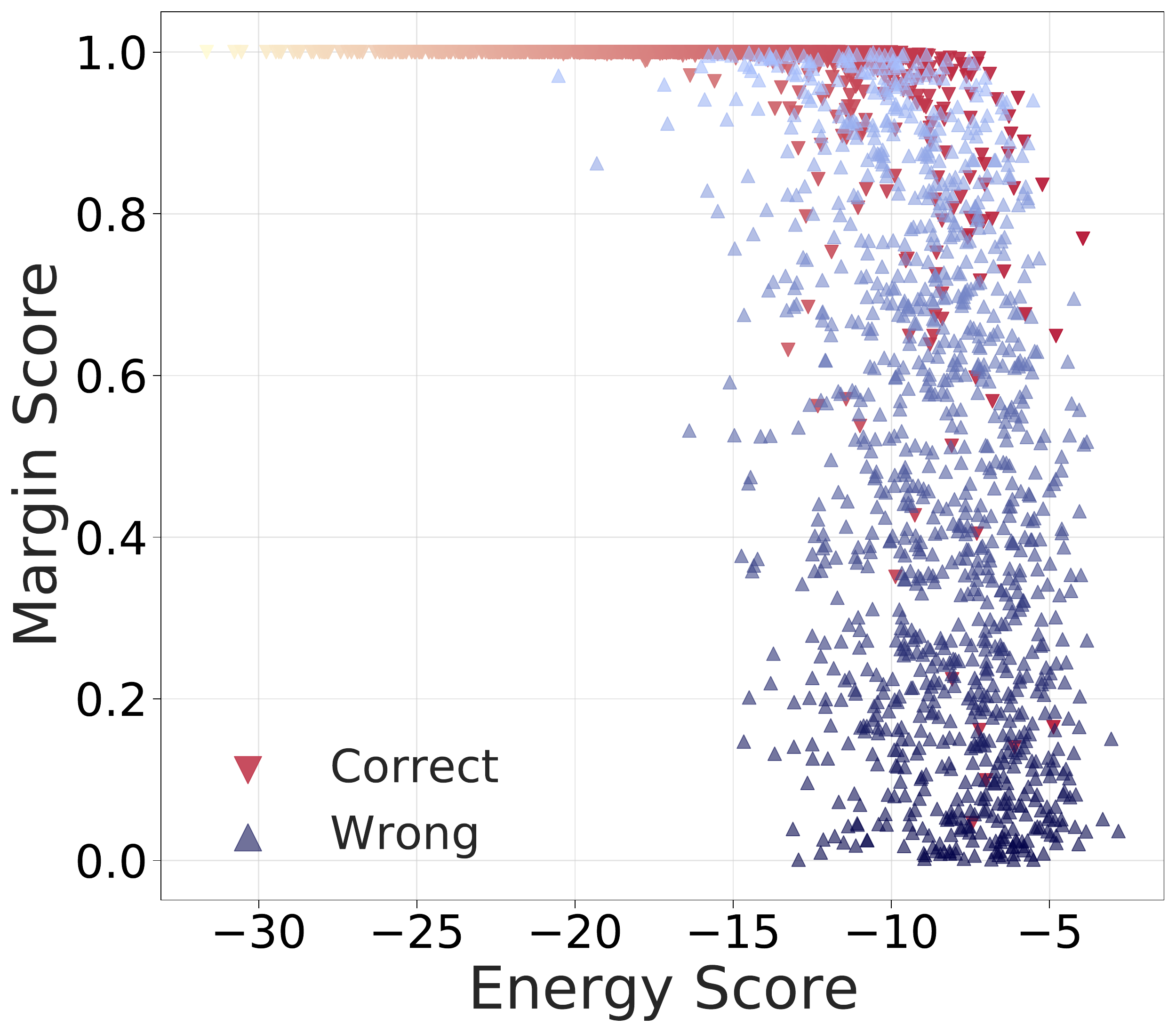}
    \label{fig2:c}
    }
    \\
    \subfigure[Wrongly/Correctly Classified Data at Epoch 60]{
    \includegraphics[scale=0.07]{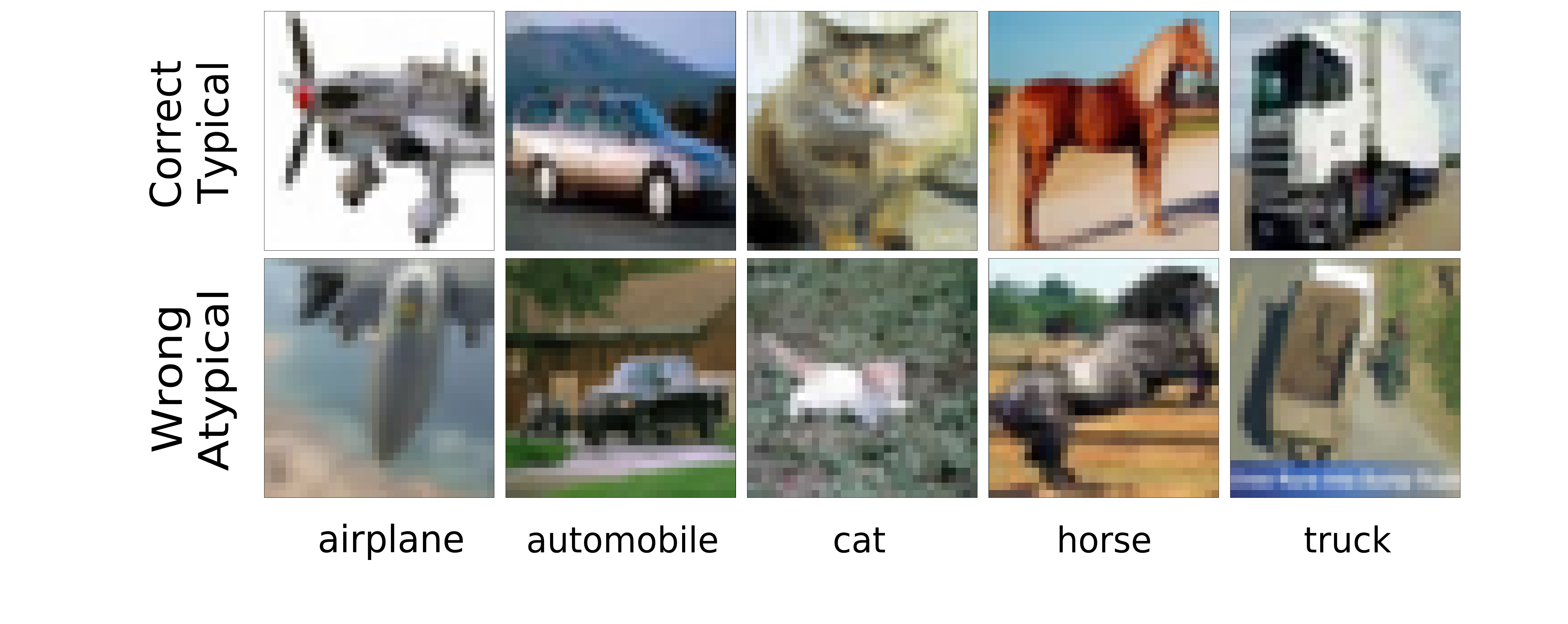}
    \label{fig2:d}
    }
    \subfigure[TSNE Visualization at Epoch 60/100]{
    \includegraphics[scale=0.1512]{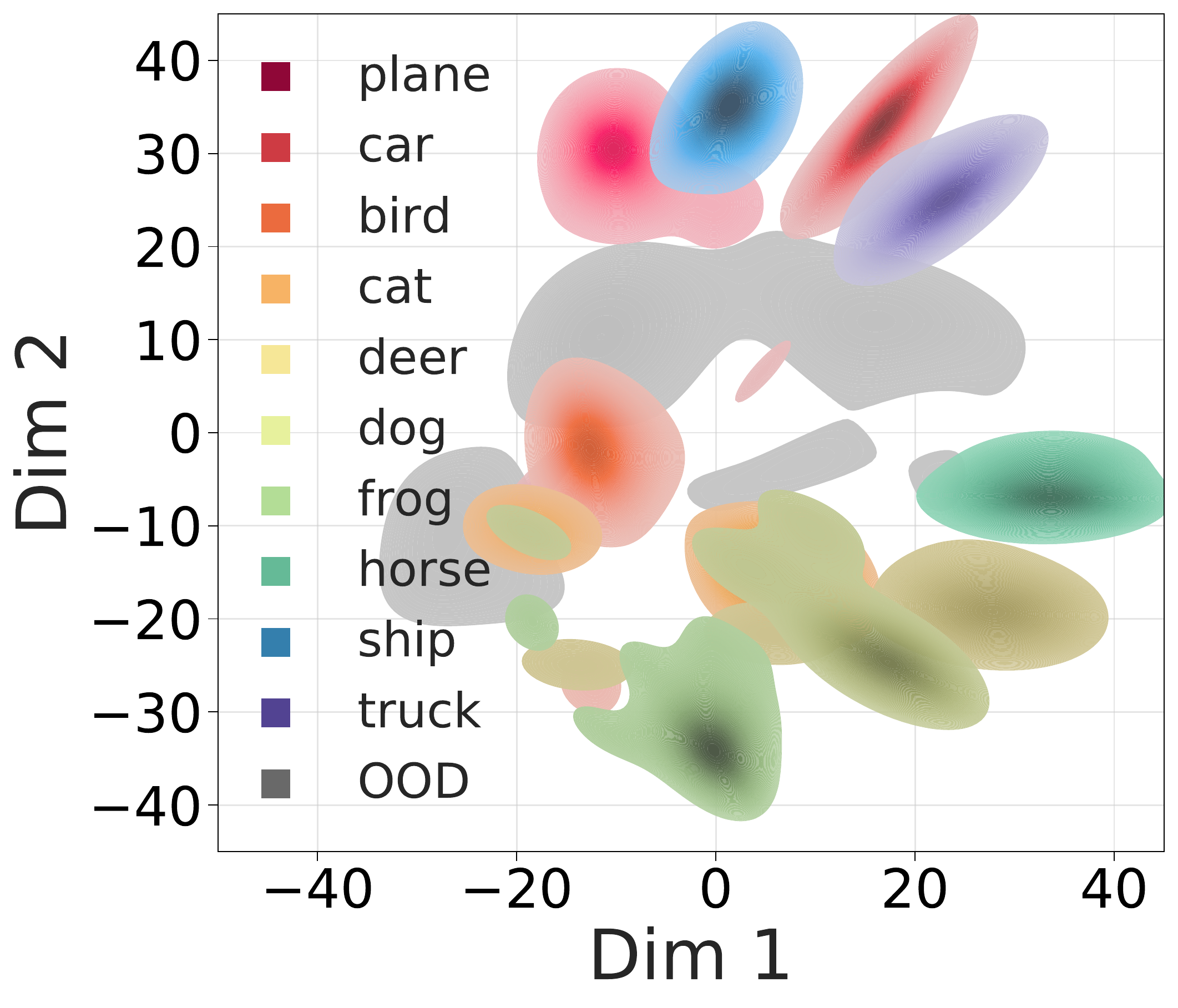}
    \includegraphics[scale=0.1512]{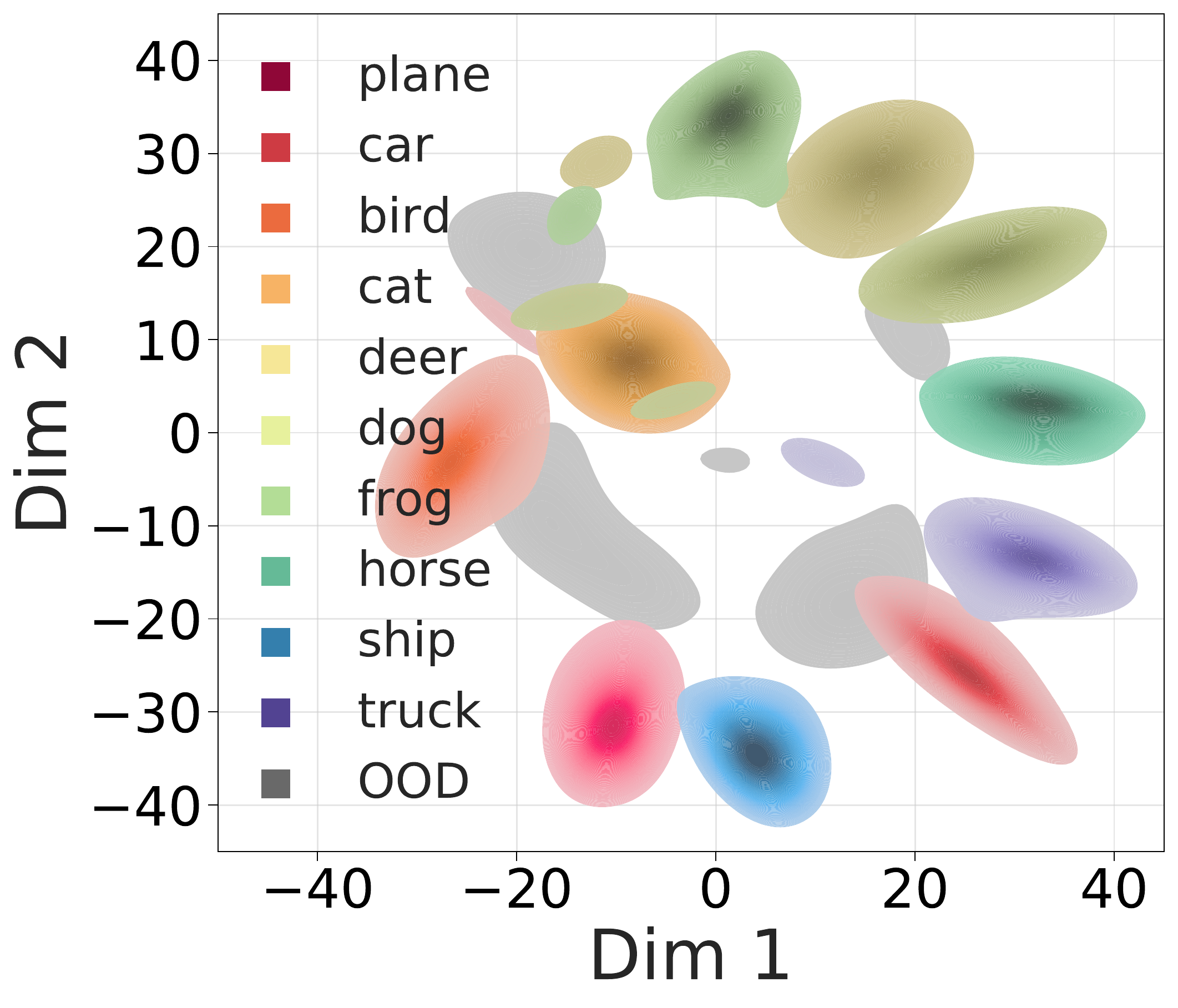}
    \label{fig2:e}
    }
    \subfigure[Diff. Effects]{
    \includegraphics[scale=0.124]{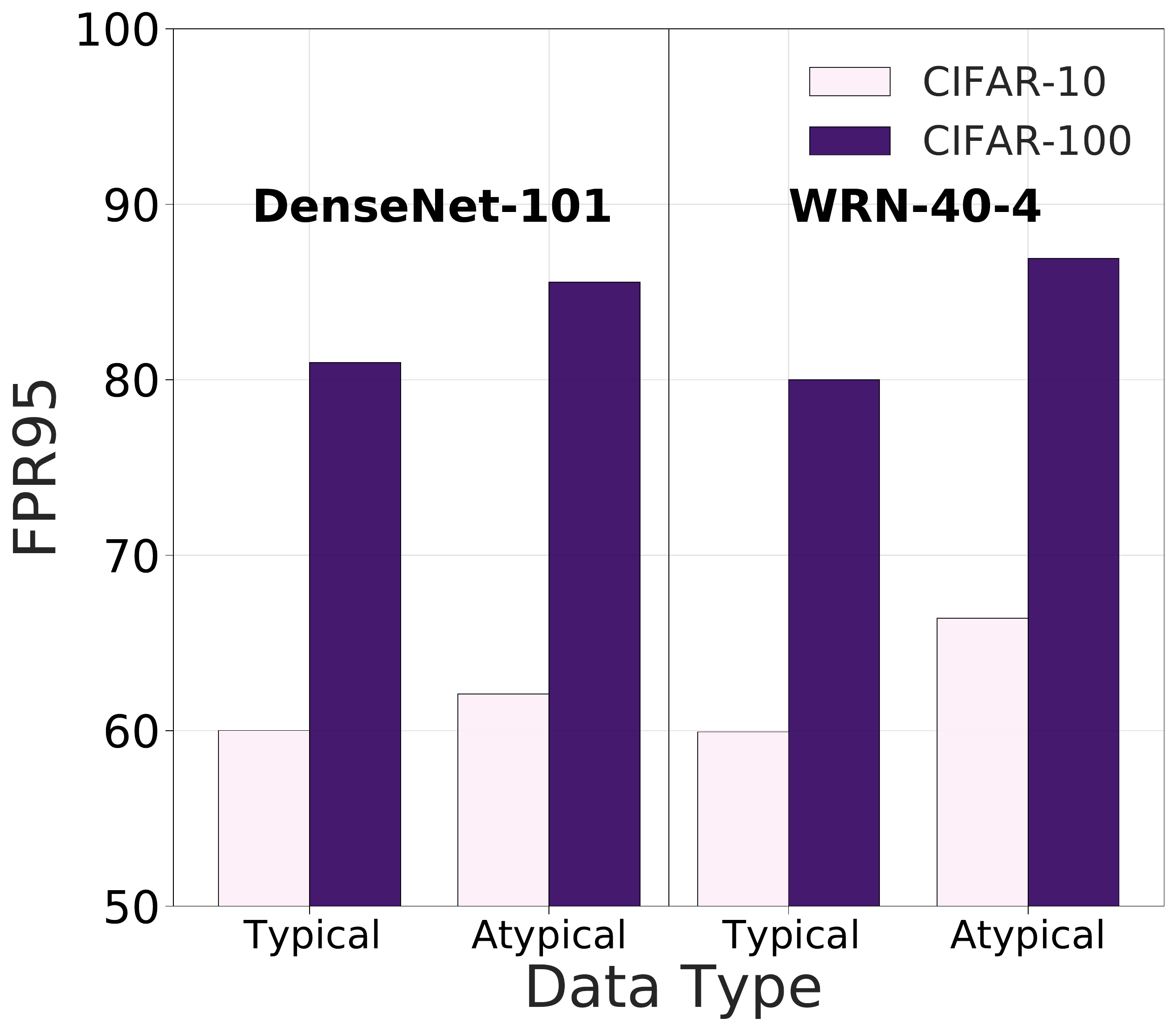}
    \label{fig2:f}
    }
\end{center}
\vspace{-4mm}
\caption{\textbf{Delve into the data-level attribution of the phenomenon with the original multi-classification on CIFAR-10:}  
(a) training/testing loss and accuracy on ID data; (b) comparison of ID/OOD distributions based on Energy score at Epoch 60/100 (c) scatter plot of wrongly/correctly classified samples at Epoch 60 using Margin value and Energy score (d) visualization of wrongly/correctly classified samples at Epoch 60; (e) TSNE visualization of the feature embedding on ID/OOD data at Epoch 60/100. (f) Effects on OOD detection of tuning with those identified typical/atypical samples, more detailed setup, and results can be referred to in Appendix~\ref{app:exp_typical_atypical}.
Through comparison from various perspectives, we find that achieving a reasonably small loss value (at round Epoch 60) on ID data is enough for OOD detection. However, continually optimizing on those atypical samples (e.g., wrongly classified in (d)) may impair OOD detection. 
}
\label{fig: motivation_2}
\vspace{-4mm}
\end{figure*}

\section{Proposed Method: Unleashing Mask}

In this section, we introduce our new method, i.e., \textit{Unleashing Mask} (UM), to reveal the potential OOD discriminative capability of the well-trained model. First, we present and discuss the important observation that inspires our methods (Section~\ref{sec:method_part1}). Second, we provide the insights behind the two critical parts of our UM (Section~\ref{sec:method_part2}). Lastly, we introduce the overall framework and its learning objective, as well as a practical variant of UM, i.e., UMAP (Section~\ref{sec:method_part3}). 

\subsection{Overlaid OOD Detection Capability}
\label{sec:method_part1}

First, we present the phenomenon of the inconsistency between pursuing better OOD discriminative capability and smaller training errors during the original task. 
Empirically, as shown in Figure~\ref{fig: motivation_1}, we trace the OOD detection performance during the model training after multiple runs of the experiments. Across three different OOD datasets in Figure~\ref{fig1:a}, we can observe the existence of a better detection performance using the index of FPR95 metric based on the Energy~\citep{liu2020energy} score. The generality has also been demonstrated under different learning schedules, model structures, and ID datasets in Figures~\ref{fig1:b} and~\ref{fig1:c}. Without any auxiliary outliers, it motivates us to explore the underlying mechanism of the training process with ID data.

We further delve into the learning dynamics from various perspectives in Figure~\ref{fig: motivation_2}, and we reveal the critical data-level attribution for the OOD discriminative capability. In Figure~\ref{fig2:a}, we find that the training loss has reached a reasonably small value\footnote{Note that it is not the conventional overfitting~\citep{goodfellow2016deep} as the testing loss is still decreasing. In Section~\ref{sec:exp_part3} and Appendix~\ref{app:overfitting_comp}, we provide both the empirical comparison of some targeted strategies and the conceptual comparison of them.} at Epoch 60 where its detection performance achieves a satisfactory level. However, if we further minimize the training loss, the trend of the FPR95 curve shows almost the opposite direction with both training and testing loss or accuracy (see Figures~\ref{fig1:a} and~\ref{fig2:a}). The comparison of the ID/OOD distributions is presented in Figure~\ref{fig2:b}. To be specific, the statics of the two distributions indicate that the gap between the ID and OOD data gets narrow as their overlap grows along with the training. After Epoch 60, although the model becomes more confident on ID data which satisfies a part of the calibration target~\citep{hendrycks2019using}, its predictions on the OOD data also become more confident which is unexpected. Using the margin value defined in logit space (see Eq.~(\ref{eq:margin})), we gather the statical with Energy score in Figure~\ref{fig2:c}. The misclassified samples are found to be close to the decision boundary and have a high uncertainty level in model prediction. Accordingly, we extract those samples that were learned by the model at this period. As shown in Figures~\ref{fig2:d},~\ref{fig2:e} and~\ref{fig2:f}, the misclassified samples learned after Epoch 60 present much atypical semantic features, which results in more diverse feature embedding and may impair OOD detection. As deep neural networks tend to first learn the data with typical features~\citep{arpit2017closer}, we attribute the inconsistent trend to memorizing those atypical data at the later stage. 


\subsection{Unleashing the Potential Discriminative Power}
\label{sec:method_part2}

In general, the models that are developed for the original classification tasks are always seeking better performance (e.g., higher testing accuracy and lower training loss) in practice. However, the inconsistent trend revealed before provides us the possibility to unleash the potential detection power only considering the ID data in training. To this end, we have two important issues that need to address: (1) \textit{the well-trained model may have already memorized some atypical samples which cannot be figured out}; (2) \textit{how to forget those atypical samples considering the given model?} 

\paragraph{Atypical mining with constructed discrepancy.}  
As shown Figures~\ref{fig2:a} and~\ref{fig2:b}, the training statics provide limited information to accurately differentiate the stage that learns on typical or atypical data. We thus explore to construct the parameter discrepancy to mine the atypical samples from a well-trained given model in the light of the learning dynamics~\citep{goodfellow2016deep, arpit2017closer} of deep neural networks and the model uncertainty representation~\citep{gal2016dropout}. Specifically, we employ a randomly initialized layer-wise mask which applied to all layers. It is consistent with the mask generation in the conventional pruning pipeline~\citep{han2015deep}. In Figure~\ref{fig:method}, we provide empirical evidence to show that we can figure out atypical samples by a certain mask ratio $\delta$, through which we can gradually mine the model stage that misclassifies atypical samples. We provide more discussion about the underlying intuition of masking in Appendix~\ref{app:atypical_mining}.

\paragraph{Model forgetting with gradient ascent.}
As the training loss achieves zero at the final stage of the given model, we need extra optimization signals to forget those memorized atypical samples. Considering the previous consistent trend before the potential optimal stage (e.g., before Epoch 60 in Figure~\ref{fig1:a}), the optimization signal also needs to control the model update not to be too greedy to drop the discriminative features that can be utilized for OOD detection. Starting with the well-trained given model, we can employ the gradient ascent~\citep{sorg2010reward, ishida2020we} to forget the targeted samples, while the tuning phase should also prevent further updates if it achieves the expected stage. As for another implementation choice, e.g., retraining the model from scratch for our targets, we discuss it in Appendix~\ref{app:exp_less_epochs}.

\begin{figure*}[t!]
\begin{center}
    \hspace{-0.20in}
    \includegraphics[scale=0.15]{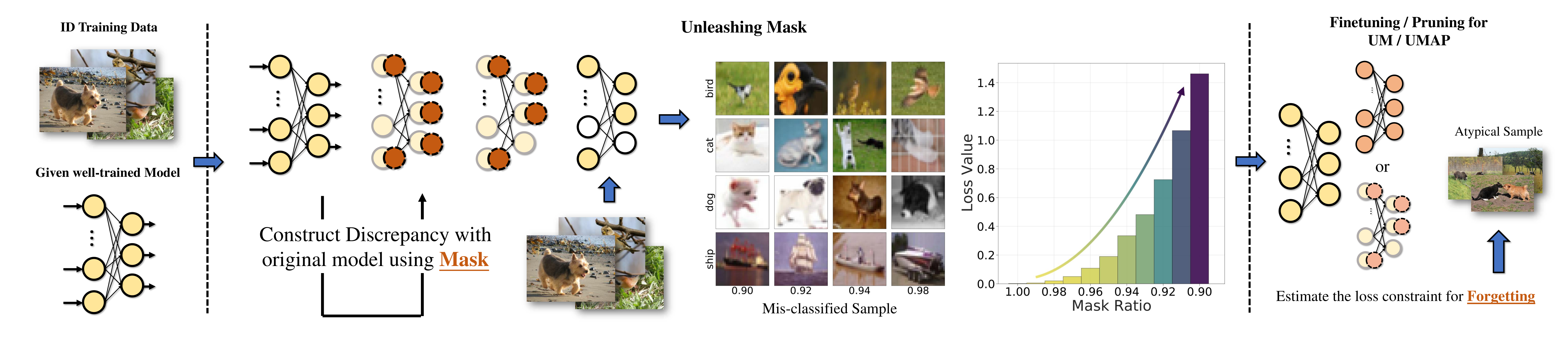}
\end{center}
\vspace{-3mm}
\caption{\textbf{Overview of Unleashing Mask:} Given a well-trained model, we initialize a mask for mining the atypical samples that are sensitive to the changes in model parameters. Then we finetune the original model or adopt pruning with the estimated forgetting threshold, i.e., the loss value estimated by the UM. The final model can serve as the base of those representative score functions to utilize the discriminative features and also as a new initialization of outlier exposure methods. In addition, we also present examples of misclassified samples in ID data after masking the original well-trained model, and loss value using the masked outputs w.r.t. different mask ratios.
}
\label{fig:method}
\vspace{-4mm}
\end{figure*}

\subsection{Method Realization}
\label{sec:method_part3}

Based on previous insights, we present our overall framework and the learning objective of the proposed UM and UMAP for OOD detection. Lastly, we discuss their compatibility with either the fundamental scoring functions or the outlier exposure approaches utilizing auxiliary outliers.

\paragraph{Framework.} As illustrated in Figure~\ref{fig:method}, our framework consists of two critical components for uncovering the intrinsic OOD detection capability: (1) the initialized mask with a specific masking rate for constructing the output discrepancy with the original model; (2) the subsequent adjustment for alleviating the memorization of atypical samples. The overall workflow starts with estimating the loss value of misclassifying those atypical samples and then conducts tuning on the model or the masked output to forget them.
\paragraph{Forgetting via Unleashing Mask (UM).} Based on previous insights, we introduce the forgetting objective as,
\begin{equation}
\label{eq:obj}
\begin{split}
\scriptsize
\vspace{-4mm}
    \min \mathcal{L}_{\text{UM}} 
    =  \min_{m_\delta\in[0,1]^n} |\ell_\text{CE}&(f) - \widehat{\ell}_\text{CE}(m_\delta\odot f^*)|\\
    &+\widehat{\ell}_\text{CE}(m_\delta\odot f^*),
\vspace{-4mm}
\end{split}
\end{equation}
where $m_\delta$ is the layer-wise mask with the masking rate $\delta$, $\ell_\text{CE}$ is the CE loss, $\widehat{\ell}_\text{CE}$ is the averaged CE loss over the ID training data, $|\cdot|$ indicates the computation for absolute value and $m_\delta \odot f^*$ denotes the masked output of the fixed pretrained model that is used to estimate the loss constraint for the learning objective of forgetting. The value of  $\widehat{\ell}_\text{CE}(\cdot)$ would be constant during the whole finetuning process. Concretely, the well-trained model will start to optimize itself again if it memorizes the atypical samples and achieves almost zero loss value. We provide a positive gradient signal when the current loss value is lower than the estimated one and vice versa. The model is expected to finally stabilize around the stage that can forget those atypical samples. To be more specific, for a mini-batch of ID samples, they are forwarded to the (pre-trained) model and the loss is automatically computed in Eq.~\eqref{eq:obj}. Based on our introduced layer-wise mask, the atypical samples would be easier to induce large loss values than the rest, and will be forced to be wrongly classified in the end-to-end optimization, in which atypical samples are forgotten without being identified.

\paragraph{Unleashing Mask Adopts Pruning (UMAP).}  Considering the potential negative effect on the original task performance when conducting tuning for forgetting, we further propose a variant of UM Adopts Pruning, i.e., UMAP, to conduct tuning based on the masked output (e.g., replace $\ell_\text{CE}(f)$ to $\ell_\text{CE}(\hat{m}_p \odot f)$ in Eq~\ref{eq:obj}) using a functionally different mask $\hat{m}_p$ with its pruning rate $p$ as follows,
\begin{equation}
\begin{split}
\scriptsize
\vspace{-4mm}
\min \mathcal{L}_{\text{UMAP}} 
=  \min_{\substack{\hat{m}_p\in[0,1]^n\\ m_\delta\in[0,1]^n}} |\ell_\text{CE}(&\hat{m}_p\odot f) - \widehat{\ell}_\text{CE}(m_\delta\odot f^*)|\\
&+\widehat{\ell}_\text{CE}(m_\delta\odot f^*),
\vspace{-4mm}
\end{split}
\end{equation}\label{eq:obj_umap}
Different from the objective of UM (i.e., Eq \ref{eq:obj}) that minimizes the loss value over the model parameter, the objective of UMAP minimizes the loss over the mask $\hat{m}_p$ to achieve the target of forgetting atypical samples. UMAP provides an extra mask to restore the detection capacity but doesn't affect the model parameter for the inference on original tasks, indicating that UMAP is a more practical choice in real-world applications (as empirically verified in our experiments like Table~\ref{tab:my_label}). We present the algorithms of UM (in Algorithm~\ref{alg:um}) and UMAP (in Algorithm~\ref{alg:umap}) in Appendix~\ref{app:algo_realization}.

\paragraph{Compatible with other methods.} As we explore the original OOD detection capability of the well-trained model, it is orthogonal and compatible with those promising methods that equip the given model with better detection ability. To be specific, through our proposed methods, we reveal the overlaid OOD detection capability by tuning the original model toward its intermediate training stage. 
The discriminative feature learned at that stage can be utilized by different scoring functions~\citep{huang2021importance,liu2020energy,sun2022dice}, like ODIN~\citep{LiangLS18} adopted in Figure~\ref{fig4:c}. For those methods~\citep{hendrycks2019using,liu2020energy, ming2022poem} utilizing the auxiliary outliers to regularize the model, our finetuned model obtained by UM and UMAP can also serve as their starting point or adjustment. As our method does not require any auxiliary outlier data to be involved in training, adjusting the model using ID data during its developing phase is practical.


\section{Experiments}
\label{sec:exp}

In this section, we present the performance comparison of the proposed method in the OOD detection scenario. Specifically, we verify the effectiveness of our UM and UMAP with two mainstreams of OOD detection approaches: (i) fundamental scoring function methods; (ii) outlier exposure methods involving auxiliary samples. To better understand our proposed method, we further conduct various explorations on the ablation study and provide the corresponding discussion on each sub-aspect considered in our work. More details and additional results are presented in Appendix~\ref{app:additional_exp_results}.




\begin{table*}[t!]
    \caption{Main Results ($\%$). Comparison with competitive OOD detection baselines. (averaged by multiple trials)}
    \centering
    \footnotesize
    \renewcommand\arraystretch{0.99}
    \resizebox{\textwidth}{!}{
    \begin{tabular}{c|l|ccccc}
        \toprule[1.5pt]
        $\mathcal{D}_\text{in}$ &  Method & AUROC$\uparrow$ & AUPR$\uparrow$ & FPR95$\downarrow$ & ID-ACC$\uparrow$ & w./w.o $\mathcal{D}_\text{aux}$ \\
        \midrule[0.6pt]
        \multirow{11}*{\textbf{CIFAR-10}}
         & MSP\citep{hendrycks17baseline} & $89.90\pm 0.30$ & $91.48\pm 0.43$ & $60.08\pm 0.76$ & $\textbf{94.01}\pm\textbf{0.08}$ & \\
         & ODIN\citep{LiangLS18} & $91.46\pm 0.56$ & $91.67\pm 0.58$ & $42.31\pm 1.38$ & $\textbf{94.01}\pm\textbf{0.08}$ & \\
         & Mahalanobis\citep{10.5555/3327757.3327819} & $75.10\pm 1.04$ & $72.32 \pm 1.92$ & $61.35\pm 1.25$ & $\textbf{94.01}\pm\textbf{0.08}$ & \\
         & Energy\citep{liu2020energy} & $92.07\pm 0.22$ & $92.72\pm 0.39$ & $42.69\pm 1.31$ & $\textbf{94.01}\pm\textbf{0.08}$ & \\
         & \textbf{Energy+UM} (ours) & $93.73\pm 0.36$ & $94.27\pm 0.60$ & $33.29\pm 1.70$ & $92.80\pm 0.47$ & \\
         & \textbf{Energy+UMAP} (ours) & $\textbf{93.97}\pm\textbf{0.11}$ & $\textbf{94.38}\pm\textbf{0.06}$ & $\textbf{30.71}\pm\textbf{1.94}$ & $\textbf{94.01}\pm\textbf{0.08}$ & \\
         \cmidrule{2-7}
         & OE\citep{hendrycks2018deep} & $97.07\pm 0.01$ & $97.31\pm 0.05$ & $13.80\pm 0.28$ & $92.59\pm 0.32$ & $\checkmark$\\
         & Energy (w. $\mathcal{D}_\text{aux}$)\citep{liu2020energy} & $94.58\pm 0.64$ & $94.69\pm 0.65$ & $18.79\pm 2.31$ & $80.91\pm 3.13$ & $\checkmark$\\
         & POEM\citep{ming2022poem}  & $94.37\pm 0.07$ & $94.51\pm 0.06$ & $18.50\pm 0.33$ & $77.24\pm 2.22$ & $\checkmark$\\
         & \textbf{OE+UM} (ours) & $\textbf{97.60}\pm \textbf{0.03}$ & $\textbf{97.87}\pm \textbf{0.02}$ & $\textbf{11.22}\pm \textbf{0.16}$ & $93.66\pm 0.12$ & $\checkmark$\\
         & \textbf{OE+UMAP} (ours) & $97.48\pm0.01$ & $97.74\pm0.00$ & $12.21\pm0.09$ & $\textbf{94.01}\pm\textbf{0.08}$ & $\checkmark$\\
        \midrule[0.6pt]
        \multirow{11}*{\textbf{CIFAR-100}}
         & MSP\citep{hendrycks17baseline} & $74.06\pm 0.69$ & $75.37\pm 0.73$ & $83.14\pm 0.87$ & $\textbf{74.86}\pm\textbf{0.21}$ & \\
         & ODIN\citep{LiangLS18} & $76.18\pm 0.14$ & $76.49\pm 0.20$ & $78.93\pm 0.31$ & $\textbf{74.86}\pm\textbf{0.21}$ & \\
         & Mahalanobis\citep{10.5555/3327757.3327819} & $63.90\pm 1.91$ & $64.31\pm 0.91$ & $78.79\pm 0.50$ & $\textbf{74.86}\pm\textbf{0.21}$ & \\
         & Energy\citep{liu2020energy} & $\textbf{76.29}\pm \textbf{0.24}$ & $\textbf{77.06}\pm \textbf{0.55}$ & $78.46\pm 0.06$ & $\textbf{74.86}\pm\textbf{0.21}$ & \\
         & \textbf{Energy+UM} (ours) & $76.22\pm 0.42$ & $76.39\pm 1.03$ & $74.05\pm 0.55$ & $64.55\pm 0.24$ & \\
         & \textbf{Energy+UMAP} (ours) & $75.57\pm0.59$ & $75.66\pm0.07$ & $\textbf{72.21}\pm\textbf{1.46}$ & $\textbf{74.86}\pm\textbf{0.21}$ & \\
         \cmidrule{2-7}
         & OE\citep{hendrycks2018deep} & $90.55\pm 0.87$ & $90.34\pm 0.94$ & $34.73\pm 3.85$ & $73.59\pm 0.30$ & $\checkmark$\\
         & Energy (w. $\mathcal{D}_\text{aux}$)\citep{liu2020energy} & $88.92\pm 0.57$ & $89.13\pm 0.56$ & $37.90\pm 2.59$ & $57.85\pm 2.65$ & $\checkmark$\\
         & POEM\citep{ming2022poem}  & $88.95\pm 0.54$ & $88.94\pm 0.31$ & $38.10\pm 1.30$ & $56.18\pm 1.92$ & $\checkmark$\\
         & \textbf{OE+UM} (ours) & $91.04\pm0.11$ & $\textbf{91.13}\pm\textbf{0.24}$ & $34.71\pm0.81$ & $\textbf{75.15}\pm\textbf{0.18}$ & $\checkmark$ \\
         & \textbf{OE+UMAP} (ours) & $\textbf{91.10}\pm\textbf{0.16}$ & $90.99\pm0.23$ & $\textbf{33.62}\pm\textbf{0.26}$ & $74.76\pm0.11$ & $\checkmark$\\
        \bottomrule[1.5pt]
    \end{tabular}}
    \label{tab:my_label}
    \vspace{-4mm}
\end{table*}

\begin{table*}[t!]
    \caption{Fine-grained Results ($\%$). Comparison on different OOD benchmark datasets. (averaged by multiple trials) }
    \centering
    \footnotesize
    \renewcommand\arraystretch{0.95}
    \resizebox{\textwidth}{!}{
    \begin{tabular}{c|l|cccccc}
        \toprule[1.5pt]
        \multirow{3}*{\textbf{ID dataset}} & \multirow{3}*{\textbf{Method}} & \multicolumn{6}{c}{\textbf{OOD dataset}} \\
        ~ & ~ & \multicolumn{2}{c}{\textbf{CIFAR-100}} & \multicolumn{2}{c}{\textbf{Textures}} & \multicolumn{2}{c}{\textbf{Places365}}  \\
        ~ & ~ & FPR95$\downarrow$ & AUROC$\uparrow$ & FPR95$\downarrow$ & AUROC$\uparrow$ & FPR95$\downarrow$ & AUROC$\uparrow$ \\
        \midrule[0.6pt]
        \multirow{12}*{\textbf{CIFAR-10}}
         & MSP & $66.43\pm1.25$ & $87.73\pm0.02$ & $65.20\pm1.33$ & $88.06\pm0.61$ & $61.34\pm0.60$ & $89.63\pm0.15$\\
         & ODIN & $55.31\pm0.85$ & $87.75\pm0.37$ & $53.11\pm4.84$ & $87.13\pm2.04$ & $43.77\pm0.20$ & $91.70\pm0.30$\\
         & Mahalanobis & $81.61\pm0.96$ & $64.52\pm0.73$ & $\textbf{20.04}\pm\textbf{1.43}$ & $\textbf{94.38}\pm\textbf{0.78}$ & $86.21\pm1.36$ & $64.00\pm1.21$\\
         & Energy & $54.65\pm1.24$ & $\textbf{89.01}\pm\textbf{1.18}$ & $57.09\pm3.52$ & $87.51\pm1.43$ & $38.62\pm1.64$ & $93.03\pm0.20$\\
         & \textbf{Energy+UM} (ours) & $\textbf{54.62}\pm\textbf{1.16}$ & $88.30\pm0.30$ & $41.61\pm3.67$ & $91.31\pm0.01$ & $\textbf{30.85}\pm\textbf{0.58}$ & $\textbf{94.27}\pm\textbf{0.16}$\\
         \cmidrule{2-8}
         ~ & \multirow{2}*{\textbf{Method}} &\multicolumn{2}{c}{\textbf{SUN}} &
        \multicolumn{2}{c}{\textbf{LSUN}} & \multicolumn{2}{c}{\textbf{iNaturalist}}\\
        ~ & ~ & FPR95$\downarrow$ & AUROC$\uparrow$ & FPR95$\downarrow$ & AUROC$\uparrow$ & FPR95$\downarrow$ & AUROC$\uparrow$ \\
        \cmidrule{2-8}
         ~& MSP & $60.27\pm0.66$ & $90.00\pm0.24$ & $36.43\pm1.94$ & $95.17\pm0.32$ & $67.53\pm1.64$ & $88.01\pm0.82$\\
         & ODIN & $41.14\pm1.29$ & $92.34\pm0.62$ & $5.16\pm0.76$ & $98.96\pm0.09$ & $54.41\pm0.91$ & $90.17\pm0.19$\\
         & Mahalanobis & $84.56\pm1.51$ & $66.41\pm4.57$ & $69.18\pm3.52$ & $66.41\pm4.57$ & $80.76\pm2.48$ & $71.77\pm1.12$\\
         & Energy & $36.73\pm1.72$ & $93.63\pm0.34$ & $6.25\pm0.43$ & $98.77\pm0.07$ & $59.11\pm1.18$ & $89.71\pm0.06$\\
         & \textbf{Energy+UM} (ours) & $\textbf{27.88}\pm\textbf{0.73}$ & $\textbf{94.83}\pm\textbf{0.11}$ & $\textbf{2.91}\pm\textbf{0.53}$ & $\textbf{99.22}\pm\textbf{0.11}$ & $\textbf{46.27}\pm\textbf{2.74}$ & $\textbf{92.75}\pm\textbf{0.80}$\\
         
        \bottomrule[1.5pt]
    \end{tabular}}
    \label{tab:my_label2}
\end{table*}

\subsection{Experimental Setups}
\label{sec:exp_part1}

\paragraph{Datasets.} Following the common benchmarks used in previous work~\citep{liu2020energy,ming2022poem}, we adopt \verb+CIFAR-10+, \verb+CIFAR-100+ \citep{krizhevsky2009learning_cifar10} as our major ID datasets, and we also adopt \texttt{ImageNet}~\citep{deng2009imagenet} for performance exploration. We use a series of different image datasets as the OOD datasets, e.g., \verb+Textures+ \citep{cimpoi2014describing}, \verb+Places365+ \citep{zhou2017places}, \verb+SUN+ \citep{5539970}, \verb+LSUN+ \citep{yu2015lsun}, \verb+iNaturalist+ \citep{van2018inaturalist} and \texttt{SVHN}~\citep{netzer2011reading_SVHN}. We also use the other ID dataset as OOD dataset when training on a specific ID dataset, given that none of them shares the same classes, e.g., we treat \verb+CIFAR-100+ as the OOD dataset when training on \verb+CIFAR-10+ for comparison. We utilize the ImageNet-1k~\citep{deng2009imagenet} training set as the auxiliary dataset for all of our experiments about fine-tuning with auxiliary outliers (e.g., OE/Energy/POEM), which is detailed in Appendix \ref{app:additional_exp_setup}. This choice follows previous literature \citep{hendrycks2018deep,liu2020energy,ming2022poem} that considers the dataset's availability and the absence of any overlap with the ID datasets.

\begin{figure*}[t!]
    \begin{center}
    
    \subfigure[Train v.s. finetune]{
    \includegraphics[scale=0.13]{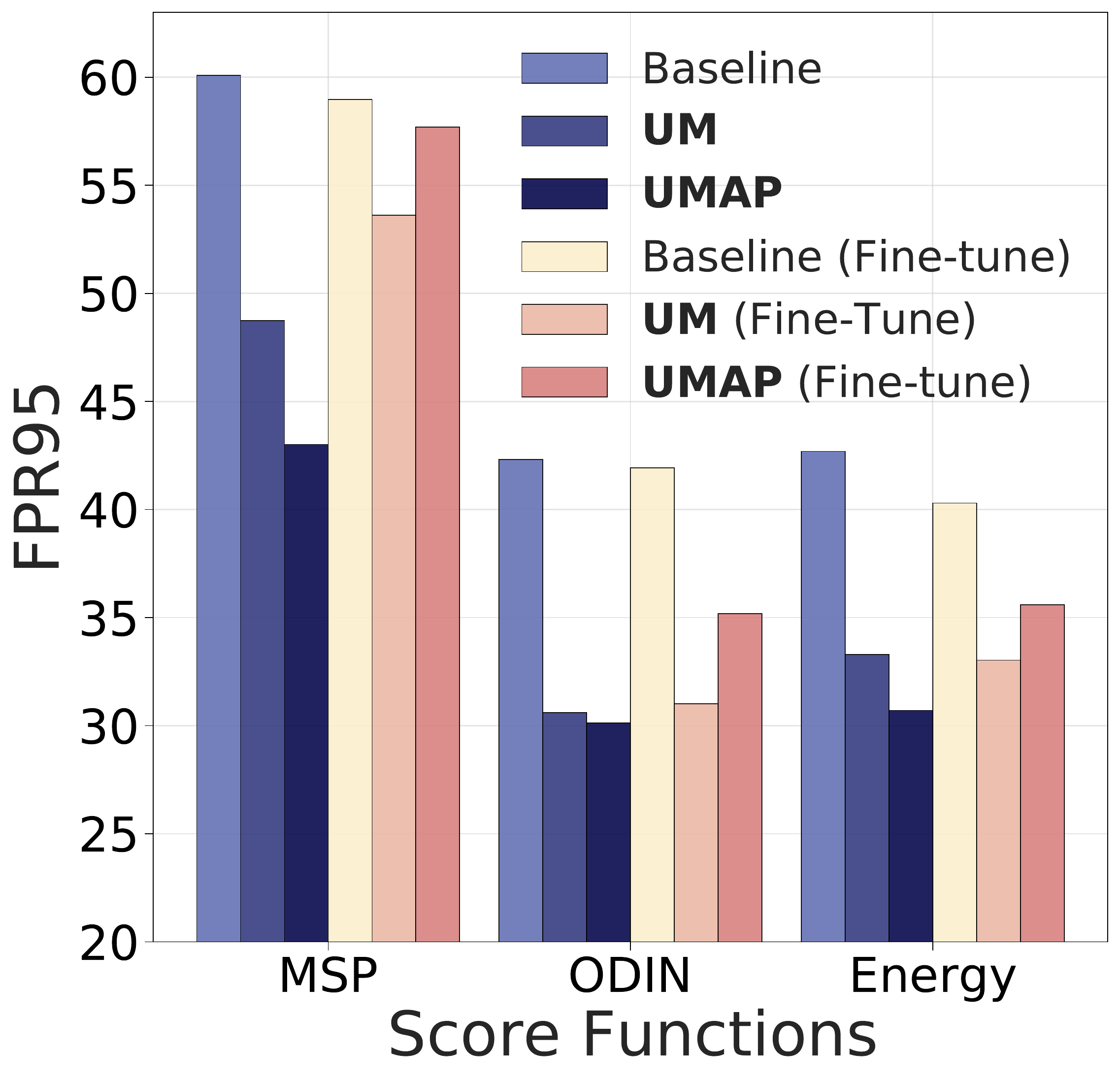}
    \label{fig4:a}
    }
    \subfigure[Compare with Overfitting]{
    \includegraphics[scale=0.13]{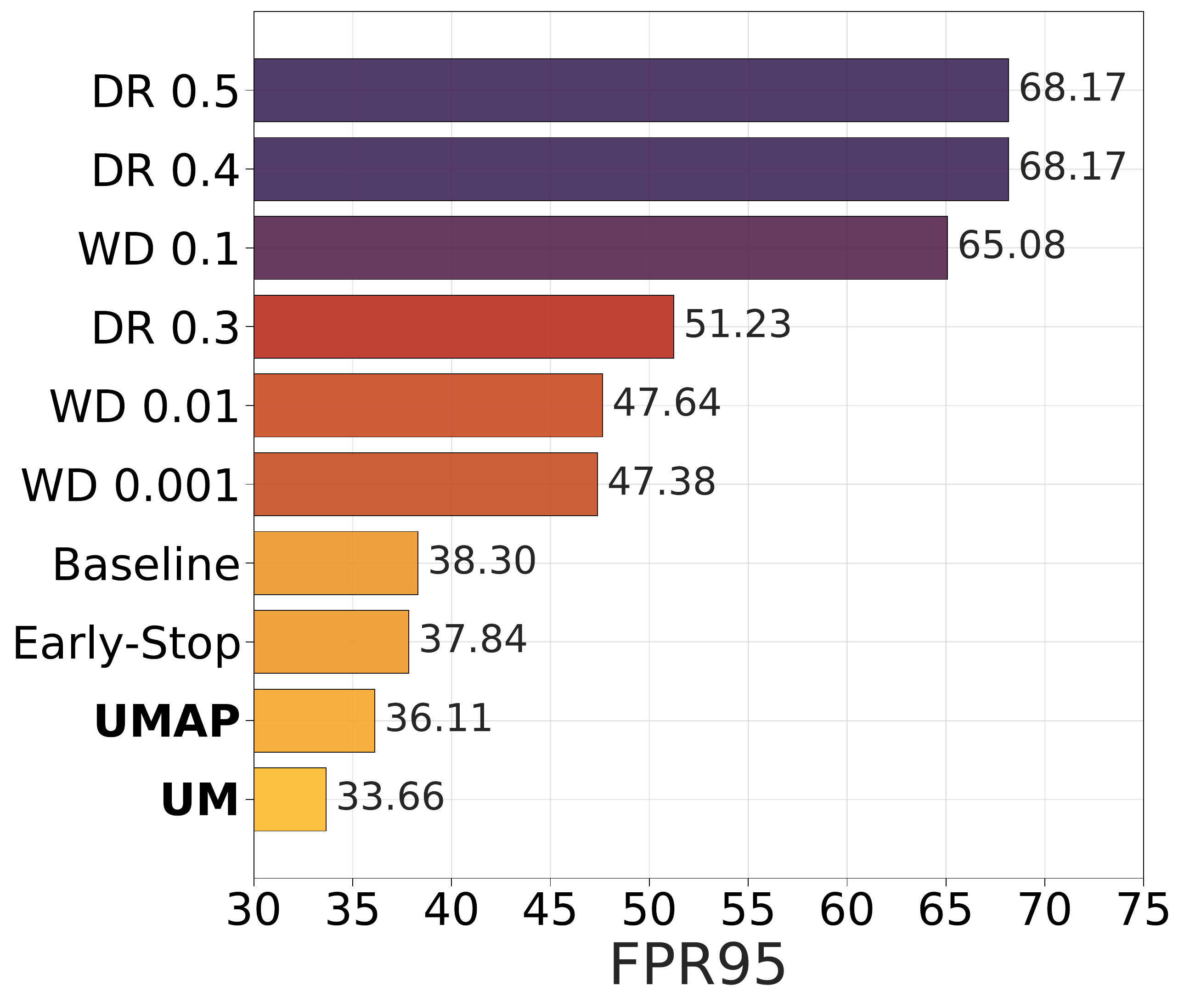}
    \label{fig4:b}
    }
    \subfigure[Diff. Score Functions]{
    \includegraphics[scale=0.13]{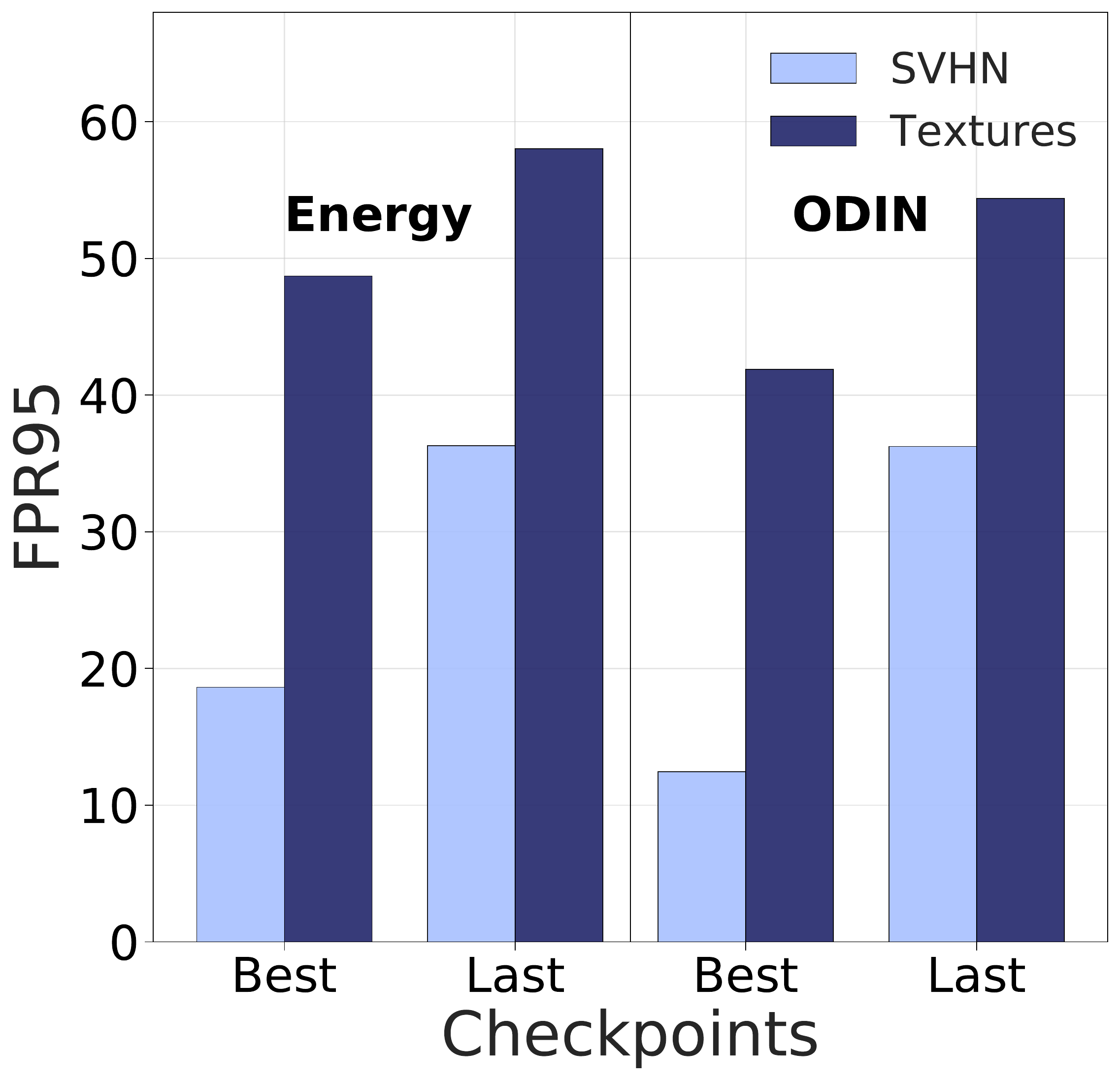}
    \label{fig4:c}
    }
    \subfigure[Diff. Masking Ratios]{
    \includegraphics[scale=0.13]{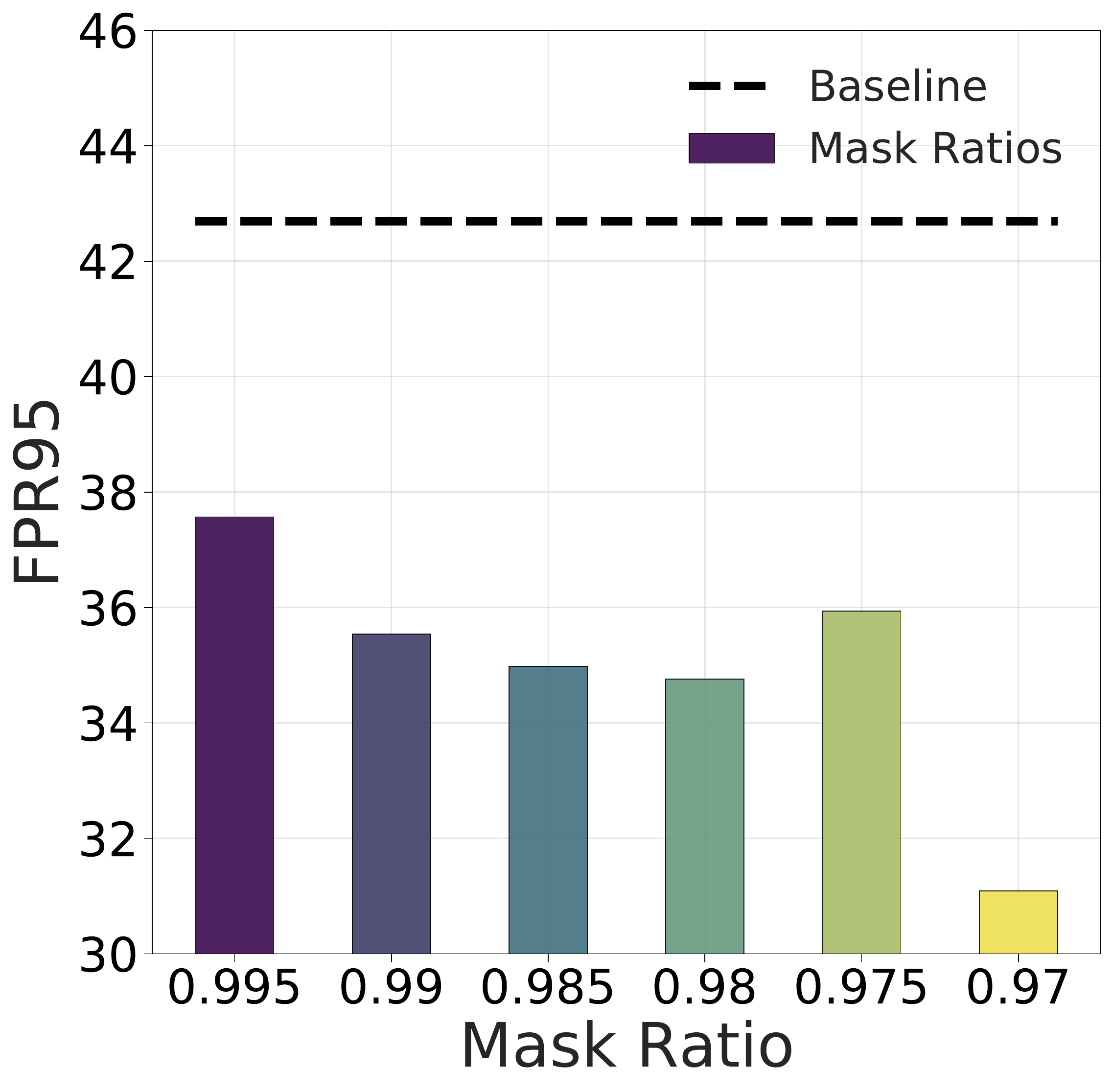}
    \label{fig4:d}
    }
    \subfigure[UMAP v.s. Prune]{
    \includegraphics[scale=0.13]{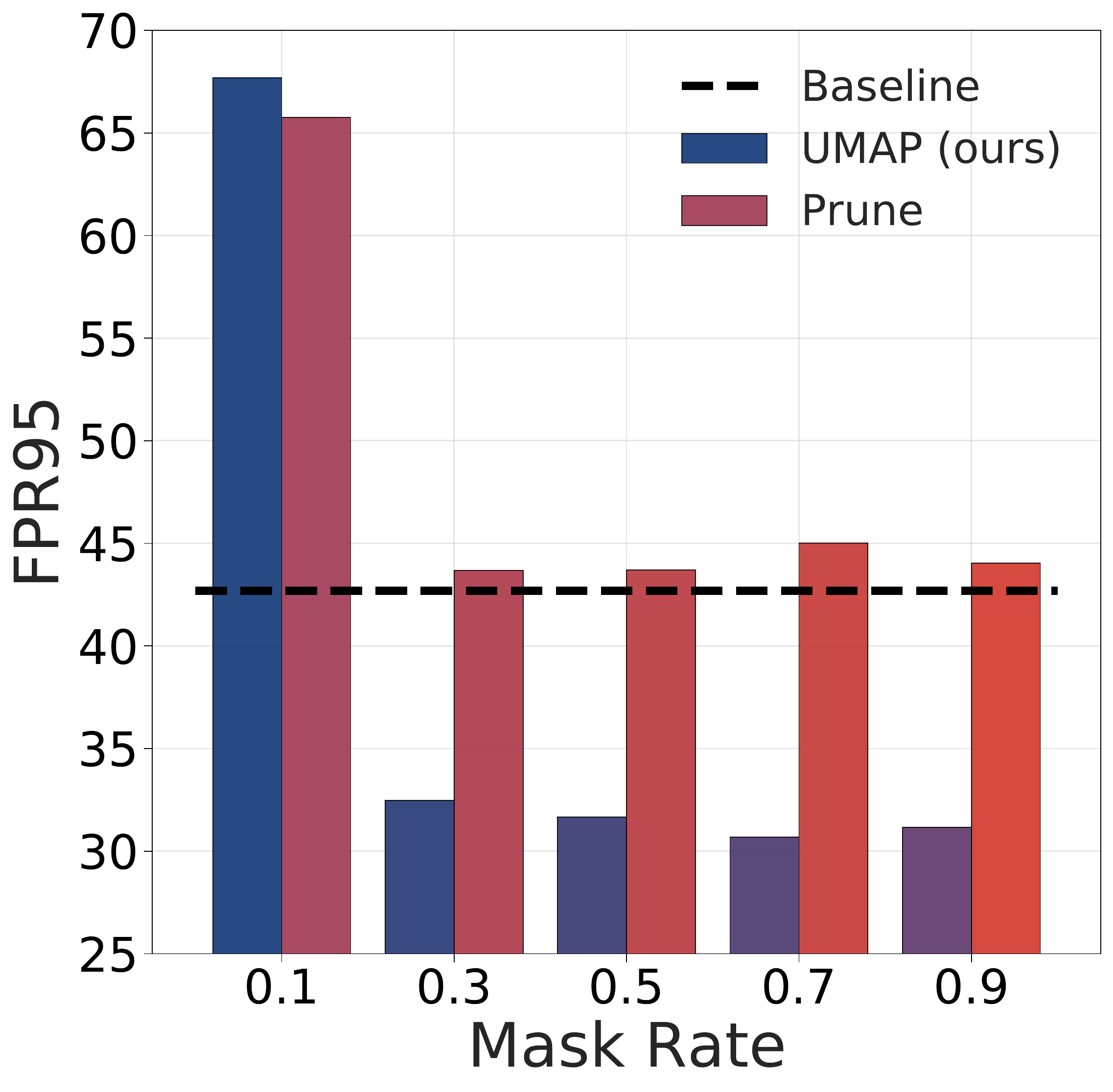}
    \label{fig4:e}
    }
    \end{center}
    \vspace{-4mm}
    \caption{\textbf{Ablation studies:} (a) efficiency of the finetuning adopted in UM and UMAP; (b) comparison of UM and UMAP with other strategies for alleviating the conventional overfitting; (c) the historical model stage using different scoring functions for OOD detection; (d) effects of using different masking ratios for atypical mining in UM; (e) comparison of using vanilla pruning with our proposed UMAP.
    }
\end{figure*}

\paragraph{Evaluation metrics.} We employ the following three common metrics to evaluate the performance of OOD detection: (i) Area Under the Receiver Operating Characteristic curve (AUROC) \citep{inproceedings} can be interpreted as the probability for a positive sample to have a higher discriminating score than a negative sample \citep{fawcett2006introduction}; (ii) Area Under the Precision-Recall curve (AUPR) \citep{manning99foundations} is an ideal metric to adjust the extreme difference between positive and negative base rates; (iii) False Positive Rate (FPR) at $95\%$ True Positive Rate (TPR) \citep{LiangLS18} indicates the probability for a negative sample to be misclassified as positive when the true positive rate is at $95$\%. We also include in-distribution testing accuracy (ID-ACC) to reflect the preservation level of the performance for the original classification task on ID data.

\paragraph{OOD detection baselines.} We compare the proposed method with several competitive baselines in the two directions. Specifically, we adopt Maximum Softmax Probability (MSP) \citep{hendrycks17baseline}, ODIN \citep{LiangLS18}, Mahalanobis score \citep{10.5555/3327757.3327819}, and Energy score \citep{liu2020energy} as scoring function baselines; We adopt OE \citep{hendrycks2018deep}, Energy-bounded learning \citep{liu2020energy}, and POEM \citep{ming2022poem} as baselines with outliers. For all scoring function methods, we assume the accessibility of well-trained models. For all methods involving outliers, we constrain all major experiments to a finetuning scenario, which is more practical in real cases. Different from training a dual-task model at the very beginning, equipping deployed models with OOD detection ability is a much more common circumstance, considering the millions of existing deep learning systems. We leave more implementation details in Appendix~\ref{app:baseline_info}.

\subsection{Performance Comparison}
\label{sec:exp_part2}

In this part, we present the performance comparison with some representative baseline methods to demonstrate the effectiveness of our UM and UMAP. 
In each category of Table~\ref{tab:my_label}, we choose one with the best detection performance to adopt UM or UMAP and check the three evaluation metrics of OOD detection and the ID-ACC. 


In Table~\ref{tab:my_label}, we summarize the results using different methods. For the scoring-based methods, our UM can further improve the overall detection performance by alleviating the memorization of atypical ID data, when the ID-ACC keeps comparable with the baseline. For the complex CIFAR-100 dataset, our UMAP can be adopted as a practical way to empower the detection performance and simultaneously avoid severely affecting the original performance on ID data. As for those methods of the second category (i.e., involving auxiliary outlier $\mathcal{D}_\text{aux}$ sampled from ImageNet), since we consider a practical workflow, i.e., fine-tuning, on the given model, OE achieves the best performance on the task. Due to the special optimization characteristic, Energy (w. $\mathcal{D}_\text{aux}$) and POEM focus more on the energy loss on differentiating OOD data while performing not well on the preservation of ID-ACC. Without sacrificing much performance on ID data, OE with our UM can still achieve better detection performance. In Table~\ref{tab:my_label2}, the fine-grained detection performance on each OOD testing set demonstrates the general effectiveness of UM and UMAP. Note that we may observe Mahalanobis can sometimes achieve the best performance on the specific OOD test set (e.g., Textures). It is probably because Mahalanobis is prone to overfitting on texture features during fine-tuning with Textures. In contrast, according to Table \ref{tab:my_label2}, Mahalanobis achieves the worst results on the other five datasets. We leave more results (e.g., completed comparison in Table~\ref{tab:my_label_complete}; more fine-grained results in Tables~\ref{tab:my_label3} and~\ref{tab:my_label4}; using another model structure in Tables~\ref{tab:label_wrn},~\ref{tab:label_wrn_zoom_in_cifar10} and~\ref{tab:label_wrn_zoom_in_cifar100}) to Appendix, which has verified the significant improvement (up to $18\%$ reduced on averaged FPR95) across various setups and also on a large-scale ID dataset (i.e., ImageNet~\citep{deng2009imagenet} in Table~\ref{tab:my_imagenet}).

\begin{figure*}[t!]
    \begin{center}
    \vspace{2mm}
    \includegraphics[scale=0.11]{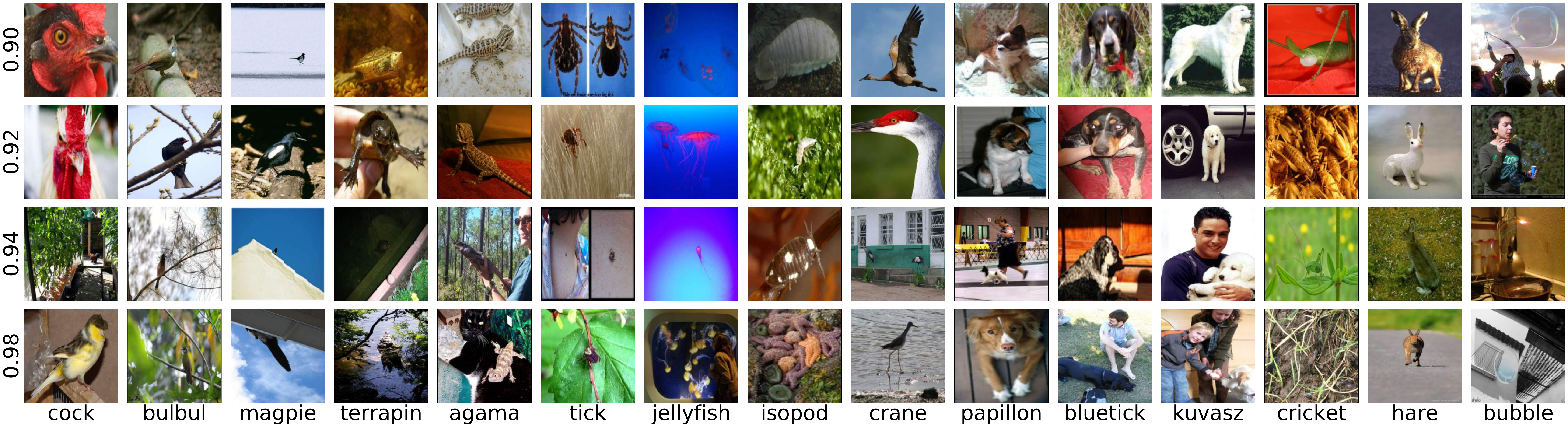}
    \end{center}
    \vspace{-2mm}
    \caption{\textbf{Sample Visualizations:} examples of the misclassified samples after adopting masking on the original well-trained model using the ImageNet dataset. At the left of each line, we indicate the mask ratio that is adopted on the model. We can find that masking with a smaller ratio forces the model to misclassify simple samples (clear contours around subjects, single color background) while masking with a larger ratio guides the model to misclassify complex samples (unclear contours, noisy background). This inspection empirically verifies our intuition using the proper mask ratio to identify those atypical samples and then force the model to forget them.
    }
    \label{fig:sample_vis}
\end{figure*}


\subsection{Ablation and Further Analysis}
\label{sec:exp_part3}

In this part, we conduct further explorations and analysis to provide a thorough understanding of our UM and UMAP. Moreover, we also provide additional experimental results about further explorations on OOD detection in Appendix~\ref{app:additional_exp_results}.


\paragraph{Practicality of the considered setting and the implementation choice.} Following the previous work~\citep{liu2020energy,hendrycks2018deep}, we consider the same setting that starts from a given well-trained model in major explorations, which is practical but can be extended to another implementation choice, i.e., retraining the whole model. In Figure~\ref{fig4:a}, we show the effectiveness of UM/UMAP under different choices. It is worth noting that UM adopting finetuning has shown the advantages of being cost-effective on convergence compared with train-from-scratch, which we leave more discussion and comparison in Appendix~\ref{app:exp_less_epochs}. 


\paragraph{Specificity and applicability of excavated OOD discriminative capability.} As mentioned before, the intrinsic OOD discriminative capability is distinguishable from conventional overfitting. We empirically compare UM/UMAP with dropout (DR), weight decay (WD), and early stop in Figure~\ref{fig4:b}. UM gain lower FPR95 from the newly designed objective for forgetting. In Figure~\ref{fig4:c}, we present the applicability of the OOD detection capability using different score functions, which implies the generated model stage better meets the requirement of uncertainty estimation.



\paragraph{Effects of the mask on mining atypical samples.} In Figure~\ref{fig4:d}, we compare UM with different mask ratios for mining the atypical samples, which seeks to find the intermediate model stage that wrongly classified the atypical samples. The results show reasonably small ratios (e.g., from $0.995$ to $0.97$) that we knocked off in the original model can help us to achieve the targets. More detailed analysis of the mask ratio and the discussion about the underlying intuition of atypical mining are provided in Appendixes~\ref{app:eff_um} and~\ref{app:atypical_mining}.




\paragraph{Exploration on UMAP and vanilla model pruning.} Although the large constraint on training loss can help reveal the OOD detection performance, the ID-ACC may be undermined under such circumstances. To mitigate this issue, we further adopt pruning in UMAP to learn a mask instead of tuning the model parameters directly. In Figure~\ref{fig4:e}, we explore various prune rates $p$ and demonstrate their effectiveness. Specifically, our UMAP can achieve a lower FPR95 than vanilla pruning with the original objective.
The prune rate can be selected from a wide range (e.g., $p \in [0.3, 0.9]$) to guarantee a fast convergence and effectiveness. We also provide additional discussion on UMAP in Appendix~\ref{app:comp_prune_umap}.

\paragraph{Sample visualization of the atypical samples identified by our mask.} 
In Figure~\ref{fig:sample_vis}, we visualize the misclassified samples using the ImageNet~\citep{deng2009imagenet} dataset with the pre-trained model by adopting different mask ratios. We can find that masking the model constructs the parameter discrepancy, which helps us to identify some ID samples with atypical semantic information (e.g., those samples in the bottom line compared with the above in each class). It demonstrates the rationality of our intuition to adopt masking. We leave more visualization results in Appendix~\ref{app:atypical_mining}.

\begin{figure}[t!]
    \begin{center}
    \includegraphics[scale=0.095]{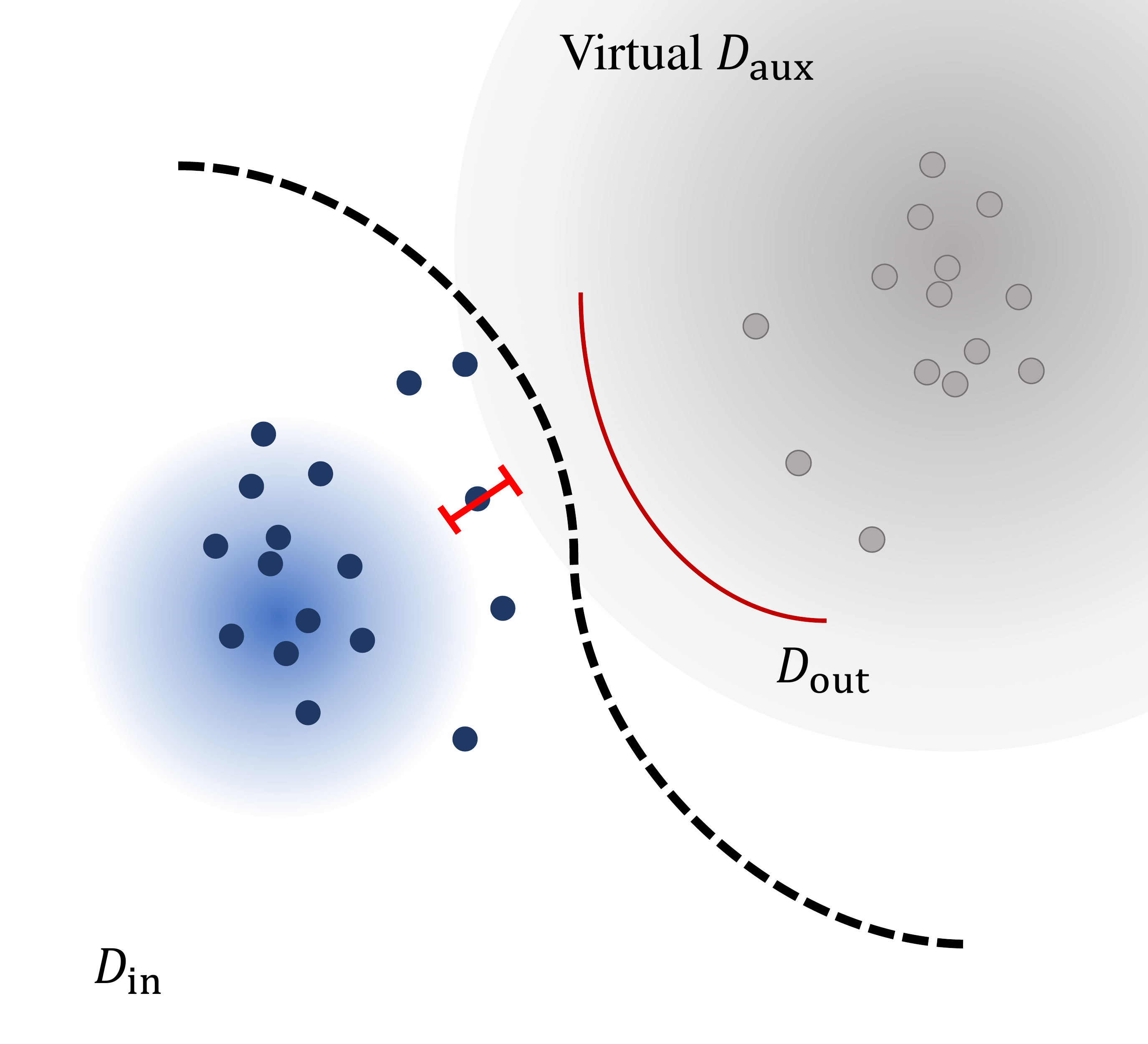}
    \includegraphics[scale=0.095]{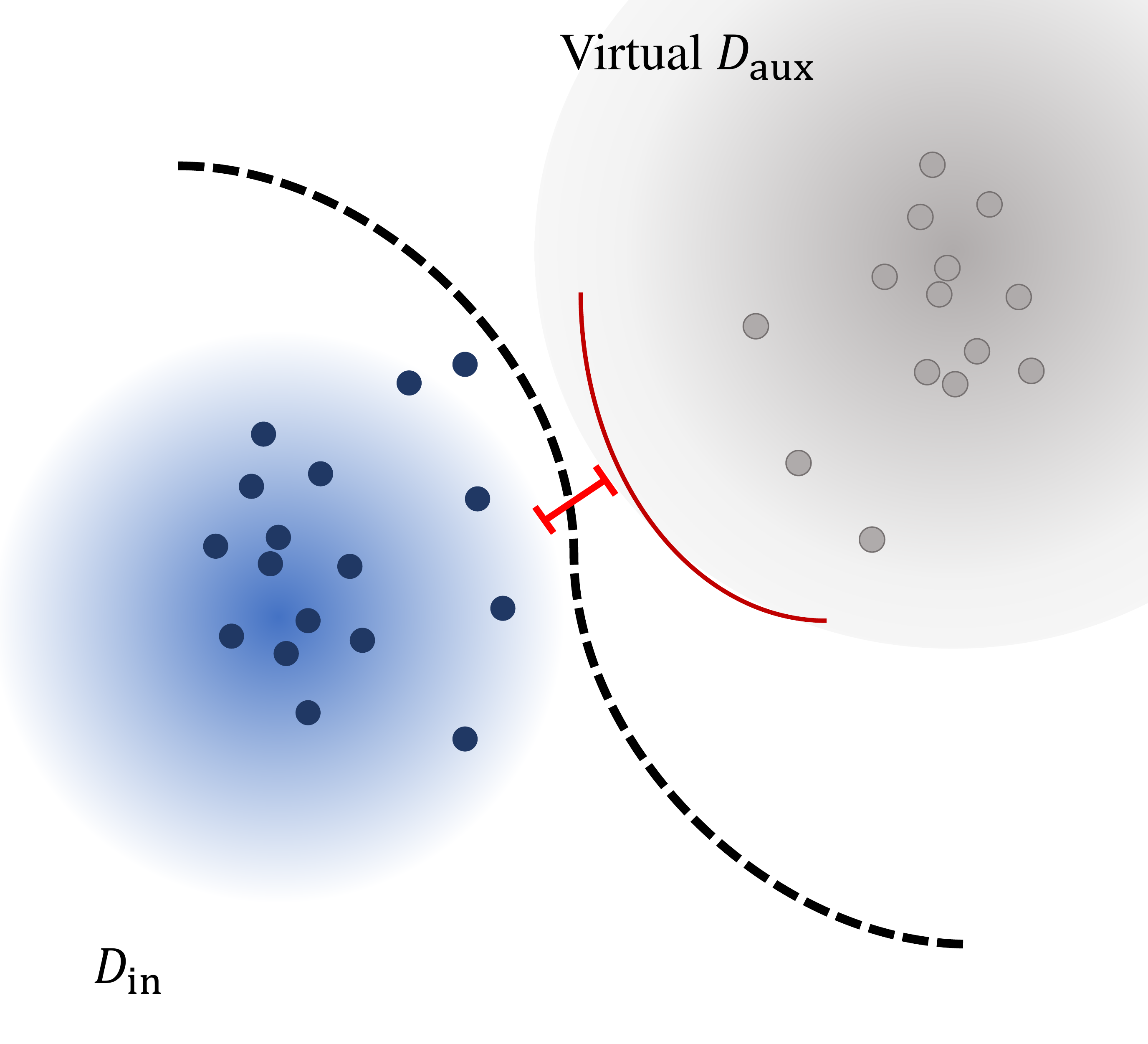}
    \includegraphics[scale=0.095]{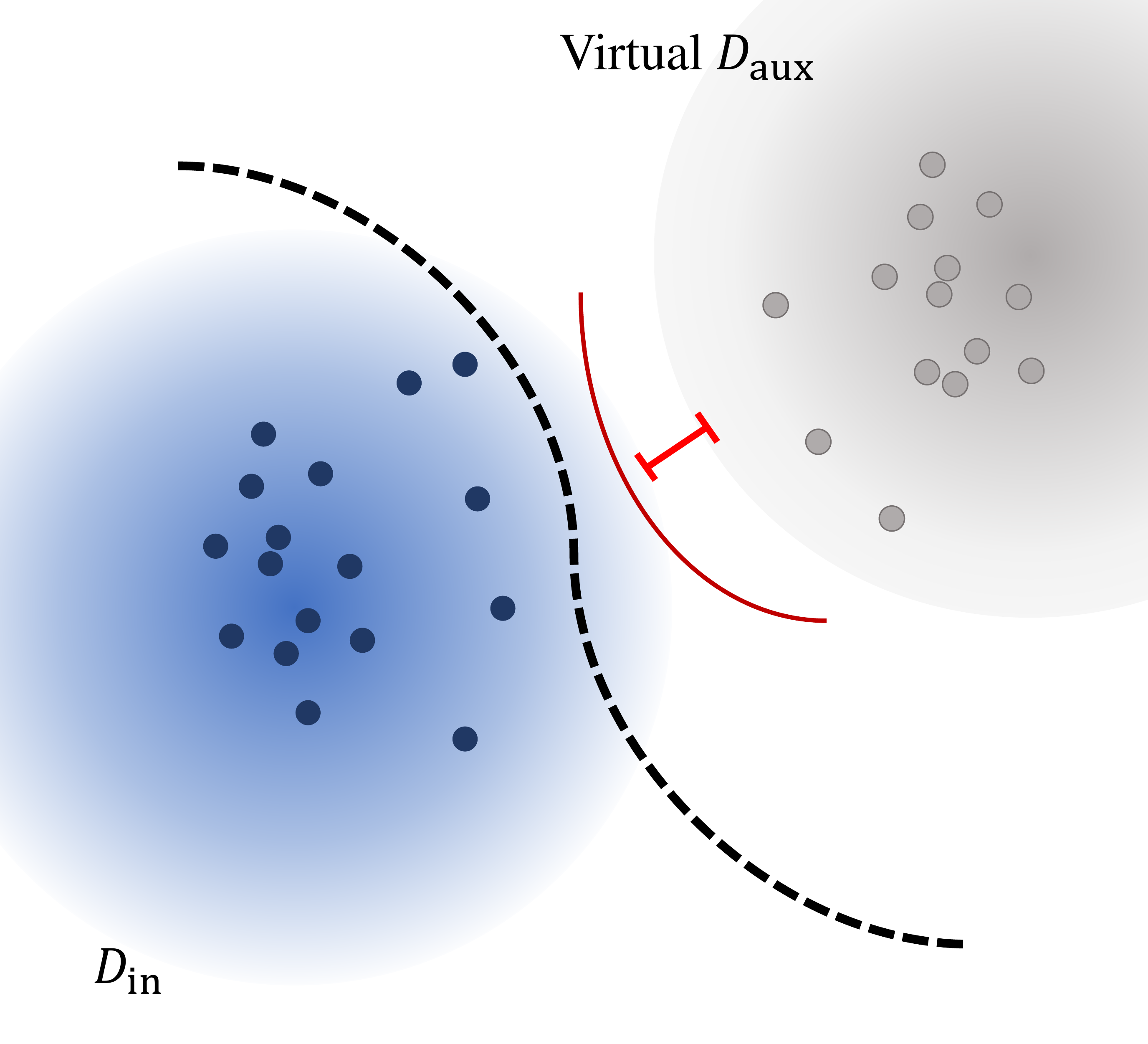}
    \includegraphics[scale=0.101]{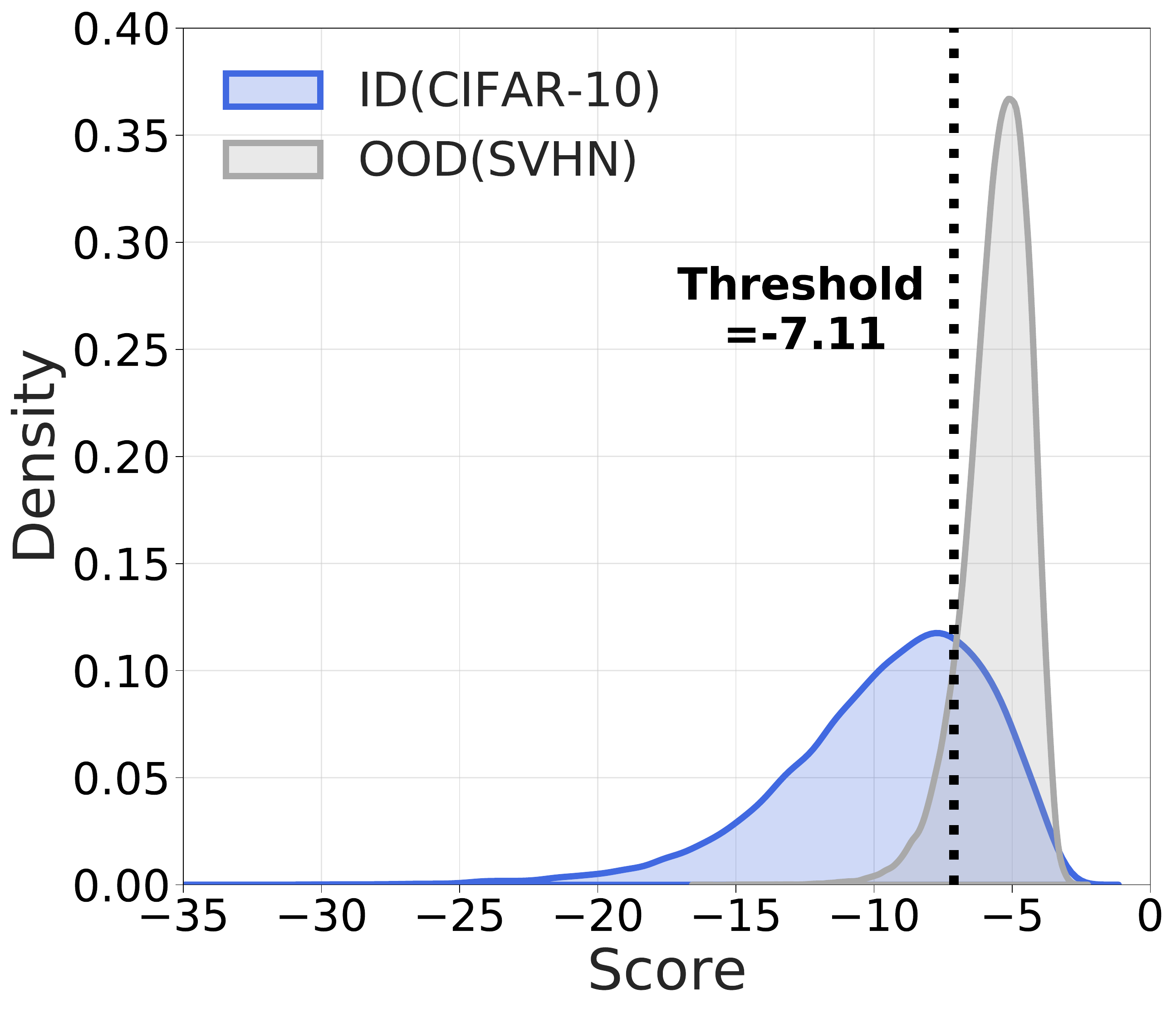}
    \includegraphics[scale=0.101]{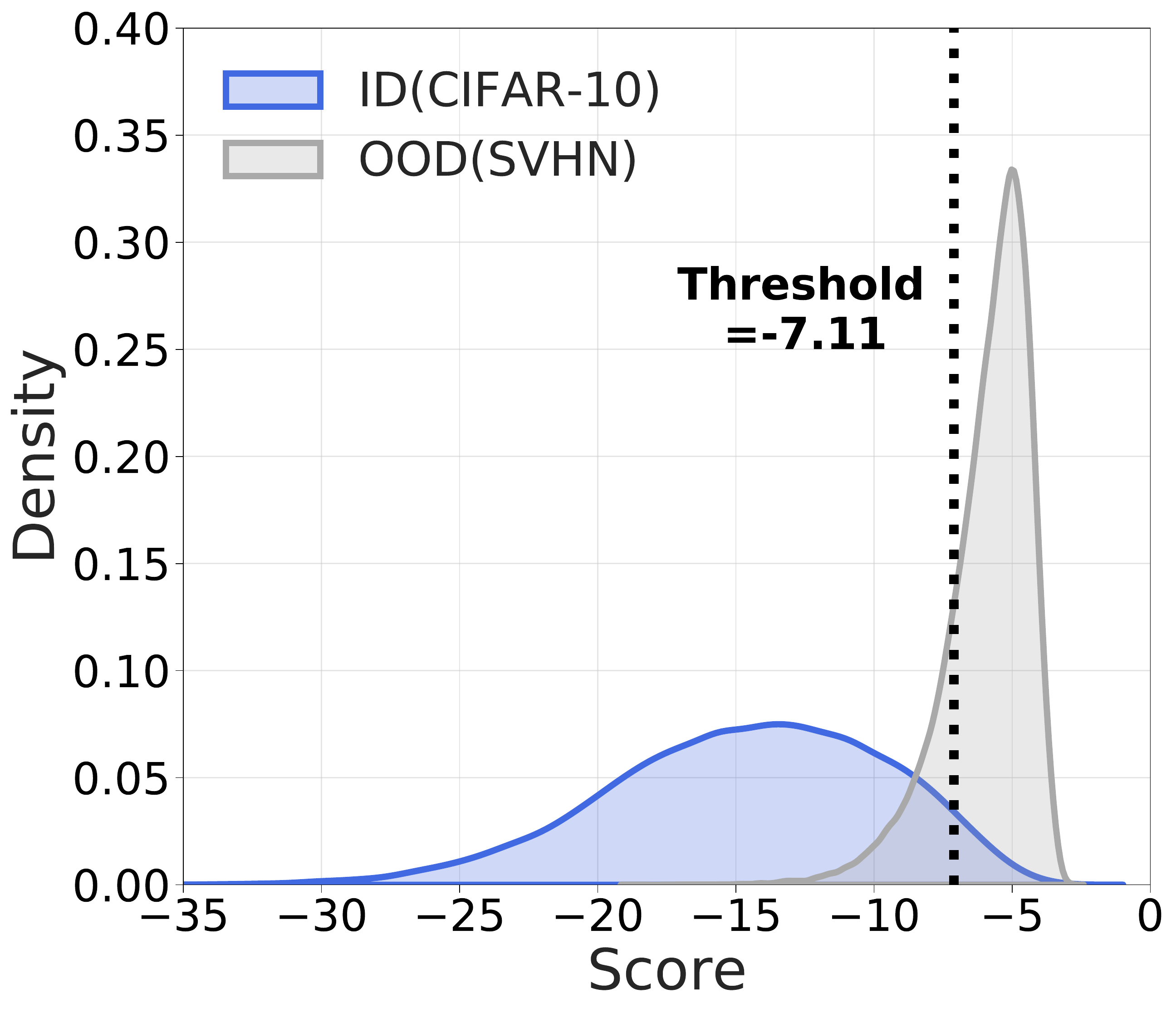}
    \includegraphics[scale=0.101]{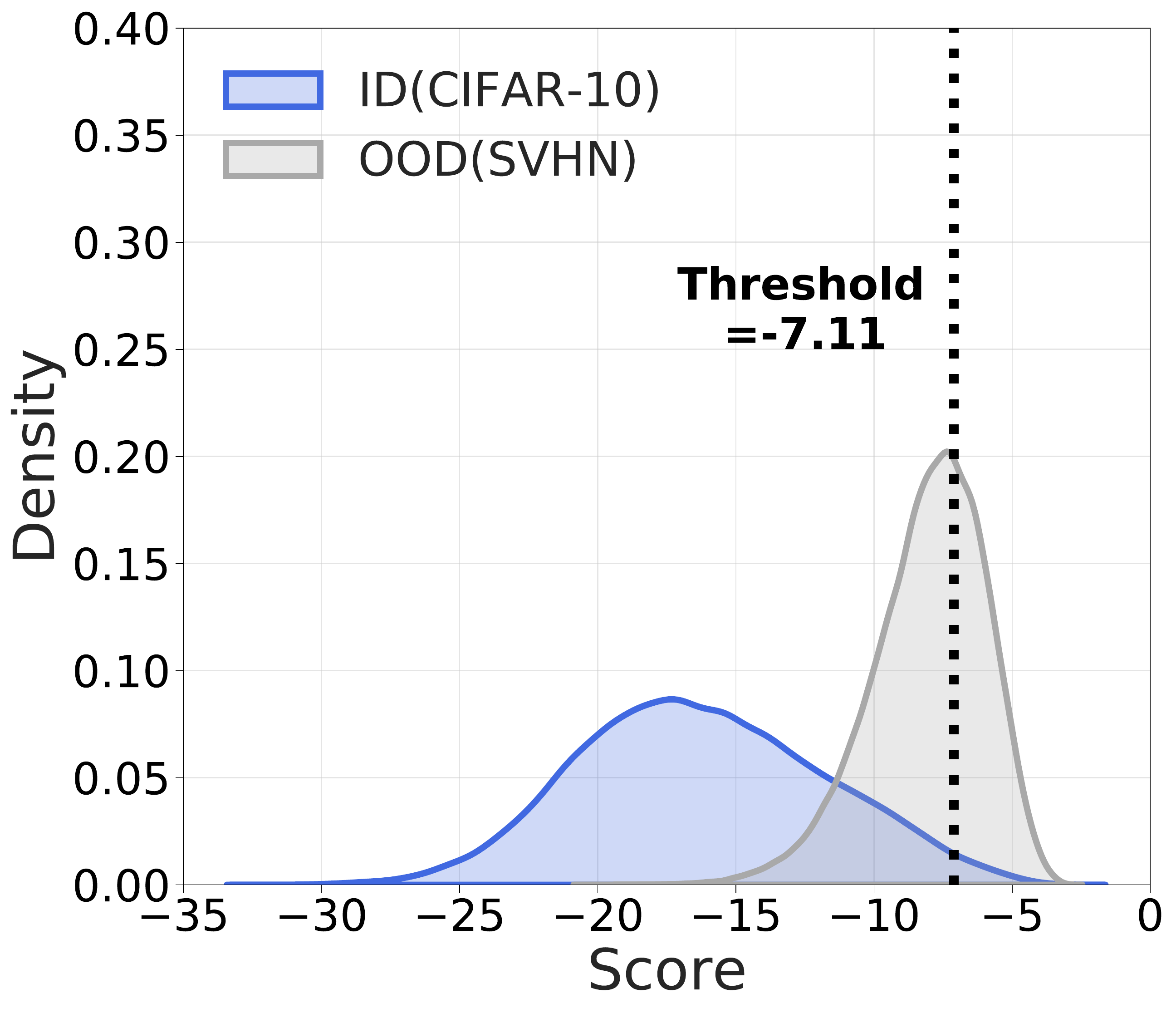}
    \end{center}
    \vspace{-4mm}
    \caption{Illustration about the framework for theoretical anslysis. From left to right: the model under-represent on $\mathcal{D}_\text{in}$ with the lower confidence on the atypical samples close to the boundary; the model reaches a near-optimal representation status on $\mathcal{D}_\text{in}$; the model over-represent on $\mathcal{D}_\text{in}$. See Figure~\ref{fig:theo} for more explanations.
    }
    \label{fig5:theo_ill}
    \vspace{-3mm}
\end{figure}

\paragraph{Theoretical insights on ID data property.} Similar to prior works~\citep{LeeLLS18, SehwagCM21}, here we present the major results based on the sample complexity analysis adopted in POEM~\citep{ming2022poem}. Due to the limited space,  please refer to Appendix~\ref{app:theo} for the completed analysis and Figure~\ref{fig:theo} for more conceptual understanding.

\begin{theorem}
    Given a simple Gaussian mixture model in binary classification with the hypothesis class $\mathcal{H}={\text{sign}(\theta^Tx), \theta\in \mathbb{R}^d}$. There exists constant $\alpha$, $\delta^*$ and $\epsilon$ that,
    \begin{equation}
        \frac{\mu^T\theta^*_{n_1,n_2}}{\sigma||\theta^*_{n_1,n_2}||} \geq \frac{||\mu||^2-\sigma^{\frac{1}{2}}||\mu||^{\frac{3}{2}}-\frac{\sigma^2(|\alpha-\delta^*|+\epsilon)}{2}}{2\sqrt{\frac{\sigma^2}{n}(d+\frac{1}{\sigma}) + ||\mu||^2}}
    \end{equation}
\end{theorem}
Since $\text{FPR}(\theta^*_{n_1,n_2})=erf(\frac{\mu^T\theta^*_{n_1,n_2}}{\sigma||\theta^*_{n_1,n_2}||})$ is monotonically decreasing, as the lower bound of $\frac{\mu^T\theta^*_{n_1,n_2}}{\sigma||\theta^*_{n_1,n_2}||}$ will increase with the constraint of $|\alpha-\delta^*|$ (which corresponds to the illustrated distance from the outlier boundary in right-most of Figure~\ref{fig5:theo_ill}) decrease in our methods, the upper bound of $\text{FPR}(\theta^*_{n_1,n_2})$ will decrease. One insight is learning more atypical ID samples needs more high-quality auxiliary outliers (near ID data) to shape the OOD detection capability.

\paragraph{Additional experimental results of explorations.} 
Except for the major performance comparisons and the previous ablations, we also provide further discussion and analysis from different views in Appendix~\ref{app:additional_exp_results}, including the practicality of the considered setting, the effects of the mask on mining atypical samples, discussion of UMAP with vanilla pruning, additional comparisons with more advanced methods and completed results of our proposed UM and UMAP.

\section{Related Work}

\textbf{OOD Detection without auxiliary data.} \citet{hendrycks17baseline} formally shed light on out-of-distribution detection, proposing to use softmax prediction probability as a baseline which is demonstrated to be unsuitable for OOD detection~\citep{hendrycks2018deep}. Subsequent works~\citep{sun2021react} keep focusing on designing post-hoc metrics to distinguish ID samples from OOD samples, among which ODIN \citep{LiangLS18} introduces small perturbations into input images to facilitate the separation of softmax score, Mahalanobis distance-based confidence score \citep{10.5555/3327757.3327819} exploits the feature space by obtaining conditional Gaussian distributions, energy-based score \citep{liu2020energy} aligns better with the probability density. Besides directly designing new score functions, many other works pay attention to various aspects to enhance the OOD detection such that LogitNorm \citep{wei2022logitnorm} produces confidence scores by training with a constant vector norm on the logits, and DICE \citep{sun2022dice} reduces the variance of the output distribution by leveraging the model sparsification.

\textbf{OOD Detection with auxiliary data.} Another promising direction toward OOD detection involves the auxiliary outliers for model regularization. On the one hand, some works generate virtual outliers such that \citet{LeeLLS18} uses generative adversarial networks to generate boundary samples, VOS \citep{du2022vos} regularizes the decision boundary by adaptively sampling virtual outliers from the low-likelihood region. On the other hand, other works tend to exploit information from natural outliers, such that outlier exposure is introduced by \citet{hendrycks2018deep}, given that diverse data are available in enormous quantities. \citep{DBLP:conf/iccv/YuA19} train an additional "head" and maximizes the discrepancy of decision boundaries of the two heads to detect OOD samples. Energy-bounded learning \citep{liu2020energy} fine-tunes the neural network to widen the energy gap by adding an energy loss term to the objective. Some other works also highlight the sampling strategy, such that ATOM \citep{chen2021atom} greedily utilizes informative auxiliary data to tighten the decision boundary for OOD detection, and POEM \citep{ming2022poem} adopts Thompson sampling to contour the decision boundary precisely. The performance of training with outliers is usually superior to that without outliers, shown in many other works~\citep{liu2020energy, fort2021exploring, sun2021react, SehwagCM21, chen2021atom,salehi2021unified, wei2022logitnorm}.

\section{Conclusion}

In this work, we explore the intrinsic OOD discriminative capability of a well-trained model from a unique data-level attribution. Without involving any auxiliary outliers in training, we reveal the inconsistent trend between minimizing original training loss and gaining OOD detection capability. We further identify the potential attribution to be the memorization on atypical samples. To excavate the overlaid capability, we propose the novel Unleashing Mask (UM) and its practical variant UMAP. Through this, we construct model-level discrepancy that figures out the memorized atypical samples and utilizes the constrained gradient ascent to encourage forgetting. It better utilizes the well-trained given model via backtracking or sub-structure pruning. We hope our work could provide new insights for revisiting the model development in OOD detection, and draw more attention toward the data-level attribution. Future work can be extended to a more systematical ID/OOD data investigation with other topics like data pruning or few-shot finetuning.

\section*{Acknowledgements}

JNZ and BH were supported by NSFC Young Scientists Fund No. 62006202, Guangdong Basic and Applied Basic Research Foundation No. 2022A1515011652, CAAI-Huawei MindSpore Open Fund, and HKBU CSD Departmental Incentive Grant. JCY was supported by the National Key R\&D Program of China (No. 2022ZD0160703),  STCSM (No. 22511106101, No. 22511105700, No. 21DZ1100100), 111 plan (No. BP0719010). JLX was supported by RGC grants 12202221 and C2004-21GF.

\clearpage

\bibliography{main}
\bibliographystyle{icml2023}

\newpage
\appendix
\onecolumn

\section*{Appendix}

\section*{Reproducibility Statement}

We provide the link of our source codes to ensure the reproducibility of our experimental results: \url{https://github.com/tmlr-group/Unleashing-Mask}. Below we summarize critical aspects to facilitate reproducible results:

\begin{itemize}
    \item \textbf{Datasets.}  The datasets we used are all publicly accessible, which is introduced in Section~\ref{sec:exp_part1}. For methods involving auxiliary outliers, we strictly follow previous works \citep{sun2021react, du2022vos} to avoid overlap between the auxiliary dataset (ImageNet-1k)~\citep{deng2009imagenet} and any other OOD datasets.
    \item \textbf{Assumption.} We set our experiments to a post-hoc scenario~\citep{liu2020energy} where a well-trained model is available, and some parts of training samples are also available for subsequent fine-tuning~\citep{hendrycks2018deep}.
    \item \textbf{Environment.} All experiments are conducted with multiple runs on NVIDIA Tesla V100-SXM2-32GB GPUs with Python 3.6 and PyTorch 1.8.
\end{itemize}

\section{Details about Considered Baselines and Metrics}
\label{app:baseline_info}

In this section, we provide the details about the baselines for the scoring functions and fine-tuning with auxiliary outliers, as well as the corresponding hyper-parameters and other related metrics that are considered in our work.
\paragraph{Maximum Softmax Probability (MSP).} \citep{hendrycks17baseline} proposes to use maximum softmax probability to discriminate ID and OOD samples. The score is defined as follows,
\begin{equation}
    S_\text{MSP}(x; f) = \max_cP(y = c | x; f) = \max~\texttt{softmax}(f(x))
\end{equation}
where $f$ represents the given well-trained model and $c$ is one of the classes $\mathcal{Y}=\{1,\ldots, C\}$. The larger softmax score indicates the larger probability for a sample to be ID data, reflecting the model's confidence on the sample. 

\paragraph{ODIN.} \citep{LiangLS18} designed the ODIN score, leveraging the temperature scaling and tiny perturbations to widen the gap between the distributions of ID and OOD samples. The ODIN score is defined as follows,
\begin{equation}
    S_\text{ODIN}(x; f) = \max_cP(y = c | \tilde{x}; f) = \max~\texttt{softmax}(\frac{f(\tilde{x})}{T})
\end{equation}
where $\tilde{x}$ represents the perturbed samples (controled by $\epsilon$), $T$ represents the temperature. For fair comparison, we adopt the suggested hyperparameters \citep{LiangLS18}: $\epsilon = 1.4\times 10^{-3}$, $T = 1.0 \times 10^4$.

\paragraph{Mahalanobis.} \citep{10.5555/3327757.3327819} introduces a Mahalanobis distance-based confidence score, exploiting the feature space of the neural networks by inspecting the class conditional Gaussian distributions. The Mahalanobis distance score is defined as follows,
\begin{equation}
    S_\text{Mahalanobis}(x; f) = \max \limits_{c} - (f(x) - \hat{\mu}_c)^T \hat{\Sigma}^{-1}(f(x) - \hat{\mu}_c)
\end{equation}
where $\hat{\mu}_c$ represents the estimated mean of multivariate Gaussian distribution of class $c$, $\hat{\Sigma}$ represents the estimated tied covariance of the $C$ class-conditional Gaussian distributions.

\paragraph{Energy.} \citep{liu2020energy} proposes to use the Energy of the predicted logits to distinguish the ID and OOD samples. The Energy score is defined as follows,
\begin{equation}
    S_\text{Energy}(x; f) = -T  \log \sum \limits_{c = 1}^C e^{f(x)_c / T}
\end{equation}
where $T$ represents the temperature parameter. As theoretically illustrated in \citet{liu2020energy}, a lower Energy score indicates a higher probability for a sample to be ID. Following \citep{liu2020energy}, we fix the $T$ to $1.0$ throughout all experiments.

\paragraph{Outlier Exposure (OE).} \citep{hendrycks2018deep} initiates a promising approach towards OOD detections by involving outliers to force apart the distributions of ID and OOD samples. In the experiments, we use the cross-entropy from $f(x_{\text{out}})$to the uniform distribution as the $\mathcal{L}_{\text{OE}}$ \citep{LeeLLS18},
\begin{equation}
\label{eq:oe_app}
    \mathcal{L}_f = \mathbb{E}_{\mathcal{D}_\text{in}}\left[\ell_\text{CE}(f(x),y)\right] + \lambda\mathbb{E}_{\mathcal{D}^\text{s}_\text{out}}\left[\log \sum \limits_{c = 1}^C e^{f(x)_c} - \mathbb{E}_{\mathcal{D}^\text{s}_\text{out}}(f(x))\right]
\end{equation}  

\paragraph{Energy (w. $\mathcal{D}_{\text{aux}}$).} In addition to using the Energy as a post-hoc score to distinguish ID and OOD samples, \citep{liu2020energy} proposes an Energy-bounded objective to further separate the two distributions. The OE objective is as follows,
\begin{equation}
\label{eq:energy_aux}
    \mathcal{L}_{\text{OE}} = \mathbb{E}_{\mathcal{D}^\text{s}_\text{in}}(\max(0, S_\text{Energy}(x,f) - m_\text{in}))^2 + \mathbb{E}_{\mathcal{D}^\text{s}_\text{out}}(\max(0, m_\text{out} - S_\text{Energy}(x,f)))^2
\end{equation}
We keep the thresholds same to \citep{liu2020energy}: $m_\text{in} = -25.0$, $m_\text{out} = -7.0$.

\paragraph{POEM.} \citep{ming2022poem} explores the Thompson sampling strategy \citep{thompson} to make the most use of outliers to learn a tight decision boundary. Though given the POEM's nature to be orthogonal to other OE methods, we use the Energy(w. $\mathcal{D}_{\text{aux}}$) as the backbone, which is the same as Eq.(~\ref{eq:energy_aux}) in \citet{liu2020energy}. The details of Thompson sampling can refer to~\citet{ming2022poem}.

\paragraph{FPR and TPR.} Suppose we have a binary classification task (to predict an image to be an ID or OOD sample in this paper). There are two possible outputs: a positive result (the model predicts an image to be an ID sample); a negative result (the model predicts an image to be an OOD sample). Since we have two possible labels and two possible outputs, we can form a confusion matrix with all possible outputs as follows,

\begin{table}[h!]
    \caption{Confusion Matrix.}
    \vspace{3mm}
    \centering
    \footnotesize
    \begin{tabular}{c|c|c}
         \toprule[0.6pt]
         & Truth: ID & Truth: OOD\\
        \midrule[0.6pt]
         Predict: ID & True Positive (TP) & False Positive (FP)\\
        \midrule[0.6pt]
         Predict: OOD & False Negative (FN) & True Negative (TN)\\
         \bottomrule[0.6pt]
    \end{tabular}
    \label{tab:confusion}
\end{table}

The false positive rate (FPR) is calculated as:
\begin{equation}
\label{eq:fpr}
    \text{FPR} = \frac{FP}{FP +TN}
\end{equation}
The true positive rate (TPR) is calculated as:
\begin{equation}
\label{eq:tpr}
    \text{TPR} = \frac{TP}{TP + FN}
\end{equation}

\paragraph{Margin value.} Let $f(x): \mathbb{R}^d \rightarrow \mathbb{R}^k$ be a model that outputs $k$ logits, following previous works~\citep{koltchinskii2002empirical,cao2019learning}, the margin value of an example (x,y) used in our Figure~\ref{fig2:c} is defined as,
\begin{equation}
\label{eq:margin}
    S_\text{margin}(x,y) = f(x)_y-\max_{j\neq y}f(x)_j
\end{equation}

\section{Theoretical Insights on ID Data Property}
\label{app:theo}

In this section, we provide a detailed discussion and theoretical analysis to explain the revealed observation and the benefits of our proposed method on ID data property. Specifically, we present the analysis based on the view of sample complexity adopted in POEM~\citep{ming2022poem}. To better demonstrate the conceptual extension, we also provide an intuitive illustration based on a comparison with POEM's previous focus on auxiliary outlier sampling in Figure~\ref{fig:theo} (extended version of Figure~\ref{fig5:theo_ill}). Briefly, we focus on the ID data property which is not discussed in the previous analytical framework.

\paragraph{Preliminary setup and notations.} As the original training task (e.g., the multi-classification task on CIFAR-10) does not involve any outliers data, it is hard to analyze the related property with OOD detection. Here we introduce an Assumption~\ref{assump:daux} about virtual $\mathcal{D}_\text{aux}$ to help complete the analytical framework. 
To sum up, we consider a binary classification task here for distinguishing ID and OOD data. Following the prior works~\citep{LeeLLS18,SehwagCM21,ming2022poem}, we assume the extracted feature approximately follows a Gaussian mixture model (GMM) with the equal class priors as $\frac{1}{2}\mathcal{N}(\mu,\sigma^2\mathcal{I})+\frac{1}{2}\mathcal{N}(-\mu,\sigma^2\mathcal{I})$. To be specific,  $\mathcal{D}_\text{in}=\mathcal{N}(\mu, \sigma^2\mathcal{I})$ and $\mathcal{D}_\text{aux}=\mathcal{N}(-\mu, \sigma^2\mathcal{I})$. Considering the hypothesis class as $\mathcal{H}={\text{sign}(\theta^Tx), \theta\in\mathbb{R}^d}$. The classifier outputs 1 if $x\sim\mathcal{D}_\text{in}$ and outputs -1 if $x\sim\mathcal{D}_\text{aux}$.


\begin{figure*}[t!]
    \begin{center}
    \subfigure[$\mathcal{D}_\text{in}$ as an Anchor]{
    \includegraphics[scale=0.135]{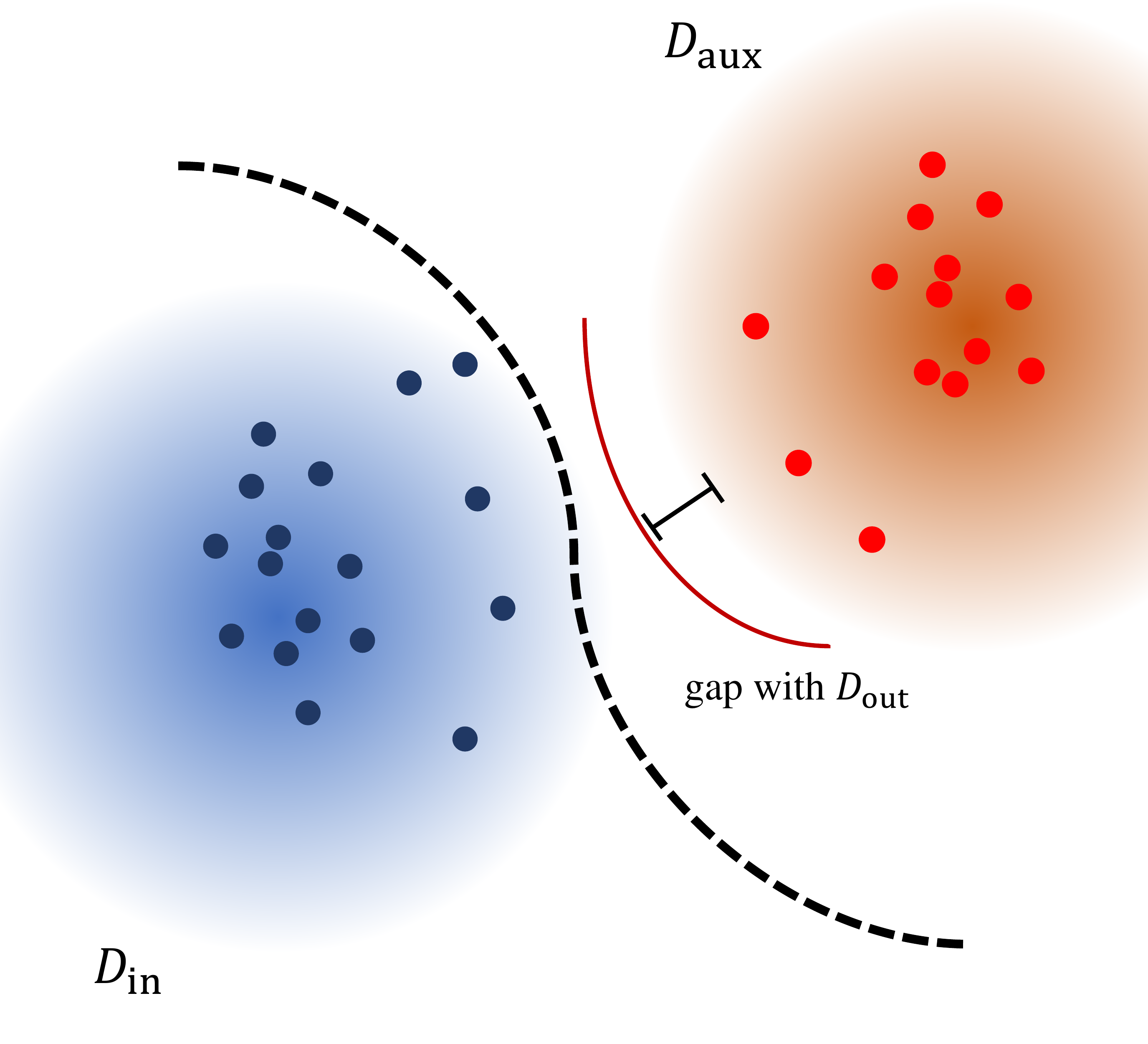}
    \label{fig6:poem_theo}
    }
    \subfigure[Under-represent on $\mathcal{D}_\text{in}$]{
    \includegraphics[scale=0.135]{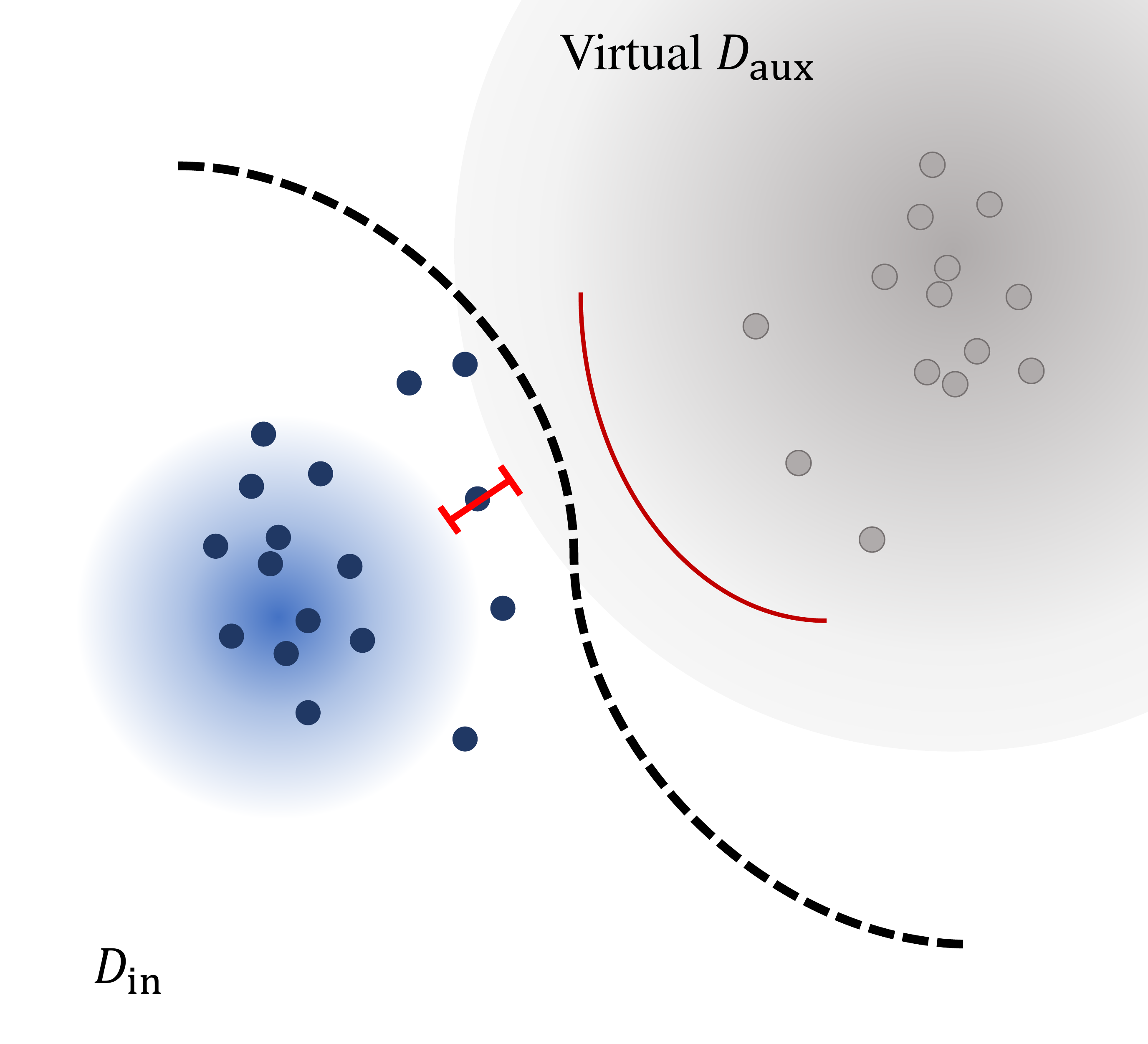}
    \label{fig6:ours_theo1}
    }
    \subfigure[Optimal-represent on $\mathcal{D}_\text{in}$]{
    \includegraphics[scale=0.135]{daux_as_anchor2.pdf}
    \label{fig6:ours_theo2}
    }
    \subfigure[Over-represent on $\mathcal{D}_\text{in}$]{
    \includegraphics[scale=0.135]{daux_as_anchor3.pdf}
    \label{fig6:ours_theo3}
    }
    \end{center}
    \caption{Illustration about the theoretical insights on ID data property considering the binary classification scenario, which is presented as a conceptual comparison based on the underlying intuition in POEM~\citep{ming2022poem}. Different from treating the ID distribution $\mathcal{D}_\text{in}$ as an analytical anchor in POEM, we present three conceptual visualizations which correspond to different training phases in model development on ID data. (a) Using $\mathcal{D}_\text{in}$ as an anchor, the boundary data (defined in \citet{ming2022poem}) sampled from $\mathcal{D}_\text{aux}$ is important to mitigate the distribution gap with the true $\mathcal{D}_\text{out}$ (as indicated with the red arc).  Without the $\mathcal{D}_\text{aux}$, we assume a virtual $\mathcal{D}_\text{aux}$ exists for the analytical target, which is highly related to the $\mathcal{D}_\text{in}$ and the model in the original classification task on $\mathcal{D}_\text{in}$. (b) the model under-represent on $\mathcal{D}_\text{in}$ with the lower confidence on the atypical samples close to the boundary; (c) the model reaches a near-optimal representation status on $\mathcal{D}_\text{in}$; (d) the model over-represent on $\mathcal{D}_\text{in}$. The corresponding OOD discriminative capability is affected by the different scenarios.
    }
    \label{fig:theo}
    \vspace{-2mm}
\end{figure*}

First, we introduce the assumption about virtual $\mathcal{D}_\text{aux}$. Considering the representation power of deep neural networks, the assumption can be valid. It is empirically supported by the evidence in Figure~\ref{fig5:theo_ill}, as the part of real $\mathcal{D}_\text{out}$ can be viewed as the virtual $\mathcal{D}_\text{aux}$. Second, to better link our method for the analysis, we introduce another assumption (i.e., Assumption~\ref{assump:masking}) about the ID training status. It can be verified by the relative degree of distinguishability indicated by a fixed threshold in Figure~\ref{fig5:theo_ill}, that the model is more confident on the $\mathcal{D}_\text{out}$ along with the training.

\begin{assumption}[Virtual $\mathcal{D}_\text{aux}$]
Given the well-trained model in the original classification task on the ID distribution $\mathcal{D}_\text{in}$, and considering the binary classification for OOD detection, we can assume the existence of a virtual $\mathcal{D}_\text{aux}$, that the OOD discriminative capacity of the current model can result from learning on the virtual $\mathcal{D}_\text{aux}$ with the outlier exposure manner.
\label{assump:daux}
\end{assumption}

\begin{assumption}[ID Training Status w.r.t. Masking]
Considering the model training phase in the original multi-class classification task on the ID distribution $\mathcal{D}_\text{in}$, and tuning with a specific mask ratio serving as the sample selection, we assume that the data points $x\sim$ virtual $\mathcal{D}_\text{aux}$ satisfy the extended constraint based on the boundary scores $-|f_\text{outlier}(x)|$ defined in POEM~\cite{ming2022poem}: $\sum^{n}_{i=1}f_\text{outlier}\leq (|\alpha-\delta^*|+\epsilon)n$, where the $f_\text{outlier}$ is a function parameterized by some unknown ground truth weights and maps the high-dimensional input x into a scalar. Generally, it represents the discrepancy between virtual $\mathcal{D}_\text{aux}$ and the true $\mathcal{D}_\text{out}$, indicated with the constraint $|\alpha-\delta^*|$ results from the masked ID data.
\label{assump:masking}
\end{assumption}

Given the above, we can naturally get the following extended lemma based on that adopted in POEM~\citep{ming2022poem}.

\begin{lemma}[Constraint of Varied Virtual $\mathcal{D}_\text{aux}$]
\label{app: lemma1}
Assume the data points $x$$\sim$ virtual $\mathcal{D}_\text{aux}$ satisfy the following constraint for resulting in the following varied boundary margin: $\sum_{i=1}^n|2x_i^T\mu|\leq n \sigma^2 (|\alpha-\delta^*|+\epsilon)$. 
\end{lemma}

\begin{proof}[proof of Lemma.~\ref{app: lemma1}]
Given the Gaussian mixture model described in the previous setup, we can obtain the following expression by Bayes' rule of $\mathbb{P}(\text{outlier}|x)$,
\begin{align}
    \mathbb{P}(\text{outlier}|x) = \frac{ \mathbb{P}(x|\text{outlier}) \mathbb{P}(\text{outlier})}{ \mathbb{P}(x)} = \frac{1}{1+e^{-\frac{1}{2\sigma^2}(d_\text{outlier}(x)-d_\text{in}(x))}},
\end{align}
where $d_\text{outlier}(x))=(x+\mu)^\top(x+\mu), d_\text{in}(x)=(x-\mu)^\top(x-\mu)$, and $\mathbb{P}(\text{outlier}|x)=\frac{1}{1+e^{-f_\text{outlier}(x)}}$ according to its definition. Then we have:
\begin{align}
    -f_\text{outlier} = -\frac{1}{2\sigma^2}&(d_\text{outlier}(x)-d_\text{in}(x)),\\
    -|f_\text{outlier}| = -\frac{1}{2\sigma^2}|(x-\mu)^\top(x-\mu)&-(x+\mu)^\top(x+\mu)|=-\frac{2}{\sigma^2}|x^\top \mu|.
\end{align}
Therefore, we can get the constraint as: $\sum_{i=1}^n|2x_i^T\mu|\leq n \sigma^2 (|\alpha-\delta^*|+\epsilon)$.
\end{proof}

With the previous assumption and lemma that incorporate our masking in the variable $\delta^*$, we present the analysis as below.

\paragraph{Complexity analysis anchored on $\mathcal{D}_\text{in}$.} With the above lemma and the assumptions of virtual $\mathcal{D}_\text{aux}$ (as illustrated in Figure~\ref{fig:theo}), we can derive the results to understand the benefits from the revealed observation and our UM and UMAP.

Consider the given classifier defined as $\theta^*_{n_1,n_2}=\frac{1}{n_1+n_2}(\sum_{i=1}^{n_1} x_i^1-\sum_{i=1}^{n_2} x_i^2)$, assume each $x_i^1$ is drawn \textit{i.i.d.} from $\mathcal{D}_\text{in}$ and each $x_i^2$ is drawn \textit{i.i.d} from $\mathcal{D}_\text{aux}$, and assume the signal/noise ratio is $\frac{||\mu||}{\sigma}=r_0\gg 1$, the dimentionality/sample size ratio is $\frac{d}{n}=r_1$, as well as exist some constant $\alpha<1$. By decomposition, we can rewrite $\theta^*_{n_1,n_2}=\mu+\frac{n_1}{n_1+n_2}\theta_\text{1}+\frac{n_2}{n_1+n_2}\theta_\text{2}$ with the following $\theta_\text{1}$ and $\theta_\text{2}$:
\begin{equation}
\label{eq:theta_decomp}
    \theta_\text{1} = \frac{1}{n_1}(\sum_{i=1}^{n_1}x_i^1)-\mu,\quad \theta_\text{2} = \frac{1}{n_2}(-\sum_{i=1}^{n_2}x_i^2)-\mu,
\end{equation}
Since $\theta_\text{1} \sim \mathcal{N}(0, \frac{\sigma^2}{n_1}\mathcal{I})$, we have that $||\theta_1||^2\sim\frac{\sigma^2}{n_1}\mathcal{X}_d^2$ and $\frac{\mu^T\theta_1}{||\mu||}\sim\mathcal{N}(0, \frac{\sigma^2}{n_1})$ to form the standard concentration bounds as:
\begin{equation}
\label{eq:concentration}
    \mathbb{P}(||\theta_1||^2\geq\frac{\sigma^2}{n_1}(d+\frac{1}{\sigma}))\leq e^{-\frac{d}{8\sigma^2}},\quad \mathbb{P}(\frac{|\mu^T\theta_1|}{||\mu||}\geq(\sigma||\mu||)^{\frac{1}{2}})\leq 2e^{-\frac{n_1||\mu||}{2\sigma}}
\end{equation}

Anchored on $\mathcal{D}_\text{in}$, the distribution of $\theta_2$ can be treated as a truncated distribution of $\theta_1$ as $x_i^2$ drawn \textit{i.i.d.} from the virtual $\mathcal{D}_\text{aux}$ are under the relative constraint with $\mathcal{D}_\text{in}$. Without losing the generality, we replace $n_1$ with $n$, and have the following inequality with a finite positive constant $a$:
\begin{equation}
\label{eq:concentration2}
    \mathbb{P}(||\theta_2||^2\geq\frac{\sigma^2}{n_1}(d+\frac{1}{\sigma}))\leq ae^{-\frac{d}{8\sigma^2}}
\end{equation}
According to Lemma~\ref{app: lemma1}, we can have that $|\mu^T\theta_2|\leq||\mu||^2+\frac{\sigma^2(|\alpha-\delta^*|+\epsilon)}{2}$. Now we can have $||\theta_1||^2\leq\frac{\sigma^2}{n}(d+\frac{1}{\sigma})$, $||\theta_2||^2\leq\frac{\sigma^2}{n}(d+\frac{1}{\sigma})$, $\frac{|\mu^T\theta_1|}{||\mu||}\leq(\sigma||\mu||)^{\frac{1}{2}}$ simultaneously hold and derive the following recall the decomposition,
\begin{equation}
\label{eq:overalltheta}
    ||\theta^*_{n_1,n_2}||^2 = ||\mu+\frac{n_1}{n_1+n_2}\theta_1+\frac{n_2}{n_1+n_2}\theta_2||^2 \leq \frac{\sigma^2}{n}(d+\frac{1}{\sigma}) + ||\mu||^2,
\end{equation}
and
\begin{equation}
\label{eq:overallmutheta}
    |\mu^T\theta^*_{n_1,n_2}|\geq\frac{1}{2}(||\mu||^2-\sigma^{\frac{1}{2}}||\mu||^{\frac{3}{2}}-\frac{\sigma^2(|\alpha-\delta^*|+\epsilon)}{2}).
\end{equation}
With the above inequality derived in Eq.~(\ref{eq:overalltheta}) and Eq.~(\ref{eq:overallmutheta}), we can have the following bound with the probability at least $1-(1+a)e^{-\frac{r_1n}{8\sigma^2}}-2e^{-\frac{n_1||\mu||}{2\sigma}}$
\begin{equation}
    \frac{\mu^T\theta^*_{n_1,n_2}}{\sigma||\theta^*_{n_1,n_2}||} \geq \frac{||\mu||^2-\sigma^{\frac{1}{2}}||\mu||^{\frac{3}{2}}-\frac{\sigma^2(|\alpha-\delta^*|+\epsilon)}{2}}{2\sqrt{\frac{\sigma^2}{n}(d+\frac{1}{\sigma}) + ||\mu||^2}}
\end{equation}
Since $\text{FPR}(\theta^*_{n_1,n_2})=erf(\frac{\mu^T\theta^*_{n_1,n_2}}{\sigma||\theta^*_{n_1,n_2}||})$ is monotonically decreasing, as the lower bound of $\frac{\mu^T\theta^*_{n_1,n_2}}{\sigma||\theta^*_{n_1,n_2}||}$ will increase as the constraint from the virtual $\mathcal{D}_\text{aux}$ changed accordingly in our UM and UMAP, the upper bound of $\text{FPR}(\theta^*_{n_1,n_2})$ will decrease.
From the above analysis, one insight we can draw is learning more atypical ID data may need more high-quality auxiliary outliers to shape the near-the-boundary behavior of the model, which can further enhance the OOD discriminative capability.

\section{Discussion about the "Conflict" Against Previous Empirical Observation}
\label{app:conflict}

In this section, we address what initially appears to be a contradiction between our observation and previous empirical studies~ \citep{vaze2022openset,fort2021exploring}, but it is not a contradiction. This work demonstrates that during training, there exists a middle stage where the model's OOD detection performance is superior to the final stage, even though the model has not achieved the best performance on ID-ACC. Some previous studies \citep{vaze2022openset,fort2021exploring} suggest that a good close-set classifier tends to have higher OOD detection performance, which may seem to contradict our claim. However, this is not the case, and we provide the following explanations.

First, the previous empirical observation \citep{vaze2022openset,fort2021exploring} of a high correlation between a good close-set classifier (e.g., high ID-ACC in \citep{vaze2022openset}) and OOD detection performance is based on \textbf{inter-model comparisons}, such as comparing different model architectures. This is consistent with our results in Table~\ref{tab:conflict1}. Even the previous model stages backtracked via our UM show similar results confirming that a better classifier (e.g., the DenseNet-101 in Table~\ref{tab:conflict1}) is better to achieve better OOD detection performance.

Second, our observation is based on \textbf{intra-model comparisons}, which compare different training stages of a single model. Our results in Figure~\ref{fig: motivation_1} across various training settings confirm this observation. Additionally, Table~\ref{tab:conflict2} shows that when we backtrack the model through UM, we obtain lower ID-ACC but better OOD detection performance. However, if we compare different models, DenseNet-101 with higher ID-ACC still outperforms Wide-ResNet, as previously mentioned.

To summarize, our observation provides an orthogonal view to exploring the relationship between ID-ACC and OOD detection performance. On the one hand, we attribute this observation to the model's memorization of atypical samples, as further demonstrated by our experiments (e.g., in Figure~\ref{fig: motivation_2}). On the other hand, we believe that this observation reveals other characteristics of a "good classifier" beyond ID-ACC, e.g., higher OOD detection capability.

\begin{table}[h!]
    \caption{\textbf{Inter-model comparison} (different models) of ID-ACC with the OOD detection performance on CIFAR-10 ($\%$). $\uparrow$ indicates higher values are better, and $\downarrow$ indicates lower values are better.}
    \vspace{2mm}
    \centering
    \footnotesize
    \begin{tabular}{c|l|cccc}
        \toprule[1.5pt]
        Method &  Model & AUROC$\uparrow$ & AUPR$\uparrow$ & FPR95$\downarrow$ & ID-ACC$\uparrow$ \\
        \midrule[0.6pt]
        \multirow{2}*{MSP}
         & DenseNet-101 & $\textbf{89.90}\pm\textbf{0.30}$ & $\textbf{91.48}\pm\textbf{0.43}$ & $\textbf{60.08}\pm\textbf{0.76}$ & $\textbf{94.01}\pm\textbf{0.08}$\\
         & WRN-40-4 & $87.12\pm0.25$ & $87.84\pm0.30$ & $68.29\pm0.96$ & $93.86\pm0.19$\\
         \midrule[0.6pt]
        \multirow{2}*{ODIN}
         & DenseNet-101 & $\textbf{91.46}\pm\textbf{0.56}$ & $\textbf{91.67}\pm\textbf{0.58}$ & $\textbf{42.31}\pm\textbf{1.38}$ & $\textbf{94.01}\pm\textbf{0.08}$\\
         & WRN-40-4 & $83.29\pm0.72$ & $82.74\pm0.79$ & $65.68\pm0.77$ & $93.86\pm0.19$\\
         \midrule[0.6pt]
         \multirow{2}*{Energy}
         & DenseNet-101 & $\textbf{92.07}\pm\textbf{0.22}$ & $\textbf{92.72}\pm\textbf{0.39}$ & $\textbf{42.69}\pm\textbf{1.31}$ & $\textbf{94.01}\pm\textbf{0.08}$\\
         & WRN-40-4 & $87.69\pm0.54$ & $88.16\pm0.69$ & $58.47\pm1.94$ & $93.86\pm0.19$\\
         \midrule[0.6pt]
         \multirow{2}*{Energy+\textbf{UM} (ours)}
         & DenseNet-101 & $\textbf{93.73}\pm\textbf{0.36}$ & $\textbf{94.27}\pm\textbf{0.60}$ & $\textbf{33.29}\pm\textbf{1.70}$ & $\textbf{92.80}\pm\textbf{0.47}$\\
         & WRN-40-4 & $91.74\pm0.43$ & $92.67\pm0.52$ & $40.40\pm1.32$ & $92.68\pm0.23$\\
        \bottomrule[1.5pt]
    \end{tabular}
    \label{tab:conflict1}
\end{table}

\begin{table}[h!]
    \caption{\textbf{Intra-model comparison }(regarding the same model) of ID-ACC with the OOD detection performance on CIFAR-10 ($\%$). $\uparrow$ indicates higher values are better, and $\downarrow$ indicates lower values are better.}
    \vspace{2mm}
    \centering
    \footnotesize
    \begin{tabular}{c|l|cccc}
        \toprule[1.5pt]
        Method &  Model & AUROC$\uparrow$ & AUPR$\uparrow$ & FPR95$\downarrow$ & ID-ACC$\uparrow$ \\
        \midrule[0.6pt]
        {Energy} & \multirow{2}*{DenseNet-101} & $92.07\pm0.22$ & $92.72\pm0.39$ & $42.69\pm1.31$ & $\textbf{94.01}\pm\textbf{0.08}$\\
         \cmidrule[0.6pt]{1-1}\cmidrule[0.6pt]{3-6}
         {Energy+\textbf{UM} (ours)} &  & $\textbf{93.73}\pm\textbf{0.36}$ & $\textbf{94.27}\pm\textbf{0.60}$ & $\textbf{33.29}\pm\textbf{1.70}$ & $92.80\pm0.47$\\
         \midrule[0.6pt]
         {Energy} & \multirow{2}*{WRN-40-4} & $87.69\pm0.54$ & $88.16\pm0.69$ & $58.47\pm1.94$ & $\textbf{93.86}\pm\textbf{0.19}$\\
         \cmidrule[0.6pt]{1-1}\cmidrule[0.6pt]{3-6}
         {Energy+\textbf{UM} (ours)} &  & $\textbf{91.74}\pm\textbf{0.43}$ & $\textbf{92.67}\pm\textbf{0.52}$ & $\textbf{40.40}\pm\textbf{1.32}$ & $92.68\pm0.23$\\
        \bottomrule[1.5pt]
    \end{tabular}
    \label{tab:conflict2}
\end{table}

\section{Discussion with Conventional Overfitting}
\label{app:overfitting_comp}

In this section, we provide a comprehensive comparison of our observation and conventional overfitting in deep learning.

First, we would refer to the concept of the conventional overfitting \citep{goodfellow2016deep,doi:10.1073/pnas.1903070116}, i.e., the model "overfits" the training data but fails to generalize and perform well on the test data that is unseen during training. The common empirical reflection of overfitting is that the training error is decreasing while the test error is increasing at the same time, which enlarges the generalization gap of the model. It has been empirically confirmed not the case in our observation as observed in Figure~\ref{fig2:a} and~\ref{fig2:b}. To be specific, for the original classification task, there is no conventional overfitting observed as the test performance is still improved at the later training stage, which is a general pursuit of the model development phase on the original tasks~\citep{goodfellow2016deep,zhang2016understanding}.

Then, when we consider the OOD detection performance of the well-trained model, our unique observation is about the inconsistency between gaining better OOD detection capability and pursuing better performance on the original classification task for the in-distribution (ID) data. It is worth noting that here the training task is not the binary classification of OOD detection, but the classification task on ID data. It is out of the rigorous concept of conventional overfitting and has received limited focus and discussion through the data-level perspective in the previous literature about OOD detection~\citep{yang2021generalized,yangopenood2022} to the best of our knowledge. Considering the practical scenario that exists target-level discrepancy, our revealed observation may encourage us to revisit the detection capability of the well-trained model.

Third, we also provide an empirical comparison with some strategies targeted for mitigating overfitting. In our experiments, for all the baseline models including that used in Figure~\ref{fig: motivation_1}, we have adopted those strategies \citep{booktitles_jmlr_dropout, TheElementsHastie2009} (e.g., drop-out, weight decay) to reduce overfitting. 
The results are summarized in the following Tables~\ref{tab:overfitting_odin_densenet}, \ref{tab:overfitting_energy_densenet}, \ref{tab:overfitting_odin_wrn} and~\ref{tab:overfitting_energy_wrn}. According to the experiments, most conventional methods proposed to prevent conventional overfitting show limited benefits in gaining better OOD detection performance, since they have a different underlying target from UM/UMAP. However, most of them suffer from the higher sacrifice on the performance of the original task and may not be compatible and practical in the current general setting, i.e., starting from a well-trained model.
In contrast, our proposed UMAP can be a more practical and flexible way to restore detection performance.


\begin{table}[h!]
    \caption{Comparison among overfitting methods and ODIN with DenseNet-101 ($\%$). $\uparrow$ indicates higher values are better, and $\downarrow$ indicates lower values are better.}
    \vspace{2mm}
    \centering
    \footnotesize
    \begin{tabular}{c|l|cccc}
        \toprule[1.5pt]
        $\mathcal{D}_\text{in}$ &  Method & AUROC$\uparrow$ & AUPR$\uparrow$ & FPR95$\downarrow$ & ID-ACC$\uparrow$ \\
        \midrule[0.6pt]
        \multirow{11}*{\textbf{CIFAR-10}}
         & Baseline & $91.67 $ & $91.89 $ & $40.74 $ & $93.67 $\\
         & Early Stopping w. ACC & $92.13 $ & $92.46 $ & $38.86 $ & $93.69 $\\
         & Weight Decay 0.1 & $86.64 $ & $86.67 $ & $60.07 $ & $88.53 $\\
         & Weight Decay 0.01 & $90.76 $ & $91.25 $ & $44.20 $ & $92.07 $\\
         & Weight Decay 0.001 & $88.93 $ & $88.25 $ & $48.95 $ & $94.26 $\\
         & Drop Rate 0.3 & $91.14 $ & $92.21 $ & $46.58 $ & $90.05 $\\
         & Drop Rate 0.4 & $84.95 $ & $86.62 $ & $62.52 $ & $82.55 $\\
         & Drop Rate 0.5 & $83.75 $ & $85.17 $ & $62.17 $ & $75.31 $\\
         \cmidrule{2-6}
         & \textbf{UM} (ours) & $92.45 $ & $93.06 $ & $37.13 $ & $92.76 $\\
         & \textbf{UMAP} (ours) & $91.92 $ & $92.88 $ & $37.69 $ & $93.69 $\\
        \bottomrule[1.5pt]
    \end{tabular}
    \label{tab:overfitting_odin_densenet}
\end{table}

\begin{table}[h!]
    \caption{Comparison among overfitting methods and Energy with DenseNet-101 ($\%$). $\uparrow$ indicates higher values are better, and $\downarrow$ indicates lower values are better.}
    \vspace{2mm}
    \centering
    \footnotesize
    \begin{tabular}{c|l|cccc}
        \toprule[1.5pt]
        $\mathcal{D}_\text{in}$ &  Method & AUROC$\uparrow$ & AUPR$\uparrow$ & FPR95$\downarrow$ & ID-ACC$\uparrow$ \\
        \midrule[0.6pt]
        \multirow{11}*{\textbf{CIFAR-10}}
         & Baseline & $92.72 $ & $93.48 $ & $38.30 $ & $93.67 $\\
         & Early Stopping w. ACC & $92.75 $ & $93.54 $ & $37.84 $ & $93.69 $\\
         & Weight Decay 0.1 & $86.78 $ & $88.04 $ & $65.08 $ & $88.53 $\\
         & Weight Decay 0.01 & $90.86 $ & $91.77 $ & $47.64 $ & $92.07 $\\
         & Weight Decay 0.001 & $90.68 $ & $90.90 $ & $47.38 $ & $94.26 $\\
         & Drop Rate 0.3 & $90.52 $ & $91.79 $ & $51.23 $ & $90.05 $\\
         & Drop Rate 0.4 & $84.29 $ & $86.43 $ & $68.17 $ & $82.55 $\\
         & Drop Rate 0.5 & $83.29 $ & $85.14 $ & $68.17 $ & $75.31 $\\
         \cmidrule{2-6}
         & \textbf{UM} (ours) & $93.58 $ & $94.14 $ & $33.66 $ & $92.76 $\\
         & \textbf{UMAP} (ours) & $93.17 $ & $93.87 $ & $36.11 $ & $93.69 $\\
        \bottomrule[1.5pt]
    \end{tabular}
    \label{tab:overfitting_energy_densenet}
\end{table}

\begin{table}[h!]
    \caption{Comparison among overfitting methods and ODIN with WRN-40-4 ($\%$). $\uparrow$ indicates higher values are better, and $\downarrow$ indicates lower values are better.}
    \vspace{2mm}
    \centering
    \footnotesize
    \begin{tabular}{c|l|cccc}
        \toprule[1.5pt]
        $\mathcal{D}_\text{in}$ &  Method & AUROC$\uparrow$ & AUPR$\uparrow$ & FPR95$\downarrow$ & ID-ACC$\uparrow$ \\
        \midrule[0.6pt]
        \multirow{11}*{\textbf{CIFAR-10}}
         & Baseline & $86.24 $ & $85.90 $ & $60.13 $ & $93.86 $\\
         & Early Stopping w. ACC & $83.80 $ & $83.30 $ & $65.13 $ & $93.99 $\\
         & Weight Decay 0.1 & $84.38 $ & $84.75 $ & $65.75 $ & $89.88 $\\
         & Weight Decay 0.01 & $88.08 $ & $88.45 $ & $55.16 $ & $93.16 $\\
         & Weight Decay 0.001 & $86.34 $ & $86.38 $ & $57.42 $ & $94.91 $\\
         & Drop Rate 0.3 & $87.53 $ & $87.25 $ & $56.12 $ & $94.22 $\\
         & Drop Rate 0.4 & $88.24 $ & $88.41 $ & $54.62 $ & $94.20 $\\
         & Drop Rate 0.5 & $89.13 $ & $89.99 $ & $53.07 $ & $93.91 $\\
         \cmidrule{2-6}
         & \textbf{UM} (ours) & $89.61 $ & $91.13 $ & $50.97 $ & $92.68 $\\
         & \textbf{UMAP} (ours) & $90.43 $ & $91.73 $ & $46.96 $ & $93.86 $\\
        \bottomrule[1.5pt]
    \end{tabular}
    \label{tab:overfitting_odin_wrn}
\end{table}

\begin{table}[h!]
    \caption{Comparison among overfitting methods and Energy with WRN-40-4 ($\%$). $\uparrow$ indicates higher values are better, and $\downarrow$ indicates lower values are better.}
    \vspace{2mm}
    \centering
    \footnotesize
    \begin{tabular}{c|l|cccc}
        \toprule[1.5pt]
        $\mathcal{D}_\text{in}$ &  Method & AUROC$\uparrow$ & AUPR$\uparrow$ & FPR95$\downarrow$ & ID-ACC$\uparrow$ \\
        \midrule[0.6pt]
        \multirow{11}*{\textbf{CIFAR-10}}
         & Baseline & $87.69 $ & $88.16 $ & $58.47 $ & $93.86 $\\
         & Early Stopping w. ACC & $88.07 $ & $88.65 $ & $67.61 $ & $93.99 $\\
         & Weight Decay 0.1 & $86.97 $ & $88.51 $ & $63.54 $ & $89.88 $\\
         & Weight Decay 0.01 & $89.77 $ & $89.82 $ & $50.23 $ & $93.16 $\\
         & Weight Decay 0.001 & $89.25 $ & $89.84 $ & $50.95 $ & $93.91 $\\
         & Drop Rate 0.3 & $89.74 $ & $90.07 $ & $52.16 $ & $93.22 $\\
         & Drop Rate 0.4 & $89.94 $ & $90.53 $ & $51.13 $ & $94.20 $\\
         & Drop Rate 0.5 & $90.09 $ & $91.04 $ & $52.76 $ & $93.91 $\\
         \cmidrule{2-6}
         & \textbf{UM} (ours) & $91.74 $ & $92.67 $ & $40.40 $ & $92.68 $\\
         & \textbf{UMAP} (ours) & $88.84 $ & $89.31 $ & $50.23 $ & $93.86 $\\
        \bottomrule[1.5pt]
    \end{tabular}
    \label{tab:overfitting_energy_wrn}
\end{table}

Given the concept discrepancy aforementioned, we can know that "memorization of the atypical samples" are not "memorization in overfitting". Those atypical samples are empirically beneficial in improving the performance on the original classification task as shown in Figure~\ref{fig: motivation_2}. However, this part of knowledge is not very necessary and even harmful to the OOD detection task as the detection performance of the model drops significantly. Based on the training and test curves in our observation, the memorization in overfitting is expected to happen later than the final stage in which the test performance would drop. Since we have already used some strategies to prevent overfitting, it does not exist. Intuitively, the "atypical samples" identified in our work are relative to the OOD detection task. The memorization of "atypical samples" indicates that the model may not be able to draw the general information of the ID distribution through further learning on those atypical samples through the original classification task. Since we mainly provide the understanding of the data-level attribution for OOD discriminative capability, further analysis from theoretical views~\citep{fang2022is} to link the conventional overfitting with OOD detection would be an interesting future direction.

\section{Additional Explanation Towards Mining the Atypical Samples}
\label{app:atypical_mining}

In this section, we provide further discussion and explanation about mining the atypical samples.

First, for identifying those atypical samples using a randomly initialized layer-wise mask~\citep{ramanujan2020s} with the well-pre-trained model, the underlying intuition is constructing the parameter-level discrepancy to mine the atypical samples. It is inspired by and based on the evidence drawn from previous literature about learning behaviors \citep{arpit2017closer,goodfellow2016deep} of deep neural networks (DNNs), sparse representation \citep{DBLP:journals/corr/abs-1803-03635,10.1007/978-3-642-42051-1_16,barham2022pathways}, and also model uncertainty representation (like dropout~\citep{gal2016dropout}). To be specific, the atypical samples tend to be learned by the DNNs later than those typical samples \citep{arpit2017closer}, and are relatively more sensitive to the changes of the model parameter as the model does not generalize well on that~\citep{booktitles_jmlr_dropout,gal2016dropout}. By the layer-wise mask, the constructed discrepancy can make the model misclassify the atypical samples and estimate loss constraint for the forgetting objective, as visualized in Figure~\ref{fig:method}.

Second, introducing the layer-wise mask has several advantages for achieving the staged target of mining atypical samples in our proposed method, while we would also admit that the layer-wise mask may not be an irreplaceable option or may not be optimal. On the one hand, considering that the model has been trained to approach the zero error on training data, utilizing the layer-wise mask is an integrated strategy to 1) figure out the atypical samples; and 2) obtain the loss value computed by the masked output that misclassifies them. The loss constraint is later used in the forgetting objective to fine-tune the model. On the other hand, the layer-wise mask is also compatible with the proposed UMAP to generate a flexible mask for restoring the detection capability of the original model.

\paragraph{More discussion and visualization using CIFAR-10 and ImageNet.} Third, we also adopt the unit/weight mask \citep{han2015deep} and visualize the misclassified samples in Figure~\ref{fig17:mine_atypical} (we also present a similar visualization about the experiments on ImageNet~\citep{deng2009imagenet} in Figure~\ref{fig18:mine_atypica_imagenetl}). The detected samples show that traditionally pruning the network according to weights can't efficiently figure out whether an image is typical or atypical while pruning randomly can do so. Intuitively, we attribute this phenomenon to the uncertain relationship~\citep{gal2016dropout} between the magnitudes and the learned patterns. Randomly masking out weights can have a harsh influence on atypical samples, which creates a discrepancy in mining them. Further investigating the specific effect of different methods that construct the parameter-level discrepancy would be an interesting sub-topic in future work. For the value of CE loss, although the atypical samples tend to have high CE loss value, they are already memorized and correctly classified as indicated by the zero training error. Only using the high CE error can not provide the loss estimation when the model does not correctly classify those samples.

\clearpage
\section{Algorithmic Realization of UM and UMAP}
\label{app:algo_realization}

In this section, we provide the detailed algorithmic realizations of our proposed Unleashing Mask (UM) (i.e., in Algorithm~\ref{alg:um}) and Unleashing Mask Adopt Pruning (UMAP) (i.e., in Algorithm~\ref{alg:umap}) given the well-trained model. 

In general, we seek to unleash the intrinsic detection power of the well-trained model by adjusting the well-trained given model. For the first part, we need to mine the atypical samples and estimate the loss value to misclassify them using the current model. For the second part, we need to tune or prune with the loss constraint for forgetting.

To estimate the loss constrain for forgetting (i.e., $\widehat{\ell}_\text{CE}(m_\delta\odot f^*)$ in Eq \ref{eq:obj} with the fixed given model $f^*$), we randomly knock out parts of weights according to a specific mask ratio $\delta$. To be specific, we sample a score from a Gaussian distribution for every weight. Then we initialize a unit matrix for every layer of the model concerning the size of the layer. We formulate the mask $m_\delta$ according to the sampled scores. We then iterate through every layer (termed as $l \in \theta_\text{layers}$) to find the threshold for each layer that is smaller than the score of the given mask ratio in that layer (termed as $\text{quantile}$). Then set all the ones, whose corresponding scores are more significant than the layers' thresholds, to zeros. 

We dot-multiply every layer's weights with the formulated binary matrix as if we delete some parts of the weights. Then, we input a batch of training samples to the masked model and treat the mean value of the outputs' CE loss as the loss constraint. After all of these have been done, we begin to fine-tune the model's weights with the loss constraint applied to the original CE loss. In our algorithms, the fine-tuning epochs $k$ is the epochs we finetune after we get the well-trained model.

For UMAP, the major difference from UM is that, instead of fine-tuning the weights, we generate a popup score for every weight, and force the gradients to pass through the scores. In every iteration, we need to formulate a binary mask according to the given prune rate $p$. This is just what we do when estimating the loss constraint. For more details, it can refer to \citep{ramanujan2020s}. In Table~\ref{tab:my_label_complete}, we summarize the complete comparison of UM and UMAP to show their effectiveness. We also provide the performance comparison by switching ID training data to be the large-scaled ImageNet, and demonstrate the effectiveness of our UM and UMAP in Table~\ref{tab:my_imagenet}. In practice, we use SVHN as a validation OOD set to tune the mask ratio. We adopt $99.6\%$ (the corresponding estimated loss constraint is about $0.6$) in this large-scale experiment to estimate the loss constraint for forgetting. Surprisingly, we also find that loss values smaller than the estimated one (i.e., $<0.6$) can also help improve OOD detection performance, distinguishing the general effectiveness of UM/UMAP.


\begin{algorithm}[h!]
   \caption{Unleashing Mask (UM)}
   \label{alg:um}
   {\bf Input:} well-trained model : $\theta$, Gaussian distribution: $N(\mu, \sigma^2)$, mask ratio : $\delta \in [0, 1]$, fine-tuning epochs of UM : $k$, training samples : $x \sim \mathcal{D}^\text{s}_\text{in}$, layer-iterated model : $\theta_{\text{layer}}$, compute the $\delta$-th quantile of the data $s$ : $\text{quantile}(s, \delta)$;\\
   {\bf Output:} fine-tuned model $\theta^k$;
\begin{algorithmic}[1]
    \STATE \begin{footnotesize}\texttt{// Initialize a popup score for every weight}\end{footnotesize}\vspace{2mm}
    \FOR{$w \in \theta$}
        \STATE {$s_w \sim N(\mu, \sigma^2)$}
    \ENDFOR\vspace{2mm}
    \STATE \begin{footnotesize}\texttt{// Generate mask by the popup scores}\end{footnotesize}\vspace{2mm}
    \FOR{$l \in \theta_\text{layers}$}
        \STATE $m_\delta^l = s_l > \text{quantile}(s_l, \delta)$ \begin{footnotesize}\texttt{// Generated mask for layer $l$}\end{footnotesize}
    \ENDFOR\vspace{2mm}
    \STATE \begin{footnotesize}\texttt{// Unleashing Mask: fine-tuning}\end{footnotesize}\vspace{2mm}
    \FOR{$t \in (1, \dots, k)$}
        \STATE $\theta^{(t + 1)} = \theta^{(t)} - \eta \frac{\partial(|\mathcal{L}_{\text{CE}}(x, \theta^{(t)}) - \mathbb{E}_{x \sim \mathcal{D}^\text{s}_\text{in}}(\hat{\mathcal{L}}_{\text{CE}}(x, m_\delta \odot \theta^{(t)}))| + \mathbb{E}_{x \sim \mathcal{D}^\text{s}_\text{in}}(\hat{\mathcal{L}}_{\text{CE}}(x, m_\delta \odot \theta^{(t)})))}{\partial\theta^{(t)}}$
    \ENDFOR
\end{algorithmic}
\end{algorithm}

\begin{algorithm}[h!]
  \caption{Unleashing Mask Adopt Pruning (UMAP)}
  \label{alg:umap}
  {\bf Input:} well-trained model : $\theta$, Gaussian distribution: $N(\mu, \sigma^2)$, mask ratio: $\delta \in [0, 1]$, fine-tuning epochs of UM: $k$, training samples: $x \sim \mathcal{D}^\text{s}_\text{in}$, prune rate: $p$ , layer-iterated model : $\theta_{\text{layer}}$,compute the $\delta$-th quantile of the data $s$ : $\text{quantile}(s, \delta)$;\\
  {\bf Output:} learnt binary mask $\hat{m}_p$;
\begin{algorithmic}[1]
    \STATE \begin{footnotesize}\texttt{// Initialize a popup score for every weight}\end{footnotesize}\vspace{2mm}
    \FOR{$w \in \theta$}
        \STATE {$s_w \sim N(\mu, \sigma^2)$}
    \ENDFOR\vspace{2mm}
    \STATE \begin{footnotesize}\texttt{// Generate mask by the popup scores}\end{footnotesize}\vspace{2mm}
    \FOR{$l \in \theta_\text{layers}$}
        \STATE $m_\delta^l = s^l > \text{quantile}(s^l, \delta)$ \begin{footnotesize}\texttt{// Generated mask for layer $l$}\end{footnotesize}
    \ENDFOR\vspace{2mm}
    \STATE \begin{footnotesize}\texttt{// Initialize the ready-to-be-learned scores for UMAP}\end{footnotesize}\vspace{2mm}
    \FOR{$w \in \theta$}
        \STATE $\hat{s}^{(1)}_w \sim N(\mu, \sigma^2)$
    \ENDFOR\vspace{2mm}
    \STATE \begin{footnotesize}\texttt{// Unleashing Mask Adopt Pruning: Pruning}\end{footnotesize}\vspace{2mm}
    \FOR{$t \in (1, \dots, k)$}
        \STATE \begin{footnotesize}\texttt{// Generate the mask for UMAP according to learned scores $\hat{s}^{t}$}\end{footnotesize}
        \FOR{$l \in \theta_\text{layers}$}
            \STATE $\hat{m}_p^l = \hat{s}^{t}_{l} > \text{quantile}(\hat{s}^{t}_{l}, p)$ 
        \ENDFOR
        \STATE $\hat{s}^{(t + 1)} = \hat{s}^{(t)} - \eta \frac{\partial(|\mathcal{L}_{\text{CE}}(x, \hat{m}_p \odot\theta) - \mathbb{E}_{x \sim \mathcal{D}^\text{s}_\text{in}}(\hat{\mathcal{L}}_{\text{CE}}(x, m_\delta \odot \theta))| + \mathbb{E}_{x \sim \mathcal{D}^\text{s}_\text{in}}(\hat{\mathcal{L}}_{\text{CE}}(x, m_\delta \odot \theta)))}{\partial\hat{s}^{(t)}}$
    \ENDFOR
    \FOR{$l \in \theta_\text{layers}$}
        \STATE $\hat{m}^l_p = \hat{s}^{k}_{l} > \text{quantile}(\hat{s}^{k}_{l}, p)$ 
    \ENDFOR
\end{algorithmic}
\end{algorithm}


\begin{table}[h!]
    \caption{Completed Results ($\%$). Comparison with competitive OOD detection baselines. $\uparrow$ indicates higher values are better, and $\downarrow$ indicates lower values are better.}
    \vspace{2mm}
    \centering
    \footnotesize
    \resizebox{\textwidth}{!}{
    \begin{tabular}{c|l|ccccc}
        \toprule[1.5pt]
        $\mathcal{D}_\text{in}$ &  Method & AUROC$\uparrow$ & AUPR$\uparrow$ & FPR95$\downarrow$ & ID-ACC$\uparrow$ & w./w.o $\mathcal{D}_\text{aux}$ \\
        \midrule[0.6pt]
        \multirow{13}*{\textbf{CIFAR-10}}
         & MSP\citep{hendrycks17baseline} & $89.90\pm 0.30$ & $91.48\pm 0.43$ & $60.08\pm 0.76$ & $\textbf{94.01}\pm\textbf{0.08}$ & \\
         & ODIN\citep{LiangLS18} & $91.46\pm 0.56$ & $91.67\pm 0.58$ & $42.31\pm 1.38$ & $\textbf{94.01}\pm\textbf{0.08}$ & \\
         & Mahalanobis\citep{10.5555/3327757.3327819} & $75.10\pm 1.04$ & $72.32 \pm 1.92$ & $61.35\pm 1.25$ & $\textbf{94.01}\pm\textbf{0.08}$ & \\
         & Energy\citep{liu2020energy} & $92.07\pm 0.22$ & $92.72\pm 0.39$ & $42.69\pm 1.31$ & $\textbf{94.01}\pm\textbf{0.08}$ & \\
         & \textbf{Energy+UM} (ours) & $93.73\pm 0.36$ & $94.27\pm 0.60$ & $33.29\pm 1.70$ & $92.80\pm 0.47$ & \\
         & \textbf{Energy+UMAP} (ours) & $\textbf{93.97}\pm\textbf{0.11}$ & $\textbf{94.38}\pm\textbf{0.06}$ & $\textbf{30.71}\pm\textbf{1.94}$ & $\textbf{94.01}\pm\textbf{0.08}$ & \\
         \cmidrule{2-7}
         & OE\citep{hendrycks2018deep} & $97.07\pm 0.01$ & $97.31\pm 0.05$ & $13.80\pm 0.28$ & $92.59\pm 0.32$ & $\checkmark$\\
         & Energy (w. $\mathcal{D}_\text{aux}$)\citep{liu2020energy} & $94.58\pm 0.64$ & $94.69\pm 0.65$ & $18.79\pm 2.31$ & $80.91\pm 3.13$ & $\checkmark$\\
         & POEM\citep{ming2022poem}  & $94.37\pm 0.07$ & $94.51\pm 0.06$ & $18.50\pm 0.33$ & $77.24\pm 2.22$ & $\checkmark$\\
         & \textbf{OE+UM} (ours) & $\textbf{97.60}\pm \textbf{0.03}$ & $\textbf{97.87}\pm \textbf{0.02}$ & $\textbf{11.22}\pm \textbf{0.16}$ & $\textbf{93.66}\pm \textbf{0.12}$ & $\checkmark$\\
         & \textbf{Energy+UM} (ours) & $93.02\pm0.42$ & $92.36\pm0.38$ & $24.41\pm1.65$ & $71.97\pm0.92$ & $\checkmark$\\
         & \textbf{POEM+UM} (ours) & $93.04\pm0.02$ & $92.99\pm0.02$ & $23.52\pm0.16$ & $67.41\pm0.27$ & $\checkmark$\\
         & \textbf{OE+UMAP} (ours) & $97.48\pm0.01$ & $97.74\pm0.00$ & $12.21\pm0.09$ & $93.44\pm0.21$ & $\checkmark$\\
         & \textbf{Energy+UMAP} (ours) & $95.63\pm1.15$ & $95.92\pm1.17$ & $17.51\pm2.59$ & $88.12\pm4.22$ & $\checkmark$\\
         & \textbf{POEM+UMAP} (ours) & $94.18\pm2.98$ & $94.15\pm3.46$ & $20.55\pm8.70$ & $76.62\pm17.95$ & $\checkmark$\\
        \midrule[0.6pt]
        \multirow{13}*{\textbf{CIFAR-100}}
         & MSP\citep{hendrycks17baseline} & $74.06\pm 0.69$ & $75.37\pm 0.73$ & $83.14\pm 0.87$ & $\textbf{74.86}\pm\textbf{0.21}$ & \\
         & ODIN\citep{LiangLS18} & $76.18\pm 0.14$ & $76.49\pm 0.20$ & $78.93\pm 0.31$ & $\textbf{74.86}\pm\textbf{0.21}$ & \\
         & Mahalanobis\citep{10.5555/3327757.3327819} & $63.90\pm 1.91$ & $64.31\pm 0.91$ & $78.79\pm 0.50$ & $\textbf{74.86}\pm\textbf{0.21}$ & \\
         & Energy\citep{liu2020energy} & $\textbf{76.29}\pm \textbf{0.24}$ & $\textbf{77.06}\pm \textbf{0.55}$ & $78.46\pm 0.06$ & $\textbf{74.86}\pm\textbf{0.21}$ & \\
         & \textbf{Energy+UM} (ours) & $76.22\pm 0.42$ & $76.39\pm 1.03$ & $74.05\pm 0.55$ & $64.55\pm 0.24$ & \\
         & \textbf{Energy+UMAP} (ours) & $75.57\pm0.59$ & $75.66\pm0.07$ & $\textbf{72.21}\pm\textbf{1.46}$ & $\textbf{74.86}\pm\textbf{0.21}$ & \\
         \cmidrule{2-7}
         & OE\citep{hendrycks2018deep} & $90.55\pm 0.87$ & $90.34\pm 0.94$ & $34.73\pm 3.85$ & $73.59\pm 0.30$ & $\checkmark$\\
         & Energy (w. $\mathcal{D}_\text{aux}$)\citep{liu2020energy} & $88.92\pm 0.57$ & $89.13\pm 0.56$ & $37.90\pm 2.59$ & $57.85\pm 2.65$ & $\checkmark$\\
         & POEM\citep{ming2022poem}  & $88.95\pm 0.54$ & $88.94\pm 0.31$ & $38.10\pm 1.30$ & $56.18\pm 1.92$ & $\checkmark$\\
         & \textbf{OE+UM} (ours) & $91.04\pm0.11$& $91.13\pm0.24$ & $34.71\pm0.81$ & $\textbf{75.15}\pm\textbf{0.18}$ & $\checkmark$ \\
         & \textbf{Energy+UM} (ours) & $90.39\pm0.40$ & $90.14\pm0.45$ & $32.65\pm3.13$ & $71.95\pm0.23$ & $\checkmark$\\
         & \textbf{POEM+UM} (ours) & $\textbf{91.18}\pm\textbf{0.35}$ & $\textbf{91.45}\pm\textbf{0.27}$ & $\textbf{30.78}\pm\textbf{1.76}$ & $70.17\pm0.01$ & $\checkmark$\\
         & \textbf{OE+UMAP} (ours) & $91.10\pm0.16$ & $90.99\pm0.23$ & $33.62\pm0.26$ & $74.76\pm0.11$ & $\checkmark$\\
         & \textbf{Energy+UMAP} (ours) & $90.52\pm0.26$ & $90.46\pm0.50$ & $32.17\pm0.30$ & $72.76\pm0.18$ & $\checkmark$\\
         & \textbf{POEM+UMAP} (ours) & $91.10\pm0.29$ & $91.41\pm0.28$ & $31.02\pm1.70$ & $71.05\pm0.04$ & $\checkmark$\\
        \bottomrule[1.5pt]
    \end{tabular}}
    \label{tab:my_label_complete}
\end{table}

\begin{table*}[t!]
    \caption{OOD Detection Performance on ImageNet Dataset. $\uparrow$ indicates higher values are better, and $\downarrow$ indicates lower values are better. We experiment with the large-scale classification on a pretrained Resnet-50 (i.e., provided by PyTorch). To avoid potential semantic or covariate overlap between ID set (ImageNet) and OOD test sets \citep{sun2021react}, we choose iNaturalist, Textures, Places365, and SUN as OOD evaluation sets following previous literatures~\citep{liu2020energy, huang2021importance}. Here we provide results of MSP, ODIN, and Energy using FPR95 and AUROC.}
    \vspace{2mm}
    \centering
    \footnotesize
    \renewcommand\arraystretch{0.9}
    \resizebox{\textwidth}{!}{
    \begin{tabular}{c|l|cccccccc|cc}
        \toprule[1.5pt]
        \multirow{3}*{\textbf{ID dataset}} & \multirow{3}*{\textbf{Method}} & \multicolumn{8}{c|}{\textbf{OOD dataset}} \\
        ~ & ~ & \multicolumn{2}{c}{\textbf{iNaturalist}} & \multicolumn{2}{c}{\textbf{Textures}} & \multicolumn{2}{c}{\textbf{Places365}} &  \multicolumn{2}{c|}{\textbf{SUN}} & \multicolumn{2}{c}{\textbf{Average}} \\
        ~ & ~ & FPR95$\downarrow$ & AUROC$\uparrow$ & FPR95$\downarrow$ & AUROC$\uparrow$ & FPR95$\downarrow$ & AUROC$\uparrow$  & FPR95$\downarrow$ & AUROC$\uparrow$  & FPR95$\downarrow$ & AUROC$\uparrow$ \\
        \midrule[0.6pt]
        \multirow{10}*{\textbf{ImageNet}}
         & MSP & $47.83$ & $89.06$ & $49.57$ & $85.62$ & $61.76$ & $84.89$ & $61.36$ & $85.02$ & $55.13$ & $85.15$\\
         & ODIN & $41.39$ & $89.83$ & $44.15$ & $84.04$ & $60.12$ & $82.46$ & $58.52$ & $82.66$ & $51.29$ & $84.75$\\
         & Energy & $49.12$ & $87.69$ & $49.59$ & $81.90$ & $66.07$ & $80.56$ & $65.00$ & $80.96$ & $57.45$ & $82.78$\\
         \cmidrule{2-12}
         & \textbf{MSP+UM} (ours) & $37.85$ & $90.70$ & $46.22$ & $86.50$ & $57.62$ & $85.56$ & $57.36$ & $85.51$ & $49.76$ & $87.07$\\
         & \textbf{ODIN+UM} (ours) & $28.88$ & $91.97$ & $\textbf{39.91}$ & $84.89$ & $51.22$ & $84.12$ & $51.06$ & $84.01$ & $42.77$ & $86.25$\\
         & \textbf{Energy+UM} (ours) & $33.11$ & $90.91$ & $44.10$ & $82.78$ & $56.52$ & $82.53$ & $55.44$ & $82.64$ & $47.29$ & $84.72$\\
         \cmidrule{2-12}
         & \textbf{MSP+UMAP} (ours) & $36.90$ & $91.61$ & $51.61$ & $87.37$ & $61.94$ & $86.07$ & $61.56$ & $84.91$ & $53.00$ & $87.24$\\
         & \textbf{ODIN+UMAP} (ours) & $\textbf{21.97}$ & $\textbf{94.71}$ & $42.02$ & $\textbf{88.35}$ & $\textbf{50.06}$ & $\textbf{86.99}$ & $\textbf{49.69}$ & $\textbf{86.92}$ & $\textbf{40.94}$ & $\textbf{89.24}$\\
         & \textbf{Energy+UMAP} (ours) & $33.03$ & $92.41$ & $64.06$ & $82.22$ & $61.76$ & $83.17$ & $60.99$ & $83.26$ & $54.96$ & $85.26$\\
         
        \bottomrule[1.5pt]
    \end{tabular}}
    \label{tab:my_imagenet}
\end{table*}

\clearpage

\section{Additional Experiment Results}
\label{app:additional_exp_results}

In this section, we provide more experiment results from different perspectives to characterize our proposed algorithms. 


\subsection{Additional Setups}
\label{app:additional_exp_setup}

\paragraph{Training details.} We conduct all major experiments on DenseNet-101 \citep{huang2017densely} with training epochs fixed to 100. The models are trained using stochastic gradient descent \citep{1177729392} with Nesterov momentum \citep{JMLR:v12:duchi11a}. We adopt Cosine Annealing \citep{LoshchilovH17} to schedule the learning rate which begins at $0.1$. We set the momentum and weight decay to be $0.9$ and $10^{-4}$ respectively throughout all experiments. The size of the mini-batch is $256$ for both ID samples (during training and testing) and OOD samples (during testing). The choice of mask ratio for our UM and UMAP is detailed and further discussed in Appendix \ref{app:eff_um}.

\paragraph{Model architecture.} For DenseNet-101, we fix the growth rate and reduce the rate to 12 and 0.5 respectively with the bottleneck block included in the backbone~\citep{ming2022poem}. We also explore the proposed UM on WideResNet \citep{zagoruyko2016wide} with 40 depth and 4 widen factor, which is termed as WRN-40-4. The batch size for both ID and OOD testing samples is $256$, and the batch size of auxiliary samples is 2000. The $\lambda$ in Eq.~(\ref{eq:oe_app}) is 0.5 to keep the OE loss comparable to the CE loss. As for the outliers sampling, we randomly retrieve $50000$ samples from ImageNet-1k~\citep{deng2009imagenet} for OE and Energy (w. $\mathcal{D}_{\text{aux}}$) and $50000$ samples using Thompson sampling~\citep{thompson} for POEM~\citep{ming2022poem}.

\begin{wrapfigure}{r}{0.38\textwidth}
  \begin{center}
    \includegraphics[scale=0.14]{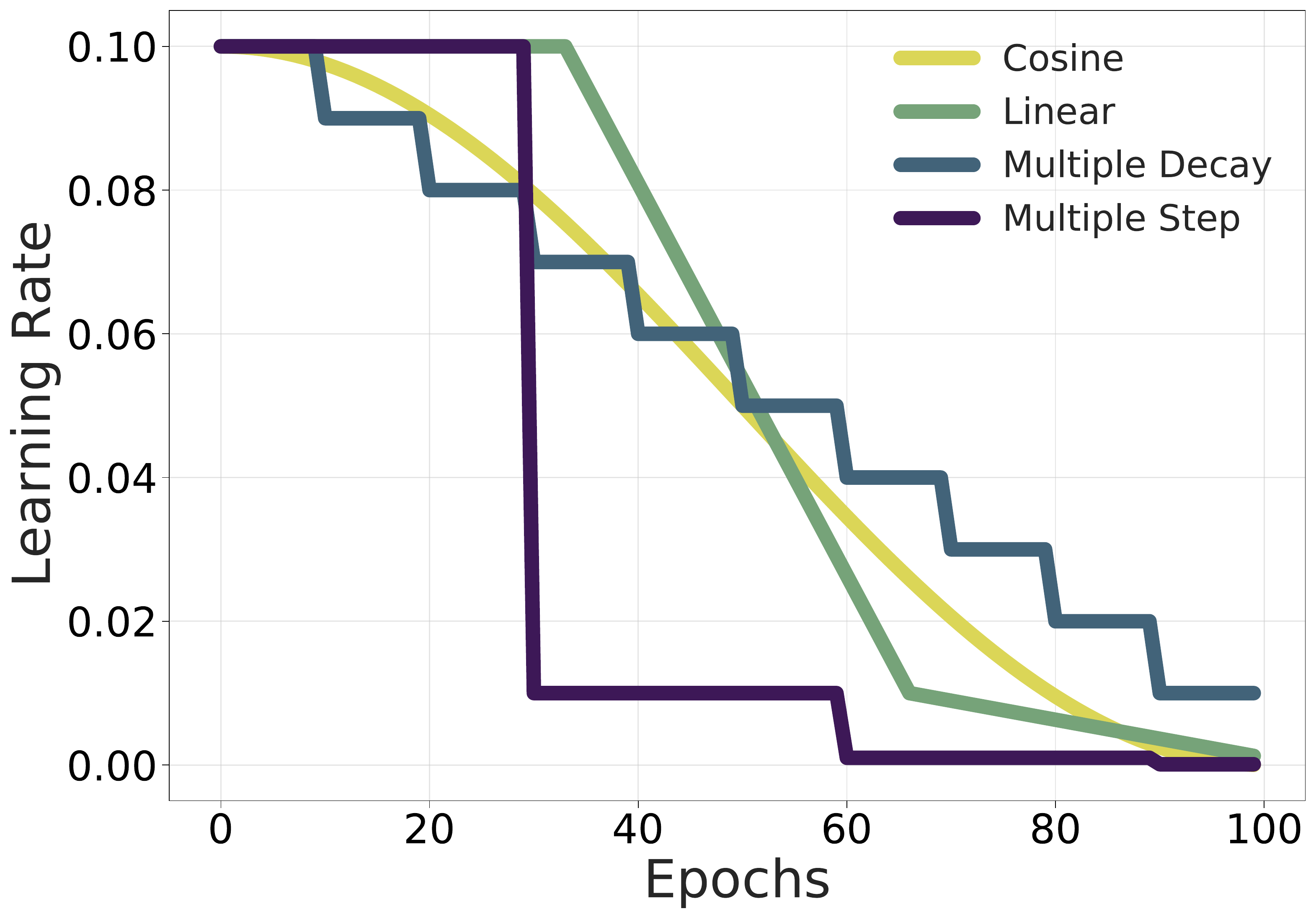}
    \end{center}
    \vspace{-2mm}
    \caption{Learning Rate Scheduler.}
    \label{fig7:lre_scheduler}
\end{wrapfigure}
\paragraph{Learning rate schedules.} We use 4 different learning rate schedules to demonstrate the existence of the overlaid OOD detection capability. For cosine annealing, we follow the common setups in \citet{LoshchilovH17}; for linear schedule, the learning rate remains the same in the first one-third epochs, decreases linearly to the tenth of the initial rate in the middle one-third epochs, and decrease linearly to $1\%$ of the initial rate in the last one-third epochs; for the multiple decay schedule, the learning rate decreases $10\%$ of the initial rate ($0.01$) every $10\%$ epochs ($10$ epochs); for the multiple step schedule, the learning rate decreases to $10\%$ of the current rate every $30$ epochs. All those learning rate schedules for our experiments are intuitively illustrated in Figure~\ref{fig7:lre_scheduler}.

\subsection{Empirical verification on typical/atypical data.}
\label{app:exp_typical_atypical}

In the following Tables \ref{tab:atypical_integrated}, \ref{tab:atypical_cifar10_densenet}, \ref{tab:atypical_cifar10_wrn}, \ref{tab:atypical_cifar100_densenet}, and \ref{tab:atypical_cifar100_wrn}, we further conduct the experiments to identify the negative effect of learning on those atypical samples by comparing with a counterpart that learning only with the typical samples. The results demonstrate that the degeneration in detection performance is more likely to come from learning atypical samples.

In Table \ref{tab:atypical_integrated}, we provide the main results for the verification using typical/atypical samples. Intuitively, we intend to separate the training dataset into a typical set and an atypical set, and train respectively on these two sets to see whether it is learning atypical samples that induce the degradation in OOD detection performance during the latter  training phase. Specifically, we input the training samples through the model (DenseNet-101) of the 60th epoch and get the CE loss for selection. We provide the ACC of the generated sets on the model of the 60th epoch (ACC in the tables). The extremely low ACCs of the atypical sets show that the model of the 60th epoch can hardly predict the right label, which meets our conceptual definition of atypical samples. We then finetune the model of the 60th epoch with the generated dataset and report the OOD performance. The results show learning from only those atypical data fails to gain better detection performance than its counterpart (i.e., learning from only those typical data), although it is beneficial to improve the performance of the original multi-class classification task. The experiments provide a conceptual verification of our conjecture which links our observation and the proposed method.

\begin{table}[h!]
    \caption{Fine-tuning on typical/atypical samples with different model structures ($\%$). $\uparrow$ indicates higher values are better, and $\downarrow$ indicates lower values are better.}
    \vspace{2mm}
    \centering
    \footnotesize
    \begin{tabular}{c|c|c|cccc}
        \toprule[1.5pt]
        $\mathcal{D}_\text{in}$ &  Dataset Size & Structure & Atypical/Typical & AUROC$\uparrow$ & AUPR$\uparrow$ & FPR95$\downarrow$ \\
        \midrule[0.6pt]
        \multirow{4}*{\textbf{CIFAR-10}}
         & \multirow{4}*{$200$}
         & \multirow{2}*{DenseNet-101}
         & Atypical & $81.45 $ & $82.40 $ & $62.10 $\\
         &  & & Typical & $\textbf{82.86} $ & $\textbf{84.38} $ & $\textbf{60.01} $\\
         \cmidrule{3-7}
         & & \multirow{2}*{WRN-40-4}
         & Atypical & $85.13 $ & $86.57 $ & $66.41 $\\
         & & & Typical & $\textbf{86.26} $ & $\textbf{86.89} $ & $\textbf{59.93} $\\
         \cmidrule{1-7}
         \multirow{4}*{\textbf{CIFAR-100}}
         & \multirow{4}*{$1000$}
         & \multirow{2}*{DenseNet-101}
         & Atypical & $71.96 $ & $73.16 $ & $85.57 $\\
         & & & Typical & $\textbf{74.79} $ & $\textbf{75.83} $ & $\textbf{80.97} $\\
         \cmidrule{3-7}
         & & \multirow{2}*{WRN-40-4}
         & Atypical & $66.64 $ & $67.41 $ & $86.92 $\\
         & & & Typical & $\textbf{71.95} $ & $\textbf{72.02} $ & $\textbf{80.00} $\\
        \bottomrule[1.5pt]
    \end{tabular}
    \label{tab:atypical_integrated}
\end{table}

\begin{table}[h!]
    \caption{Fine-tuning on typical/atypical CIFAR-10 samples with DenseNet-101 ($\%$). $\uparrow$ indicates higher values are better, and $\downarrow$ indicates lower values are better.}
    \vspace{2mm}
    \centering
    \footnotesize
    \begin{tabular}{c|c|ccccc}
        \toprule[1.5pt]
        $\mathcal{D}_\text{in}$ &  Dataset Size & Atypical/Typical & ACC & AUROC$\uparrow$ & AUPR$\uparrow$ & FPR95$\downarrow$ \\
        \midrule[0.6pt]
        \multirow{7}*{\textbf{CIFAR-10}}
         & \multirow{2}*{$200$}
         & Atypical & $3.50$ & $81.45 $ & $82.40 $ & $62.10 $\\
         & & Typical & $100.00$ & $\textbf{82.86} $ & $\textbf{83.48} $ & $\textbf{60.01} $\\
         \cmidrule{2-7}
         & \multirow{2}*{$350$}
         & Atypical & $11.14$ & $85.90 $ & $86.01 $ & $55.10 $\\
         & & Typical & $100.00$ &  $\textbf{85.90} $ & $\textbf{86.16} $ & $\textbf{52.81} $\\
         \cmidrule{2-7}
         & \multirow{2}*{$500$}
         & Atypical & $16.80$ & $84.94 $ & $85.33 $ & $59.27 $\\
         & & Typical & $100.00$ & $\textbf{85.53} $ & $\textbf{86.10} $ & $\textbf{58.74} $\\
        \bottomrule[1.5pt]
    \end{tabular}
    \label{tab:atypical_cifar10_densenet}
\end{table}

\begin{table}[h!]
    \caption{Fine-tuning on typical/atypical CIFAR-10 samples with WRN-40-4 ($\%$). $\uparrow$ indicates higher values are better, and $\downarrow$ indicates lower values are better.}
    \vspace{2mm}
    \centering
    \footnotesize
    \begin{tabular}{c|c|ccccc}
        \toprule[1.5pt]
        $\mathcal{D}_\text{in}$ &  Dataset Size & Atypical/Typical & ACC & AUROC$\uparrow$ & AUPR$\uparrow$ & FPR95$\downarrow$ \\
        \midrule[0.6pt]
        \multirow{7}*{\textbf{CIFAR-10}}
         & \multirow{2}*{$200$}
         & Atypical & $3.50$ & $85.13 $ & $86.57 $ & $66.41 $\\
         & & Typical & $100.00$ & $\textbf{86.26} $ & $\textbf{86.89} $ & $\textbf{59.93} $\\
         \cmidrule{2-7}
         & \multirow{2}*{$350$}
         & Atypical & $11.14$ & $82.92 $ & $84.24 $ & $68.57 $\\
         & & Typical & $100.00$ & $\textbf{85.82} $ & $\textbf{87.84} $ & $\textbf{65.54} $\\
         \cmidrule{2-7}
         & \multirow{2}*{$500$}
         & Atypical & $16.80$ & $82.88 $ & $83.22 $ & $66.75 $\\
         & & Typical & $100.00$ & $\textbf{87.38} $ & $\textbf{87.93} $ & $\textbf{52.27} $\\
        \bottomrule[1.5pt]
    \end{tabular}
    \label{tab:atypical_cifar10_wrn}
\end{table}

\begin{table}[h!]
    \caption{Fine-tuning on typical/atypical CIFAR-100 samples with DenseNet-101 ($\%$). $\uparrow$ indicates higher values are better, and $\downarrow$ indicates lower values are better.}
    \vspace{2mm}
    \centering
    \footnotesize
    \begin{tabular}{c|c|ccccc}
        \toprule[1.5pt]
        $\mathcal{D}_\text{in}$ &  Dataset Size & Atypical/Typical & ACC & AUROC$\uparrow$ & AUPR$\uparrow$ & FPR95$\downarrow$ \\
        \midrule[0.6pt]
        \multirow{7}*{\textbf{CIFAR-100}}
         & \multirow{2}*{$500$}
         & Atypical & $1.00$ & $72.69 $ & $73.28 $ & $80.71 $\\
         & & Typical & $100.00$ & $\textbf{74.07} $ & $\textbf{75.20} $ & $\textbf{80.19} $\\
         \cmidrule{2-7}
         & \multirow{2}*{$800$}
         & Atypical & $2.88$ & $69.74 $ & $71.15 $ & $85.46 $\\
         & & Typical & $100.00$ & $\textbf{72.49} $ & $\textbf{73.17} $ & $\textbf{81.97} $\\
         \cmidrule{2-7}
         & \multirow{2}*{$1000$}
         & Atypical & $3.50$ & $71.96 $ & $73.16 $ & $85.57 $\\
         & & Typical & $100.00$ & $\textbf{74.79} $ & $\textbf{75.83} $ & $\textbf{80.97} $\\
        \bottomrule[1.5pt]
    \end{tabular}
    \label{tab:atypical_cifar100_densenet}
\end{table}

\begin{table}[h!]
    \caption{Fine-tuning on typical/atypical CIFAR-100 samples with WRN-40-4 ($\%$). $\uparrow$ indicates higher values are better, and $\downarrow$ indicates lower values are better.}
    \vspace{2mm}
    \centering
    \footnotesize
    \begin{tabular}{c|c|ccccc}
        \toprule[1.5pt]
        $\mathcal{D}_\text{in}$ &  Dataset Size & Atypical/Typical & ACC & AUROC$\uparrow$ & AUPR$\uparrow$ & FPR95$\downarrow$ \\
        \midrule[0.6pt]
        \multirow{7}*{\textbf{CIFAR-100}}
         & \multirow{2}*{$500$}
         & Atypical & $1.00$ & $66.03 $ & $66.17 $ & $89.56 $\\
         & & Typical & $100.00$ & $\textbf{68.60} $ & $\textbf{69.93} $ & $\textbf{86.53} $\\
         \cmidrule{2-7}
         & \multirow{2}*{$800$}
         & Atypical & $2.88$ & $67.59 $ & $68.66 $ & $85.61 $\\
         & & Typical & $100.00$ & $\textbf{70.25} $ & $\textbf{68.95} $ & $\textbf{79.66} $\\
         \cmidrule{2-7}
         & \multirow{2}*{$1000$}
         & Atypical & $3.50$ & $66.64 $ & $67.41 $ & $86.92 $\\
         & & Typical & $100.00$ & $\textbf{71.95} $ & $\textbf{72.02} $ & $\textbf{80.00} $\\
        \bottomrule[1.5pt]
    \end{tabular}
    \label{tab:atypical_cifar100_wrn}
\end{table}

\subsection{Empirical Efficiency of UM and UMAP}
\label{app:exp_less_epochs}


As mentioned before, UM adopts finetuning on the proposed objective for forgetting has shown the advantages of being cost-effective compared with train-from-scratch. For the tuning epochs, we show in Figures \ref{fig12:abla_app_6} and \ref{fig13:abla_app_7} that fine-tuning using UM can converge within about 20 epochs, indicating that we can apply our UM/UMAP for far less than 100 epochs (compared with train-from-scratch) to restore the better detection performance of the original well-trained model. It is intuitively reasonable that finetuning with the newly designed objective would benefit from the well-trained model, allowing a faster convergence since the two phases consider the same task with the same training data. As for the major experiments conducted in our work, finetuning adopts 100 epochs for better exploring and presenting its learning dynamics for research purposes, and this configuration is indicated in the training details of Section \ref{sec:exp_part1}.

Here, we also provide an extra comparison to directly show the relative efficiency of our proposed UM/UMAP in the following Table \ref{tab:20epoch_densenet} and Table \ref{tab:20epoch_wrn}. The results demonstrate that UM and UMAP can efficiently restore detection performance compared with the baseline. Considering the significance of the OOD awareness for those safety-critical areas, it is worthwhile to further excavate the OOD detection capability of the deployed well-trained model using our UM and UMAP. 

However, there may be a concern that while both UM/UMAP and OE-based methods need extra fine-tuning processes, why should we choose UM/UMAP instead of OE-based methods, given that OE-based methods can also achieve good performance on OOD detection task. The intuition of UM/UMAP is to unleash the OOD detection capability of a pre-trained model with ID data, which is orthogonal to those OE-based methods (e.g. DOE~\citep{wang2023outofdistribution}), improving the OOD capability of a pre-train model with both ID data and auxiliary data. On the one hand, OE-based methods need sampling/synthesizing large auxiliary OOD datasets, while UM/UMAP only needs the ID data. On the other hand, although both require additional costs to fine-tune the model, they are orthogonal and can be coupled (as discussed in Section~\ref{sec:method_part3}). To further address the concern, we conduct additional experiments(i.e., OE-based results in Table~\ref{tab:additional_sota}) to validate their mutual benefit in the combination. According to the results, we can find that UM/UMAP with DOE achieves better performance. This is because while OE-based methods can improve the performance of OOD detection by fine-tuning with both ID and auxiliary outliers, UM/UMAP can serve as a method (only using ID data) to encourage optimization to learn a more appropriate model for OOD detection.

\begin{table}[h!]
    \caption{Fine-tuning for 20 epochs with DenseNet-101 ($\%$). $\uparrow$ indicates higher values are better, and $\downarrow$ indicates lower values are better.}
    \vspace{2mm}
    \centering
    \footnotesize
    \begin{tabular}{c|c|lcccc}
        \toprule[1.5pt]
        $\mathcal{D}_\text{in}$ &  Epoch & Method &  AUROC$\uparrow$ & AUPR$\uparrow$ & FPR95$\downarrow$ & ID-ACC$\uparrow$\\
        \midrule[0.6pt]
        \multirow{11}*{\textbf{CIFAR-10}}
         & \multirow{5}*{$100$}
         & MSP & $89.90 $ & $91.48 $ & $60.08 $ & $94.01 $\\
         & & ODIN & $91.46 $ & $91.67 $ & $42.31 $ & $94.01 $\\
         & & Energy & $92.07 $ & $92.72 $ & $42.69 $ & $94.01 $\\
         & & Energy + UM & $93.73 $ & $94.27 $ & $33.29 $ & $92.80 $\\
         & & Energy + UMAP & $93.97 $ & $94.38 $ & $30.71 $ & $94.01 $\\
        \cmidrule{2-7}
         & \multirow{6}*{$\textbf{20}$}
         & MSP + UM & $90.31 $ & $91.99 $ & $53.61 $ & $91.70 $\\
         & & ODIN +UM & $94.08 $ & $94.67 $ & $31.01 $ & $91.70 $\\
         & & Energy + UM & $93.60 $ & $94.32 $ & $33.03 $ & $91.70 $\\
         & & MSP + UMAP & $88.70 $ & $90.39 $ & $57.69 $ & $94.01 $\\
         & & ODIN + UMAP & $92.88 $ & $93.33 $ & $35.19 $ & $94.01 $\\
         & & Energy + UMAP & $92.88 $ & $93.39 $ & $35.60 $ & $94.01 $\\
        \bottomrule[1.5pt]
    \end{tabular}
    \label{tab:20epoch_densenet}
\end{table}

\begin{table}[h!]
    \caption{Fine-tuning for 20 epochs with WRN-40-4 ($\%$). $\uparrow$ indicates higher values are better, and $\downarrow$ indicates lower values are better.}
    \vspace{2mm}
    \centering
    \footnotesize
    \begin{tabular}{c|c|lcccc}
        \toprule[1.5pt]
        $\mathcal{D}_\text{in}$ &  Epoch & Method &  AUROC$\uparrow$ & AUPR$\uparrow$ & FPR95$\downarrow$ & ID-ACC$\uparrow$\\
        \midrule[0.6pt]
        \multirow{11}*{\textbf{CIFAR-10}}
         & \multirow{5}*{$100$}
         & MSP & $87.12 $ & $87.84 $ & $68.29 $ & $93.86 $\\
         & & ODIN & $83.29 $ & $82.74 $ & $65.68 $ & $93.86 $\\
         & & Energy & $87.69 $ & $88.16 $ & $58.47 $ & $93.86 $\\
         & & Energy + UM & $91.74 $ & $92.67 $ & $40.40 $ & $92.68 $\\
         & & Energy + UMAP & $88.84 $ & $89.31 $ & $50.23 $ & $93.86 $\\
        \cmidrule{2-7}
         & \multirow{6}*{$\textbf{20}$}
         & MSP + UM & $89.86 $ & $91.32 $ & $51.62 $ & $91.96 $\\
         & & ODIN +UM & $91.97 $ & $92.58 $ & $41.78 $ & $91.96 $\\
         & & Energy + UM & $92.95 $ & $93.64 $ & $36.21 $ & $91.96 $\\
         & & MSP + UMAP & $88.77 $ & $90.61 $ & $61.60 $ & $93.86 $\\
         & & ODIN + UMAP & $90.85 $ & $91.89 $ & $45.70 $ & $93.86 $\\
         & & Energy + UMAP & $91.66 $ & $92.49 $ & $42.94 $ & $93.86 $\\
        \bottomrule[1.5pt]
    \end{tabular}
    \label{tab:20epoch_wrn}
\end{table}

\subsection{Fine-grained Results on OOD Data}
\label{app:exp_finegrained}

In order to further understand the effectiveness of the proposed UM and UMAP on different OOD datasets, we report the fine-grained results of our experiments on CIFAR-10 and CIFAR-100 with $6$ OOD datasets (CIFAR-10/CIFAR-100, textures, Places365, SUN, LSUN, iNaturalist). 
The results on the $6$ OOD datasets show the general effectiveness of the proposed UM as well as UMAP. 
In Table~\ref{tab:my_label3}, \textbf{OE + UM} can outperform all the OOD baselines, and further improve the OOD performance even though the original detection performance is already well. By equipping with our proposed UM and UMAP, the baselines can outperform their counterparts on most of the OOD datasets. For instance, the FPR95 can decrease from $1.91$ to $1.42$. In Table~\ref{tab:my_label4}, we also take a closer check about results on CIFAR-100 with $6$ OOD datasets. Our proposed method can almost improve all competitive baselines (either the scoring functions or the finetuning with auxiliary outliers) on the $6$ OOD datasets. In both w. $\mathcal{D}_\text{aux}$ and w.o. $\mathcal{D}_\text{aux}$ scenarios, Unleashing Mask can significantly excavate the intrinsic OOD detection capability of the model. In addition to unleashing the excellent OOD performance, UMAP can also maintain the high ID-ACC by learning a binary mask instead of tuning the well-trained original parameters directly. Due to the space limit, we separate the results of SVHN dataset in Tables~\ref{tab:label_zoom_in_SVHN_densenet} and~\ref{tab:label_zoom_in_SVHN_wrn} to show the relative comparison of our UM and UMAP. The results demonstrate the general effectiveness of UM/UMAP compared with the original Energy score. Besides, we find Mahalanobis performs dramatically well which is an outlier method against other post-hoc baselines when SVHN as OOD set in our experiments. Our conjecture about this phenomenon is that the Mahalanobis score can perform better on those specific OOD data by inspecting the class conditional Gaussian distributions~\citep{LeeLLS18}. Nonetheless, the proposed UM/UMAP can still outstrip all the baselines on most OOD datasets under various settings at the perspective of average, showing their distinguishing effectiveness and practicability.

\begin{table}[t!]
    \caption{Fine-grained Results of DenseNet-101 on CIFAR-10 ($\%$). Comparison on different OOD benchmark datasets respectively. $\uparrow$ indicates higher values are better, and $\downarrow$ indicates lower values are better.}
    \vspace{2mm}
    \centering
    \footnotesize
    \renewcommand\arraystretch{0.9}
    \resizebox{\textwidth}{!}{
    \begin{tabular}{c|l|cccccc}
        \toprule[1.5pt]
        \multirow{3}*{\textbf{ID dataset}} & \multirow{3}*{\textbf{Method}} & \multicolumn{6}{c}{\textbf{OOD dataset}} \\
        ~ & ~ & \multicolumn{2}{c}{\textbf{CIFAR-100}} & \multicolumn{2}{c}{\textbf{Textures}} & \multicolumn{2}{c}{\textbf{Places365}}  \\
        ~ & ~ & FPR95$\downarrow$ & AUROC$\uparrow$ & FPR95$\downarrow$ & AUROC$\uparrow$ & FPR95$\downarrow$ & AUROC$\uparrow$ \\
        \midrule[0.6pt]
        \multirow{22}*{\textbf{CIFAR-10}}
         & \textbf{Energy + UMAP} & $54.95\pm2.61$ & $87.72\pm1.05$ & $33.59\pm1.32$ & $92.67\pm0.23$ & $32.80\pm4.14$ & $93.57\pm1.12$ \\
         \cmidrule{2-8}
         & OE & $59.29\pm1.30$ & $88.51\pm0.22$ & $2.89\pm0.30$ & $99.16\pm0.05$ & $11.14\pm1.11$ & $97.50\pm0.19$\\
         & Energy (w. $\mathcal{D}_\text{aux}$) & $79.88\pm2.47$ & $74.99\pm2.40$ & $4.27\pm0.57$ & $98.80\pm0.19$ & $14.22\pm3.99$ & $97.07\pm0.72$\\
         & POEM & $82.30\pm1.57$ & $72.74\pm1.42$ & $1.91\pm0.41$ & $99.40\pm0.10$ & $11.24\pm2.70$ & $96.67\pm0.48$\\
         & \textbf{OE + UM} (ours) & $\textbf{55.74}\pm\textbf{1.47}$ & $\textbf{89.53}\pm\textbf{0.18}$ & $\textbf{1.42}\pm\textbf{0.15}$ & $\textbf{99.49}\pm\textbf{0.04}$ & $\textbf{7.77}\pm\textbf{0.69}$ & $\textbf{98.15}\pm\textbf{0.08}$\\
         & \textbf{Energy (w. $\mathcal{D}_\text{aux}$)+ UM} (ours) & $84.52\pm0.01$ & $70.09\pm0.47$ & $8.30\pm0.88$ & $97.76\pm0.05$ & $20.27\pm1.30$ & $96.06\pm0.31$\\
         & \textbf{POEM + UM} (ours) & $84.87\pm1.56$ & $68.97\pm0.39$ & $4.73\pm0.52$ & $98.88\pm0.13$ & $19.83\pm0.34$ & $96.35\pm0.09$\\
         & \textbf{OE + UMAP} (ours) & $59.05\pm1.41$ & $89.14\pm0.14$ & $1.86\pm0.07$ & $99.35\pm0.00$ & $8.21\pm0.12$ & $98.07\pm0.03$\\
         & \textbf{Energy (w. $\mathcal{D}_\text{aux}$) + UMAP} (ours) & $75.18\pm4.96$ & $80.93\pm4.49$ & $2.24\pm1.34$ & $99.25\pm0.29$ & $9.30\pm2.12$ & $97.90\pm0.40$\\
         & \textbf{POEM + UMAP} (ours) & $79.33\pm4.14$ & $76.89\pm4.86$ & $2.10\pm1.37$ & $99.34\pm0.30$ & $9.94\pm6.92$ & $98.01\pm1.11$\\
         \cmidrule{2-8}
         ~ & \multirow{2}*{\textbf{Method}} &\multicolumn{2}{c}{\textbf{SUN}} &
        \multicolumn{2}{c}{\textbf{LSUN}} & \multicolumn{2}{c}{\textbf{iNaturalist}}\\
        ~ & ~ & FPR95$\downarrow$ & AUROC$\uparrow$ & FPR95$\downarrow$ & AUROC$\uparrow$ & FPR95$\downarrow$ & AUROC$\uparrow$ \\
        \cmidrule{2-8}
         & \textbf{Energy + UMAP} & $29.05\pm2.78$ & $94.41\pm0.73$ & $2.31\pm0.88$ & $99.42\pm0.04$ & $47.22\pm14.03$ & $92.63\pm2.14$ \\
         \cmidrule{2-8}
         & OE & $8.38\pm0.71$ & $98.00\pm0.14$ & $5.90\pm1.43$ & $98.60\pm0.21$ & $5.09\pm0.64$ & $98.76\pm0.11$\\
         & Energy (w. $\mathcal{D}_\text{aux}$) & $10.30\pm3.82$ & $97.77\pm0.64$ & $12.80\pm4.67$ & $96.08\pm1.38$ & $6.93\pm1.86$ & $98.40\pm0.32$\\
         & POEM & $8.39\pm2.42$ & $98.16\pm0.43$ & $9.69\pm1.89$ & $97.25\pm0.58$ & $3.78\pm0.90$ & $98.99\pm0.17$\\
         & \textbf{OE + UM} (ours) & $\textbf{5.51}\pm\textbf{0.44}$ & $\textbf{98.55}\pm\textbf{0.07}$ & $\textbf{3.51}\pm\textbf{0.43}$ & $\textbf{98.93}\pm\textbf{0.09}$ & $\textbf{2.87}\pm\textbf{0.49}$ & $\textbf{99.14}\pm\textbf{0.09}$\\
         & \textbf{Energy (w. $\mathcal{D}_\text{aux}$)+ UM} (ours) & $16.13\pm1.86$ & $96.84\pm0.30$ & $23.27\pm2.40$ & $92.11\pm0.94$ & $11.20\pm2.35$ & $97.54\pm0.42$\\
         & \textbf{POEM + UM} (ours) & $16.16\pm0.57$ & $97.01\pm0.08$ & $25.69\pm0.15$ & $93.38\pm0.27$ & $9.30\pm1.60$ & $98.05\pm0.24$\\
         & \textbf{OE + UMAP} (ours) & $6.16\pm0.02$ & $98.49\pm0.01$ & $4.53\pm0.16$ & $98.86\pm0.06$ & $3.40\pm0.74$ & $98.96\pm0.09$\\
         & \textbf{Energy (w. $\mathcal{D}_\text{aux}$) + UMAP} (ours) & $6.67\pm1.50$ & $98.40\pm0.32$ & $23.50\pm5.61$ & $94.78\pm2.04$ & $3.77\pm2.14$ & $98.93\pm0.46$\\
         & \textbf{POEM + UMAP} (ours) & $7.00\pm5.80$ & $98.46\pm0.96$ & $21.17\pm12.84$ & $94.74\pm3.67$ & $3.63\pm2.78$ & $98.99\pm0.54$\\
        \bottomrule[1.5pt]
    \end{tabular}}
    \label{tab:my_label3}
\end{table}

\begin{table}[t!]
    \caption{Fine-grained Results of DenseNet-101 on CIFAR-100 ($\%$). Comparison on different OOD benchmark datasets respectively. $\uparrow$ indicates higher values are better, and $\downarrow$ indicates lower values are better.}
    \vspace{2mm}
    \centering
    \footnotesize
    \renewcommand\arraystretch{0.9}
    \resizebox{\textwidth}{!}{
    \begin{tabular}{c|l|cccccc}
        \toprule[1.5pt]
        \multirow{3}*{\textbf{ID dataset}} & \multirow{3}*{\textbf{Method}} & \multicolumn{6}{c}{\textbf{OOD dataset}} \\
        ~ & ~ & \multicolumn{2}{c}{\textbf{CIFAR-10}} & \multicolumn{2}{c}{\textbf{Textures}} & \multicolumn{2}{c}{\textbf{Places365}}  \\
        ~ & ~ & FPR95$\downarrow$ & AUROC$\uparrow$ & FPR95$\downarrow$ & AUROC$\uparrow$ & FPR95$\downarrow$ & AUROC$\uparrow$ \\
        \midrule[0.6pt]
        \multirow{30}*{\textbf{CIFAR-100}}
         & MSP & $83.53\pm0.33$ & $75.11\pm0.27$ & $86.90\pm0.18$ & $71.45\pm0.40$ & $85.83\pm0.48$ & $70.54\pm0.42$\\
         & ODIN & $85.29\pm0.17$ & $73.31\pm0.24$ & $86.45\pm1.27$ & $71.91\pm0.27$ & $84.35\pm0.64$ & $73.58\pm0.51$\\
         & Mahalanobis & $98.25\pm0.05$ & $49.60\pm1.51$ & $\textbf{33.06}\pm\textbf{3.76}$ & $\textbf{90.19}\pm\textbf{1.21}$ & $95.20\pm0.49$ & $53.69\pm1.55$\\
         & Energy & $\textbf{82.16}\pm\textbf{0.59}$ & $\textbf{75.31}\pm\textbf{0.21}$ & $90.20\pm0.30$ & $68.98\pm0.34$ & $82.39\pm0.97$ & $73.78\pm0.66$\\
         & \textbf{Energy+UM} (ours) & $89.62\pm0.07$ & $66.12\pm0.93$ & $86.99\pm1.22$ & $65.39\pm1.44$ & $\textbf{77.30}\pm\textbf{2.08}$ & $\textbf{76.06}\pm\textbf{1.05}$\\
         \cmidrule{2-8}
         & OE & $90.97\pm0.46$ & $69.02\pm0.39$ & $14.36\pm0.25$ & $95.92\pm0.12$ & $40.19\pm6.97$ & $90.70\pm2.01$\\
         & Energy (w. $\mathcal{D}_\text{aux}$) & $96.14\pm0.06$ & $57.52\pm0.88$ & $9.02\pm0.06$ & $97.16\pm0.26$ & $35.18\pm4.73$ & $93.29\pm1.14$\\
         & POEM & $96.19\pm0.16$ & $55.82\pm1.05$ & $7.63\pm1.40$ & $97.69\pm0.06$ & $32.67\pm3.73$ & $93.94\pm0.68$\\
         & \textbf{OE + UM} (ours) & $89.61\pm0.08$ & $\textbf{71.24}\pm\textbf{0.08}$ & $16.78\pm0.25$ & $95.60\pm0.08$ & $39.77\pm0.34$ & $91.07\pm0.13$\\
         & \textbf{Energy (w. $\mathcal{D}_\text{aux}$)+ UM} (ours) & $95.38\pm0.45$ & $63.41\pm0.14$ & $6.41\pm0.83$ & $97.77\pm0.33$ & $30.96\pm3.61$ & $92.85\pm0.69$\\
         & \textbf{POEM + UM} (ours) & $95.78\pm0.14$ & $60.23\pm0.70$ & $\textbf{5.17}\pm\textbf{0.18}$ & $\textbf{98.53}\pm\textbf{0.03}$ & $\textbf{23.90}\pm\textbf{0.84}$ & $\textbf{95.45}\pm\textbf{0.11}$\\
         & \textbf{OE + UMAP} (ours) & $\textbf{90.72}\pm\textbf{0.35}$ & $69.76\pm0.25$ & $15.32\pm0.23$ & $95.72\pm0.01$ & $36.42\pm1.91$ & $92.08\pm0.49$\\
         & \textbf{Energy (w. $\mathcal{D}_\text{aux}$) + UMAP} (ours) & $95.39\pm0.10$ & $63.26\pm0.18$ & $6.52\pm0.44$ & $97.83\pm0.18$ & $31.18\pm0.43$ & $93.13\pm0.41$\\
         & \textbf{POEM + UMAP} (ours) & $95.69\pm0.17$ & $61.62\pm0.24$ & $5.23\pm0.58$ & $98.52\pm0.01$ & $26.06\pm1.16$ & $94.91\pm0.25$\\
         \cmidrule{2-8}
         ~ & \multirow{2}*{\textbf{Method}} &\multicolumn{2}{c}{\textbf{SUN}} &
        \multicolumn{2}{c}{\textbf{LSUN}} & \multicolumn{2}{c}{\textbf{iNaturalist}}\\
        ~ & ~ & FPR95$\downarrow$ & AUROC$\uparrow$ & FPR95$\downarrow$ & AUROC$\uparrow$ & FPR95$\downarrow$ & AUROC$\uparrow$ \\
        \cmidrule{2-8}
         & MSP & $88.75\pm0.23$ & $66.75\pm0.25$ & $67.83\pm1.37$ & $82.94\pm0.32$ & $85.00\pm0.73$ & $76.62\pm0.25$\\
         & ODIN & $88.49\pm0.99$ & $69.64\pm0.61$ & $34.80\pm2.55$ & $93.92\pm0.75$ & $81.67\pm2.77$ & $78.36\pm1.57$\\
         & Mahalanobis & $95.53\pm0.37$ & $54.37\pm1.35$ & $89.31\pm4.83$ & $43.19\pm16.36$ & $93.63\pm1.19$ & $49.60\pm1.51$\\
         & Energy & $97.17\pm0.92$ & $69.04\pm0.83$ & $35.09\pm3.17$ & $93.49\pm0.87$ & $85.70\pm2.14$ & $75.82\pm1.72$\\
         & \textbf{Energy+UM} (ours) & $\textbf{81.96}\pm\textbf{2.26}$ & $\textbf{71.47}\pm\textbf{1.88}$ & $\textbf{22.54}\pm\textbf{5.93}$ & $\textbf{94.98}\pm\textbf{1.75}$ & $\textbf{74.28}\pm\textbf{3.72}$ & $\textbf{80.72}\pm\textbf{3.75}$\\
         \cmidrule{2-8}
         & OE & $44.47\pm9.10$ & $90.70\pm2.01$ & $\textbf{5.75}\pm\textbf{1.18}$ & $98.57\pm0.13$ & $25.51\pm4.12$ & $94.46\pm0.88$\\
         & Energy (w. $\mathcal{D}_\text{aux}$) & $32.69\pm5.69$ & $93.63\pm1.48$ & $55.75\pm4.31$ & $87.96\pm1.03$ & $17.34\pm4.54$ & $96.50\pm0.81$\\
         & POEM & $30.45\pm5.11$ & $94.26\pm0.90$ & $46.68\pm3.59$ & $90.30\pm2.17$ & $16.50\pm2.09$ & $96.63\pm0.23$\\
         & \textbf{OE + UM} (ours) & $44.23\pm0.20$ & $90.28\pm0.03$ & $5.80\pm0.33$ & $98.63\pm0.03$ & $26.72\pm1.95$ & $94.51\pm0.43$\\
         & \textbf{Energy (w. $\mathcal{D}_\text{aux}$)+ UM} (ours) & $28.98\pm3.15$ & $93.18\pm0.69$ & $37.56\pm4.81$ & $91.98\pm0.83$ & $10.83\pm2.06$ & $97.09\pm0.69$\\
         & \textbf{POEM + UM} (ours) & $\textbf{21.34}\pm\textbf{1.07}$ & $\textbf{95.76}\pm\textbf{0.22}$ & $33.74\pm6.22$ & $94.43\pm1.18$ & $\textbf{8.85}\pm\textbf{0.23}$ & $\textbf{97.93}\pm\textbf{0.03}$\\
         & \textbf{OE + UMAP} (ours) & $39.58\pm2.02$ & $91.37\pm0.52$ & $5.77\pm0.71$ & $\textbf{98.64}\pm\textbf{0.12}$ & $23.33\pm1.24$ & $95.08\pm0.22$\\
         & \textbf{Energy (w. $\mathcal{D}_\text{aux}$) + UMAP} (ours) & $29.65\pm1.06$ & $93.38\pm0.19$ & $35.94\pm0.75$ & $92.08\pm0.39$ & $13.96\pm2.48$ & $96.87\pm0.15$\\
         & \textbf{POEM + UMAP} (ours) & $23.73\pm0.71$ & $95.25\pm0.06$ & $33.09\pm5.94$ & $93.57\pm0.88$ & $9.76\pm1.09$ & $97.77\pm0.24$\\
        \bottomrule[1.5pt]
    \end{tabular}}
    \label{tab:my_label4}
\end{table}

\subsection{Experiments on Different Model Structure}
\label{app:exp_diff_model}

Following \ref{sec:exp_part2}, we additionally conduct critical experiments on the WRN-40-4~\citep{lin2021mood} backbone to demonstrate the effectiveness of the proposed UM and UMAP. In Figure~\ref{fig10:abla_app_4}, we can find during the model training phase on ID data, there also exists the overlaid OOD detection capability can be explored in later development. In Table~\ref{tab:label_wrn}, we show the comparison of multiple OOD detection baselines, evaluating the OOD performance on the different OOD datasets mentioned in Section~\ref{sec:exp_part1}. The results again demonstrate that our proposed method indeed excavates the intrinsic detection capability and improves the performance. 

\begin{table}[t!]
    \caption{Results of WRN-40-4. Comparison with competitive OOD detection baselines ($\%$). We respectively train WRN-40-4 on CIFAR-10 and CIFAR-100. For those methods involving outliers, we retrieve $5000$ samples from ImageNet-1k. $\uparrow$ indicates higher values are better, and $\downarrow$ indicates lower values are better.}
    \vspace{2mm}
    \centering
    \footnotesize
    \resizebox{\textwidth}{!}{
    \begin{tabular}{c|l|cccc}
        \toprule[1.5pt]
        $\mathcal{D}_\text{in}$ &  Method & AUROC$\uparrow$ & AUPR$\uparrow$ & FPR95$\downarrow$ & ID-ACC$\uparrow$ \\
        \midrule[0.6pt]
        \multirow{6}*{\textbf{CIFAR-10}}
         & MSP\citep{hendrycks17baseline} & $87.12\pm0.25 $ & $87.84\pm0.30$ & $68.29\pm0.96$ & $\textbf{93.86}\pm\textbf{0.19}$ \\
         & ODIN\citep{LiangLS18} & $83.29\pm0.72 $ & $82.74\pm0.79$ & $65.68\pm0.77$ & $\textbf{93.86}\pm\textbf{0.19}$ \\
         & Mahalanobis\citep{10.5555/3327757.3327819} & $77.57\pm0.28 $ & $76.11\pm0.10$ & $61.18\pm0.10$ & $\textbf{93.86}\pm\textbf{0.19}$ \\
         & Energy\citep{liu2020energy} & $87.69\pm0.54 $ & $88.16\pm0.69$ & $58.47\pm1.94$ & $\textbf{93.86}\pm\textbf{0.19}$ \\
         & \textbf{Energy+UM} (ours) & $\textbf{91.74}\pm\textbf{0.43} $ & $\textbf{92.67}\pm\textbf{0.52}$ & $\textbf{40.40}\pm\textbf{1.32}$ & $92.68\pm0.23$  \\
         & \textbf{Energy+UMAP} (ours) & $88.84\pm1.02 $ & $89.31\pm1.44$ & $50.23\pm2.25$ & $\textbf{93.86}\pm\textbf{0.19}$ \\
        \midrule[0.6pt]
        \multirow{6}*{\textbf{CIFAR-100}}
         & MSP\citep{hendrycks17baseline} & $72.34\pm0.63 $ & $72.69\pm0.44$ & $85.40\pm0.59$ &  $\textbf{75.01}\pm\textbf{0.07}$ \\
         & ODIN\citep{LiangLS18} & $68.78\pm0.67 $ & $66.92\pm0.72$ & $85.28\pm0.64$ &  $\textbf{75.01}\pm\textbf{0.07}$  \\
         & Mahalanobis\citep{10.5555/3327757.3327819} & $68.20\pm0.99 $ & $68.30\pm1.15$ & $76.46\pm2.02$ &  $\textbf{75.01}\pm\textbf{0.07}$  \\
         & Energy\citep{liu2020energy} & $74.00\pm 0.41$ & $73.02\pm0.47$ & $81.37\pm0.08$ &  $\textbf{75.01}\pm\textbf{0.07}$  \\
         & \textbf{Energy+UM} (ours) & $76.07\pm0.04 $ & $76.94\pm0.06$ & $74.29\pm1.66$ & $59.08\pm2.75$  \\
         & \textbf{Energy+UMAP} (ours) & $\textbf{77.35}\pm\textbf{0.78} $ & $\textbf{77.43}\pm\textbf{0.91}$ & $\textbf{68.20}\pm\textbf{0.06}$ & $\textbf{75.01}\pm\textbf{0.07}$ \\
        \bottomrule[1.5pt]
    \end{tabular}}
    \label{tab:label_wrn}
\end{table}

As for the fine-grained results of WRN-40-4, we report results on $6$ OOD datasets respectively. When trained on CIFAR-10, UM can outstrip all the scoring function baselines on $5$ OOD datasets except Textures on which Mahalanobis performs better while UMAP still has excellent OOD performance ranking only second to UM. When trained on CIFAR-100, UM and UMAP can also outperform the baselines on most OOD datasets. The fine-grained results of WRN-40-4 further demonstrate the effectiveness of the proposed UM/UMAP on other architectures. The future extension can also take other advanced model structures for OOD detection~\citep{ming2022delving} into consideration. 

\begin{table}[t!]
    \caption{Fine-grained Results of WRN-40-4 on CIFAR-10 ($\%$). Comparison on different OOD benchmark datasets. $\uparrow$ indicates higher values are better, and $\downarrow$ indicates lower values are better.}
    \vspace{2mm}
    \centering
    \footnotesize
    \renewcommand\arraystretch{0.95}
    \resizebox{\textwidth}{!}{
    \begin{tabular}{c|l|cccccc}
        \toprule[1.5pt]
        \multirow{3}*{\textbf{ID dataset}} & \multirow{3}*{\textbf{Method}} & \multicolumn{6}{c}{\textbf{OOD dataset}} \\
        ~ & ~ & \multicolumn{2}{c}{\textbf{CIFAR-100}} & \multicolumn{2}{c}{\textbf{Textures}} & \multicolumn{2}{c}{\textbf{Places365}}  \\
        ~ & ~ & FPR95$\downarrow$ & AUROC$\uparrow$ & FPR95$\downarrow$ & AUROC$\uparrow$ & FPR95$\downarrow$ & AUROC$\uparrow$ \\
        \midrule[0.6pt]
        \multirow{14}*{\textbf{CIFAR-10}}
         & MSP & $70.96\pm0.70$ & $86.08\pm0.08$ & $68.81\pm1.29$ & $86.53\pm0.83$ & $68.31\pm0.25$ & $86.71\pm0.13$\\
         & ODIN & $64.97\pm0.08$ & $83.36\pm0.11$ & $66.86\pm2.24$ & $81.34\pm0.81$ & $66.49\pm1.16$ & $83.47\pm0.93$\\
         & Mahalanobis & $79.84\pm0.55$ & $70.33\pm0.24$ & $\textbf{22.56}\pm\textbf{0.08}$ & $\textbf{94.07}\pm\textbf{0.04}$ & $85.09\pm0.59$ & $67.90\pm0.37$\\
         & Energy & $61.09\pm0.58$ & $86.66\pm0.04$ & $64.29\pm1.72$ & $85.56\pm0.53$ & $55.32\pm0.13$ & $88.29\pm0.26$\\
         & \textbf{Energy+UM} (ours) & $\textbf{57.21}\pm\textbf{1.41}$ & $\textbf{87.56}\pm\textbf{0.15}$ & $46.49\pm1.03$ & $89.74\pm0.45$ & $\textbf{40.68}\pm\textbf{4.46}$ & $\textbf{92.51}\pm\textbf{0.97}$\\
         & \textbf{Energy+UMAP} (ours) & $65.45\pm1.10$ & $84.65\pm0.95$ & $59.14\pm1.64$ & $85.27\pm1.74$ & $48.16\pm1.89$ & $90.43\pm0.47$\\
         \cmidrule{2-8}
         ~ & \multirow{2}*{\textbf{Method}} &\multicolumn{2}{c}{\textbf{SUN}} &
        \multicolumn{2}{c}{\textbf{LSUN}} & \multicolumn{2}{c}{\textbf{iNaturalist}}\\
        ~ & ~ & FPR95$\downarrow$ & AUROC$\uparrow$ & FPR95$\downarrow$ & AUROC$\uparrow$ & FPR95$\downarrow$ & AUROC$\uparrow$ \\
        \cmidrule{2-8}
         & MSP & $68.62\pm0.50$ & $86.95\pm0.23$ & $52.97\pm3.07$ & $92.41\pm0.18$ & $76.05\pm0.01$ & $83.44\pm0.36$\\
         & ODIN & $65.47\pm0.78$ & $83.79\pm1.10$ & $31.89\pm3.44$ & $94.34\pm0.89$ & $79.28\pm0.18$ & $79.80\pm0.35$\\
         & Mahalanobis & $82.92\pm0.28$ & $70.52\pm0.47$ & $64.31\pm0.57$ & $67.75\pm0.55$ & $81.50\pm2.91$ & $74.97\pm2.91$\\
         & Energy & $54.88\pm0.18$ & $88.67\pm0.30$ & $24.99\pm1.38$ & $95.98\pm0.37$ & $75.89\pm0.85$ & $82.40\pm0.22$\\
         & \textbf{Energy+UM} (ours) & $\textbf{38.92}\pm\textbf{3.46}$ & $\textbf{92.98}\pm\textbf{0.95}$ & $\textbf{8.38}\pm\textbf{0.77}$ & $\textbf{98.18}\pm\textbf{0.16}$ & $\textbf{66.02}\pm\textbf{6.70}$ & $\textbf{85.22}\pm\textbf{3.42}$\\
         & \textbf{Energy+UMAP} (ours) & $45.94\pm2.64$ & $91.27\pm0.56$ & $14.10\pm0.04$ & $97.46\pm0.14$ & $74.69\pm0.15$ & $81.13\pm1.12$\\
         
        \bottomrule[1.5pt]
    \end{tabular}}
    \label{tab:label_wrn_zoom_in_cifar10}
\end{table}

\begin{table}[t!]
    \caption{Fine-grained Results of WRN-40-4 on CIFAR-100 ($\%$). Comparison on different OOD benchmark datasets. $\uparrow$ indicates higher values are better, and $\downarrow$ indicates lower values are better.}
    \vspace{2mm}
    \centering
    \footnotesize
    \renewcommand\arraystretch{0.95}
    \resizebox{\textwidth}{!}{
    \begin{tabular}{c|l|cccccc}
        \toprule[1.5pt]
        \multirow{3}*{\textbf{ID dataset}} & \multirow{3}*{\textbf{Method}} & \multicolumn{6}{c}{\textbf{OOD dataset}} \\
        ~ & ~ & \multicolumn{2}{c}{\textbf{CIFAR-10}} & \multicolumn{2}{c}{\textbf{Textures}} & \multicolumn{2}{c}{\textbf{Places365}}  \\
        ~ & ~ & FPR95$\downarrow$ & AUROC$\uparrow$ & FPR95$\downarrow$ & AUROC$\uparrow$ & FPR95$\downarrow$ & AUROC$\uparrow$ \\
        \midrule[0.6pt]
        \multirow{14}*{\textbf{CIFAR-100}}
         & MSP & $83.83\pm0.29$ & $75.50\pm0.21$ & $86.15\pm0.23$ & $72.36\pm0.40$ & $86.72\pm0.29$ & $69.60\pm0.27$\\
         & ODIN & $83.70\pm0.30$ & $74.32\pm0.03$ & $81.57\pm1.74$ & $71.67\pm0.28$ & $88.07\pm0.11$ & $64.83\pm1.36$\\
         & Mahalanobis & $96.89\pm0.11$ & $68.78\pm0.67$ & $\textbf{31.02}\pm\textbf{2.04}$ & $\textbf{91.85}\pm\textbf{0.91}$ & $93.34\pm0.44$ & $61.28\pm1.62$\\
         & Energy & $\textbf{81.32}\pm\textbf{0.47}$ & $\textbf{77.49}\pm\textbf{0.26}$ & $86.38\pm0.49$ & $73.50\pm0.45$ & $84.45\pm0.38$ & $69.82\pm0.63$\\
         & \textbf{Energy+UM} (ours) & $89.23\pm1.51$ & $63.85\pm1.73$ & $78.90\pm0.07$ & $72.58\pm1.33$ & $\textbf{80.46}\pm\textbf{1.99}$ & $70.49\pm1.01$\\
         & \textbf{Energy+UMAP} (ours) & $94.11\pm0.72$ & $60.77\pm0.96$ & $66.94\pm4.49$ & $75.82\pm5.33$ & $82.59\pm0.27$ & $\textbf{71.92}\pm\textbf{3.58}$\\
         \cmidrule{2-8}
         ~ & \multirow{2}*{\textbf{Method}} &\multicolumn{2}{c}{\textbf{SUN}} &
        \multicolumn{2}{c}{\textbf{LSUN}} & \multicolumn{2}{c}{\textbf{iNaturalist}}\\
        ~ & ~ & FPR95$\downarrow$ & AUROC$\uparrow$ & FPR95$\downarrow$ & AUROC$\uparrow$ & FPR95$\downarrow$ & AUROC$\uparrow$ \\
        \cmidrule{2-8}
         & MSP & $88.88\pm0.83$ & $65.22\pm0.85$ & $78.56\pm0.66$ & $79.10\pm0.53$ & $86.72\pm0.29$ & $73.75\pm0.56$\\
         & ODIN & $91.00\pm0.10$ & $59.06\pm1.81$ & $70.14\pm2.42$ & $84.03\pm0.86$ & $87.86\pm1.13$ & $64.52\pm1.51$\\
         & Mahalanobis & $94.22\pm0.01$ & $60.09\pm1.56$ & $89.73\pm2.87$ & $40.81\pm3.07$ & $87.25\pm3.28$ & $74.98\pm2.85$\\
         & Energy & $88.35\pm0.52$ & $64.04\pm0.76$ & $59.84\pm0.06$ & $87.91\pm0.53$ & $88.91\pm0.78$ & $67.81\pm0.91$\\
         & \textbf{Energy+UM} (ours) & $84.04\pm0.09$ & $67.19\pm0.14$ & $33.87\pm1.21$ & $92.29\pm0.02$ & $76.91\pm6.07$ & $79.28\pm4.17$\\
         & \textbf{Energy+UMAP} (ours) & $\textbf{80.53}\pm\textbf{1.31}$ & $\textbf{72.68}\pm\textbf{4.50}$ & $\textbf{27.79}\pm\textbf{2.19}$ & $\textbf{93.39}\pm\textbf{0.57}$ & $\textbf{55.53}\pm\textbf{4.61}$ & $\textbf{85.65}\pm\textbf{0.83}$\\
         
        \bottomrule[1.5pt]
    \end{tabular}}
    \label{tab:label_wrn_zoom_in_cifar100}
\end{table}

\begin{table}[t!]
    \caption{Results of DenseNet-101 when SVHN as OOD set ($\%$). Comparison on different ID benchmark datasets. $\uparrow$ indicates higher values are better, and $\downarrow$ indicates lower values are better.}
    \vspace{2mm}
    \centering
    \footnotesize
    \renewcommand\arraystretch{0.95}
    \resizebox{\textwidth}{!}{
    \begin{tabular}{c|l|ccc|ccc}
        \toprule[1.5pt]
        \multirow{3}*{\textbf{OOD dataset}} & \multirow{3}*{\textbf{Method}} & \multicolumn{6}{c}{\textbf{ID dataset}} \\
        ~ & ~ & \multicolumn{3}{c|}{\textbf{CIFAR-10}} & \multicolumn{3}{c}{\textbf{CIFAR-100}} \\
        ~ & ~ & AUROC$\uparrow$ & AUPR$\uparrow$ & FPR95$\downarrow$ & AUROC$\uparrow$ & AUPR$\uparrow$ & FPR95$\downarrow$ \\
        \midrule[0.6pt]
        \multirow{15}*{\textbf{SVHN}}
         & MSP & $90.67\pm1.35$ & $87.36\pm2.53$ & $63.35\pm3.54$ & $75.00\pm4.09$ & $64.57\pm5.21$ & $84.11\pm4.64$\\
         & ODIN & $92.19\pm1.14$ & $86.30\pm2.17$ & $43.27\pm5.07$ & $72.64\pm2.02$ & $60.22\pm2.09$ & $91.43\pm2.32$\\
         & Mahalanobis & $\textbf{98.12}\pm\textbf{0.92}$ & $93.76\pm5.52$ & $\textbf{7.09}\pm\textbf{1.95}$ & $\textbf{92.63}\pm\textbf{1.04}$ & $\textbf{82.93}\pm\textbf{5.01}$ & $\textbf{32.55}\pm\textbf{2.93}$\\
         & Energy & $92.84\pm1.86$ & $89.23\pm3.62$ & $46.37\pm7.31$ & $77.64\pm3.76$ & $68.45\pm5.28$ & $86.50\pm2.98$\\
         & \textbf{Energy+UM} (ours) & $95.42\pm1.77$ & $93.29\pm2.61$ & $28.90\pm13.53$ & $78.80\pm7.41$ & $70.60\pm9.59$ & $85.67\pm7.28$\\
         & \textbf{Energy+UMAP} (ours) & $97.37\pm1.76$ & $\textbf{95.67}\pm\textbf{2.45}$ & $15.04\pm11.34$ & $83.31\pm5.31$ & $75.10\pm6.76$ & $75.25\pm9.74$\\
         \cmidrule{2-8}
         & OE & $98.97\pm0.05$ & $98.15\pm0.15$ & $3.91\pm0.16$ & $95.50\pm1.49$ & $91.86\pm2.52$ & $21.67\pm8.92$\\
         & Energy (w. $\mathcal{D}_\text{aux}$) & $98.93\pm0.23$ & $98.07\pm0.36$ & $3.17\pm0.33$ & $96.38\pm0.63$ & $94.13\pm1.01$ & $19.14\pm6.80$\\
         & POEM & $94.37\pm0.07$ & $94.50\pm0.06$ & $18.50\pm0.33$ & $94.00\pm0.98$ & $89.25\pm1.99$ & $36.56\pm4.43$\\
         & \textbf{OE + UM} (ours) & $99.43\pm0.09$ & $\textbf{98.98}\pm\textbf{0.17}$ & $\textbf{1.73}\pm\textbf{0.41}$ & $95.96\pm0.95$ & $92.76\pm1.53$ & $20.09\pm6.58$\\
         & \textbf{Energy (w. $\mathcal{D}_\text{aux}$)+ UM} (ours) & $98.16\pm0.48$ & $96.67\pm0.95$ & $7.18\pm2.74$ & $96.22\pm1.78$ & $93.40\pm2.75$ & $18.46\pm12.78$\\
         & \textbf{POEM + UM} (ours) & $98.63\pm0.32$ & $97.49\pm0.51$ & $4.08\pm1.41$ & $95.90\pm0.25$ & $93.89\pm0.01$ & $26.66\pm4.39$\\
         & \textbf{OE + UMAP} (ours) & $\textbf{99.46}\pm\textbf{0.01}$ & $98.85\pm0.01$ & $2.27\pm0.00$ & $95.05\pm0.28$ & $90.57\pm0.18$ & $24.20\pm2.05$\\
         & \textbf{Energy (w. $\mathcal{D}_\text{aux}$) + UMAP} (ours) & $99.21\pm0.11$ & $98.73\pm0.21$ & $1.91\pm0.47$ & $\textbf{97.05}\pm\textbf{0.99}$ & $\textbf{94.41}\pm\textbf{1.99}$ & $\textbf{12.52}\pm\textbf{4.99}$\\
         & \textbf{POEM + UMAP} (ours) & $99.36\pm0.19$ & $98.73\pm0.47$ & $2.58\pm1.53$ & $96.04\pm0.95$ & $93.83\pm1.36$ & $23.57\pm5.85$\\
        \bottomrule[1.5pt]
    \end{tabular}}
    \label{tab:label_zoom_in_SVHN_densenet}
\end{table}

\begin{table}[t!]
    \caption{Results of WRN-40-4 when SVHN as OOD set ($\%$). Comparison on different ID benchmark datasets. $\uparrow$ indicates higher values are better, and $\downarrow$ indicates lower values are better.}
    \vspace{2mm}
    \centering
    \footnotesize
    \renewcommand\arraystretch{0.95}
    \resizebox{\textwidth}{!}{
    \begin{tabular}{c|l|ccc|ccc}
        \toprule[1.5pt]
        \multirow{3}*{\textbf{OOD dataset}} & \multirow{3}*{\textbf{Method}} & \multicolumn{6}{c}{\textbf{ID dataset}} \\
        ~ & ~ & \multicolumn{3}{c|}{\textbf{CIFAR-10}} & \multicolumn{3}{c}{\textbf{CIFAR-100}} \\
        ~ & ~ & AUROC$\uparrow$ & AUPR$\uparrow$ & FPR95$\downarrow$ & AUROC$\uparrow$ & AUPR$\uparrow$ & FPR95$\downarrow$ \\
        \midrule[0.6pt]
        \multirow{6}*{\textbf{SVHN}}
         & MSP & $87.68\pm0.65$ & $81.94\pm1.75$ & $72.35\pm0.88$ & $70.83\pm2.06$ & $56.20\pm0.88$ & $86.30\pm2.09$\\
         & ODIN & $76.98\pm6.71$ & $65.19\pm8.60$ & $84.78\pm8.44$ & $63.00\pm0.59$ & $47.27\pm0.18$ & $94.63\pm1.08$\\
         & Mahalanobis & $\textbf{97.48}\pm\textbf{0.22}$ & $\textbf{94.49}\pm\textbf{0.80}$ & $\textbf{12.07}\pm\textbf{0.71}$ & $\textbf{91.26}\pm\textbf{2.35}$ & $\textbf{84.19}\pm\textbf{3.42}$ & $\textbf{42.72}\pm\textbf{11.32}$\\
         & Energy & $86.26\pm3.92$ & $79.81\pm5.42$ & $72.80\pm13.80$ & $77.47\pm1.78$ & $64.99\pm1.57$ & $80.36\pm1.70$\\
         & \textbf{Energy+UM} (ours) & $95.96\pm1.07$ & $93.79\pm1.42$ & $25.09\pm9.32$ & $86.77\pm2.93$ & $82.39\pm2.40$ & $76.59\pm13.00$\\
         & \textbf{Energy+UMAP} (ours) & $91.67\pm4.21$ & $86.15\pm6.36$ & $44.16\pm17.68$ & $81.24\pm5.09$ & $69.53\pm7.29$ & $69.90\pm12.66$\\
        \bottomrule[1.5pt]
    \end{tabular}}
    \label{tab:label_zoom_in_SVHN_wrn}
\end{table}

\subsection{Additional Experiments on More Advanced Post-hoc and OE-based Methods}

Except for some representative methods (like MSP, Energy, OE, POEM) that have been considered in the experiments, in Table \ref{tab:additional_sota}, we add more advanced post-hoc and OE-based methods \cite{sun2021react,sun2022dice,djurisic2023extremely,katz2022training,wang2023outofdistribution} as comparison to further validate the effectiveness of the proposed UM/UMAP.

\begin{table}[h!]
    \caption{Results of additional comparison with advanced methods of post-hoc scoring functions or fine-tuning with auxiliary outliers ($\%$). $\uparrow$ indicates higher values are better, and $\downarrow$ indicates lower values are better.}
    \vspace{2mm}
    \centering
    \footnotesize
    \begin{tabular}{c|l|ccccc}
        \toprule[1.5pt]
        $\mathcal{D}_\text{in}$ &  Method & AUROC$\uparrow$ & AUPR$\uparrow$ & FPR95$\downarrow$ & ID-ACC$\uparrow$ & w./w.o $\mathcal{D}_\text{aux}$ \\
        \midrule[0.6pt]
        \multirow{9}*{\textbf{CIFAR-10}}
         & ReAct\cite{sun2021react} & $92.76\pm 0.26$ & $93.57\pm 0.34$ & $38.43\pm 1.31$ & $\textbf{93.74}\pm \textbf{0.10}$ & \\
         & DICE\cite{sun2022dice} & $90.66\pm 0.94$ & $91.11\pm 1.04$ & $41.51\pm 2.39$ & $92.55\pm 0.20$ & \\
         & ASH-S\cite{djurisic2023extremely} & $95.13\pm 0.08$ & $95.51\pm 0.14$ & $25.87\pm 0.22$ & $93.69\pm 0.10$ & \\
         & \textbf{ASH-S+UM}(ours) & $\textbf{95.22}\pm \textbf{0.27}$ & $\textbf{95.53}\pm \textbf{0.26}$ & $\textbf{23.94}\pm \textbf{1.27}$ & $92.47\pm 0.13$ & \\
         & \textbf{ASH-S+UMAP}(ours) & $94.57\pm 0.94$ & $94.86\pm 0.95$ & $25.95\pm 3.24$ & $93.69\pm 0.10$ & \\
         \cmidrule{2-7}
         & WOODS\cite{katz2022training} & $97.32\pm0.04$ & $96.78\pm0.12$ & $12.94\pm0.46$ & $\textbf{92.22}\pm\textbf{0.84}$ & $\checkmark$\\
         & DOE\cite{wang2023outofdistribution} & $97.41\pm0.01$ & $97.38\pm0.04$ & $12.84\pm0.23$ & $91.02\pm0.02$ & $\checkmark$\\
         & \textbf{DOE+UM}(ours) & $\textbf{97.49}\pm\textbf{0.00}$ & $\textbf{97.63}\pm\textbf{0.03}$ & $\textbf{11.62}\pm\textbf{0.11}$ & $90.57\pm0.32$ & $\checkmark$\\
         & \textbf{DOE+UMAP}(ours) & $97.48\pm0.01$ & $96.62\pm0.05$ & $11.85\pm0.16$ & $91.02\pm0.02$ & $\checkmark$\\
        \midrule[0.6pt]
        \multirow{9}*{\textbf{CIFAR-100}}
         & ReAct\cite{sun2021react} & $75.66\pm 0.30$ & $75.49\pm 0.24$ & $79.04\pm 0.39$ & $73.69\pm 0.22$ & \\
         & DICE\cite{sun2022dice} & $78.41\pm 0.65$ & $79.11\pm 0.52$ & $69.33\pm 2.35$ & $67.77\pm 0.47$ & \\
         & ASH-S\cite{djurisic2023extremely} & $83.53\pm 0.07$ & $84.02\pm 0.10$ & $61.78\pm 0.41$ & $\textbf{74.09}\pm \textbf{0.13}$ & \\
         & \textbf{ASH-S+UM}(ours) & $\textbf{83.60}\pm \textbf{0.37}$ & $\textbf{84.27}\pm \textbf{0.55}$ & $\textbf{56.56}\pm \textbf{2.13}$ & $64.39\pm 0.87$ & \\
         & \textbf{ASH-S+UMAP}(ours) & $81.98\pm 0.04$ & $82.84\pm 0.01$ & $59.10\pm 0.01$ & $\textbf{74.09}\pm \textbf{0.13}$ & \\
         \cmidrule{2-7}
         & WOODS\cite{katz2022training} & $92.68\pm0.12$ & $87.18\pm0.79$ & $33.10\pm1.85$ & $\textbf{73.08}\pm\textbf{0.97}$ & $\checkmark$\\
         & DOE\cite{wang2023outofdistribution} & $92.77\pm0.01$ & $86.79\pm0.00$ & $30.63\pm0.11$ & $73.01\pm0.17$ & $\checkmark$\\
         & \textbf{DOE+UM}(ours) & $92.93\pm0.08$ & $87.33\pm0.24$ & $\textbf{29.47}\pm\textbf{0.15}$ & $71.11\pm0.03$ & $\checkmark$\\
         & \textbf{DOE+UMAP}(ours) & $\textbf{92.95}\pm\textbf{0.01}$ & $\textbf{87.41}\pm\textbf{0.05}$ & $29.61\pm0.32$ & $73.01\pm0.17$ & $\checkmark$\\
        \bottomrule[1.5pt]
    \end{tabular}
    \label{tab:additional_sota}
\end{table}

\subsection{Additional Verification for Intrinsic OOD Discriminative Capability}
\label{app:verf_intrinsic_power}

In Section~\ref{sec:exp_part3}, we display the overlaid OOD detection capability on CIFAR-10 using SVHN as the OOD dataset. Here, we additionally verify the previously observed trend during training when training DenseNet-101 on CIFAR-100 using iNaturalist as an OOD dataset. In Figure~\ref{fig8:abla_app_2}, we trace the three evaluation metrics during training on CIFAR-100 using $4$ different learning rate schedules. Consistent with the original experiment, we still use iNaturalist as the OOD dataset. It can be seen for all three metrics that exists a middle stage where the model has the better OOD detection capability (For FPR95, it is smaller (better) in the middle stage; for AUROC and AUPR, they are higher (better) in the middle stage). Besides that, we also look into the change of OOD performance on other architecture (e.g., WRN-40-4) in Figure~\ref{fig9:abla_app_3} and Figure~\ref{fig10:abla_app_4}. In Figure~\ref{fig9:abla_app_3}, we display the curves of three metrics of WRN-40-4 when trained on CIFAR-10 with SVHN and Textures as OOD datasets.  The trend that the OOD performance first goes better and then converges to worse OOD performance can be reflected. In Figure~\ref{fig10:abla_app_4}, we continually provide curves of the three metrics of WRN-40-4 during training on CIFAR-100 with iNaturalist, Places365, and SUN as OOD datasets. A clear better middle stage can still be 
excavated in this scenario.

\begin{figure}[t!]
    \begin{center}
    \subfigure[CIFAR10 FPR95 Curves]{
    \includegraphics[scale=0.18]{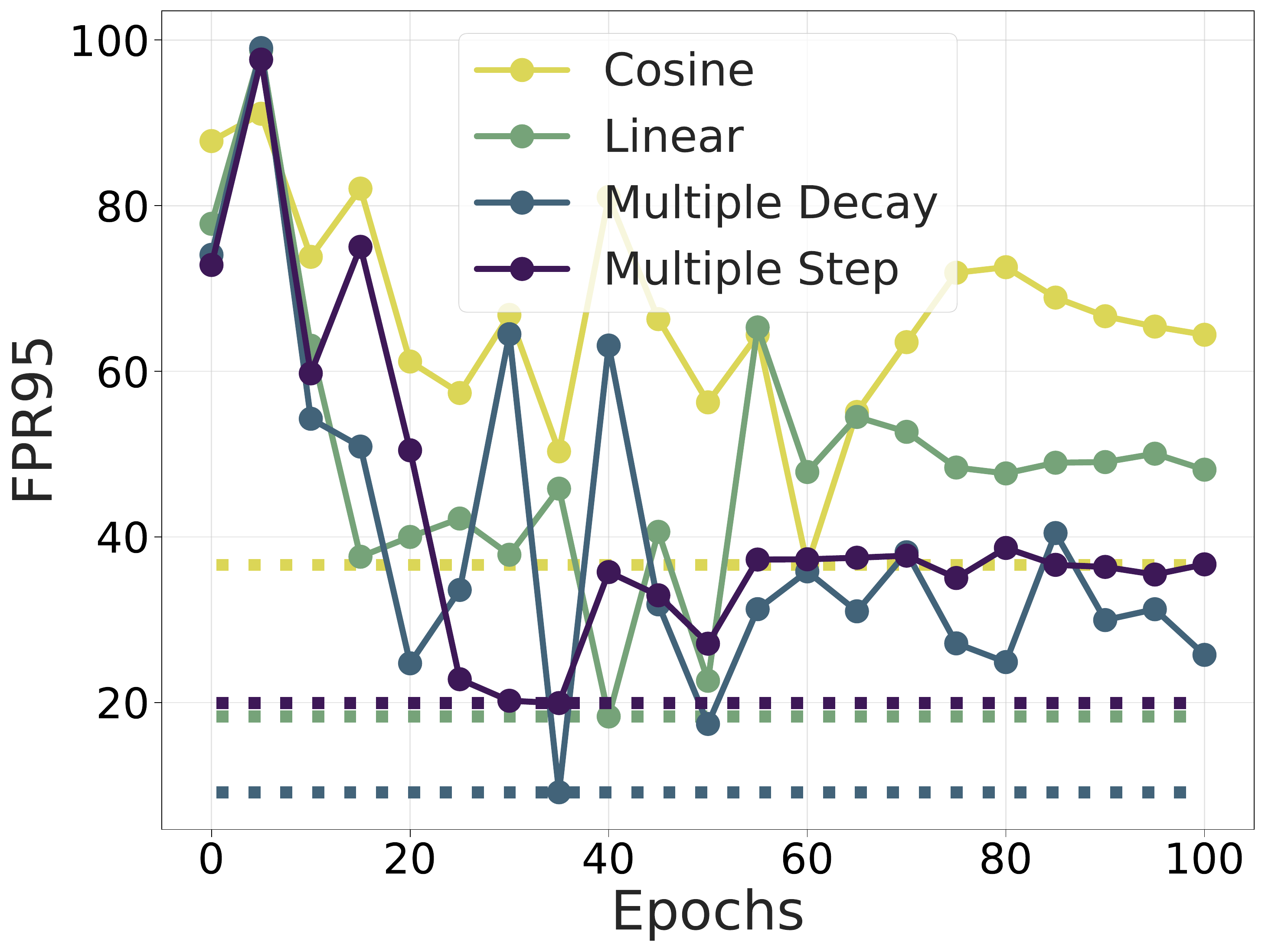}
    \label{fig8:abla_app_2_a}
    }
    \subfigure[CIFAR10 AUROC Curves]{
    \includegraphics[scale=0.18]{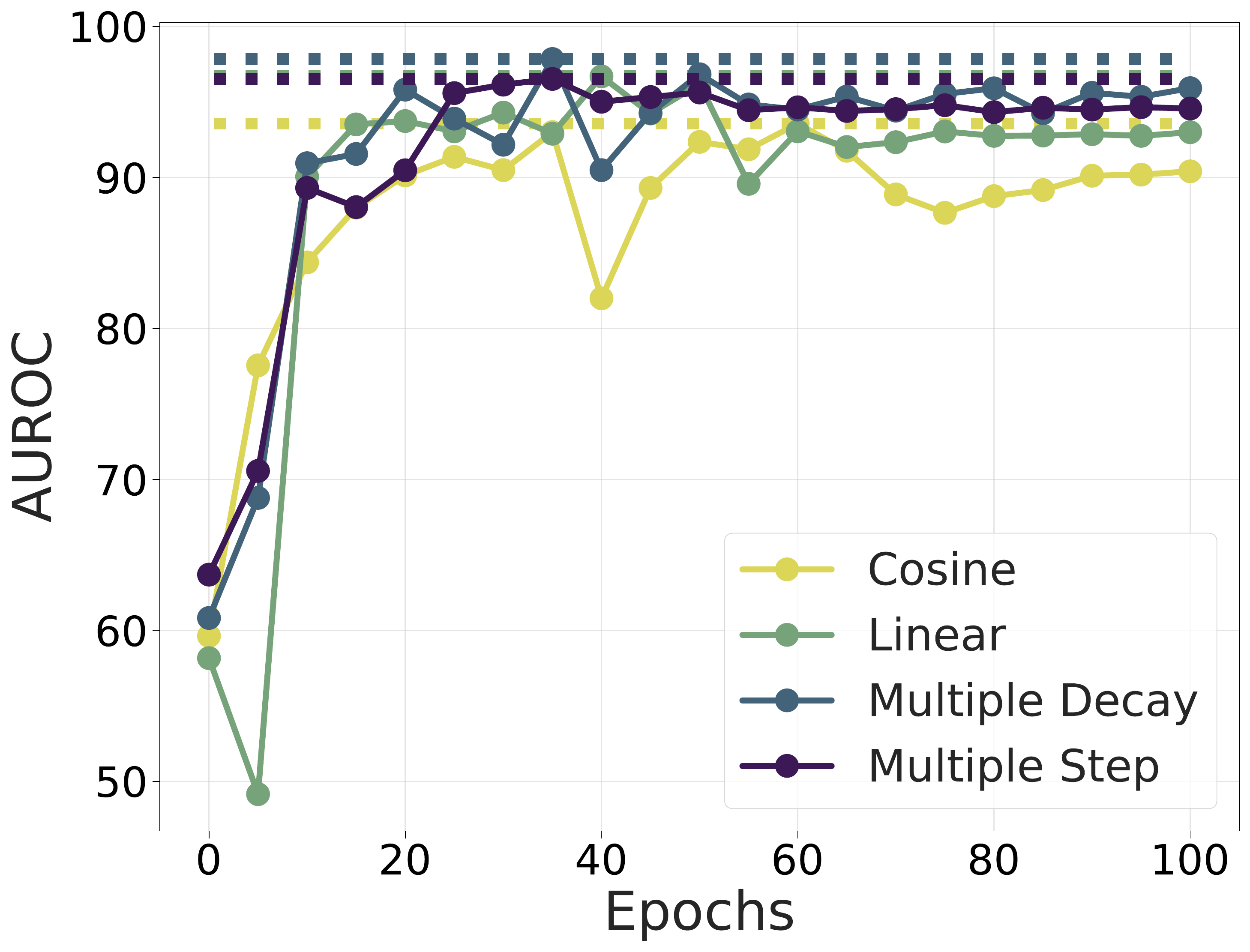}
    \label{fig8:abla_app_2_b}
    }
    \subfigure[CIFAR10 AUPR Curves]{
    \includegraphics[scale=0.18]{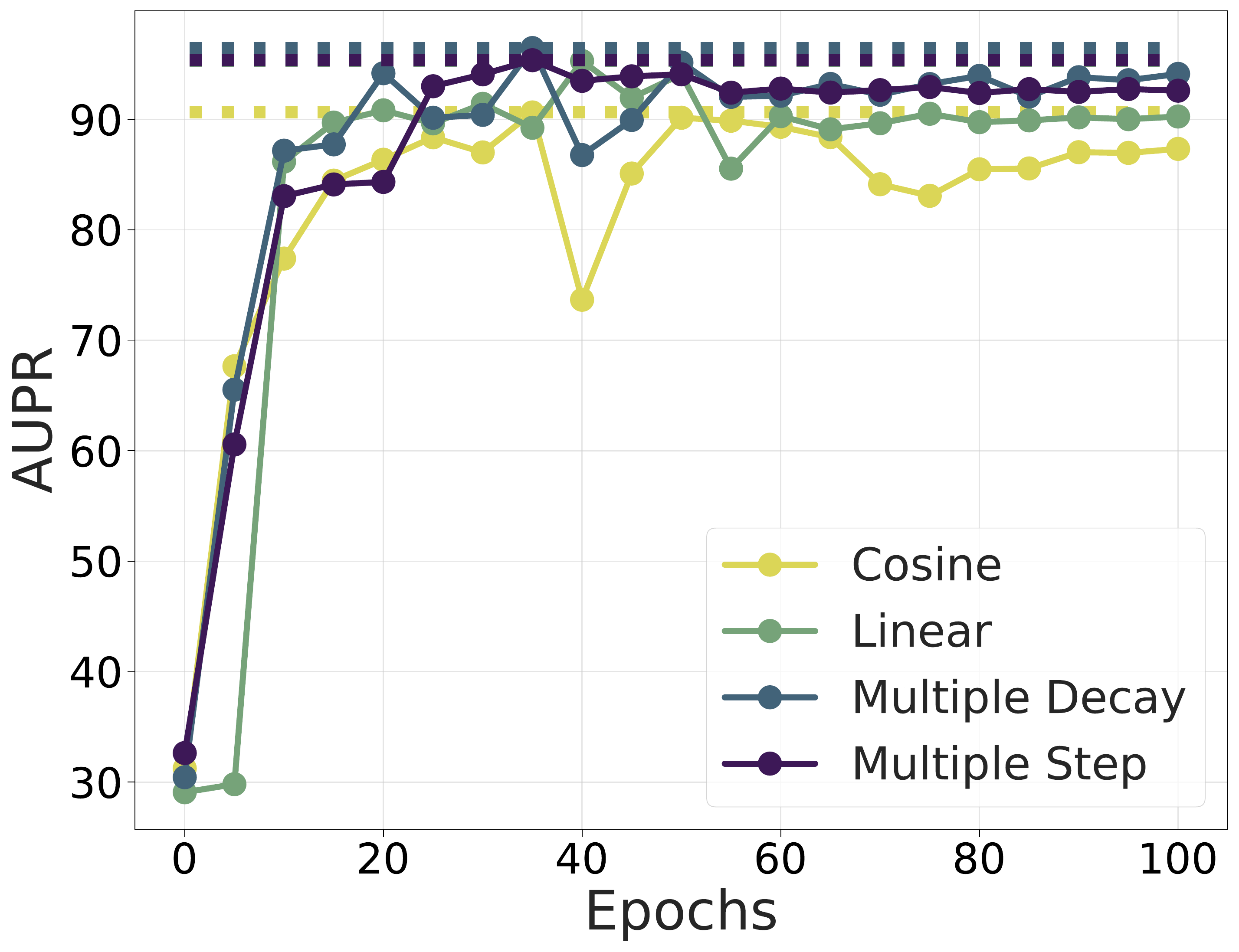}
    \label{fig8:abla_app_2_c}
    }
    \subfigure[CIFAR100 FPR95 Curves]{
    \includegraphics[scale=0.18]{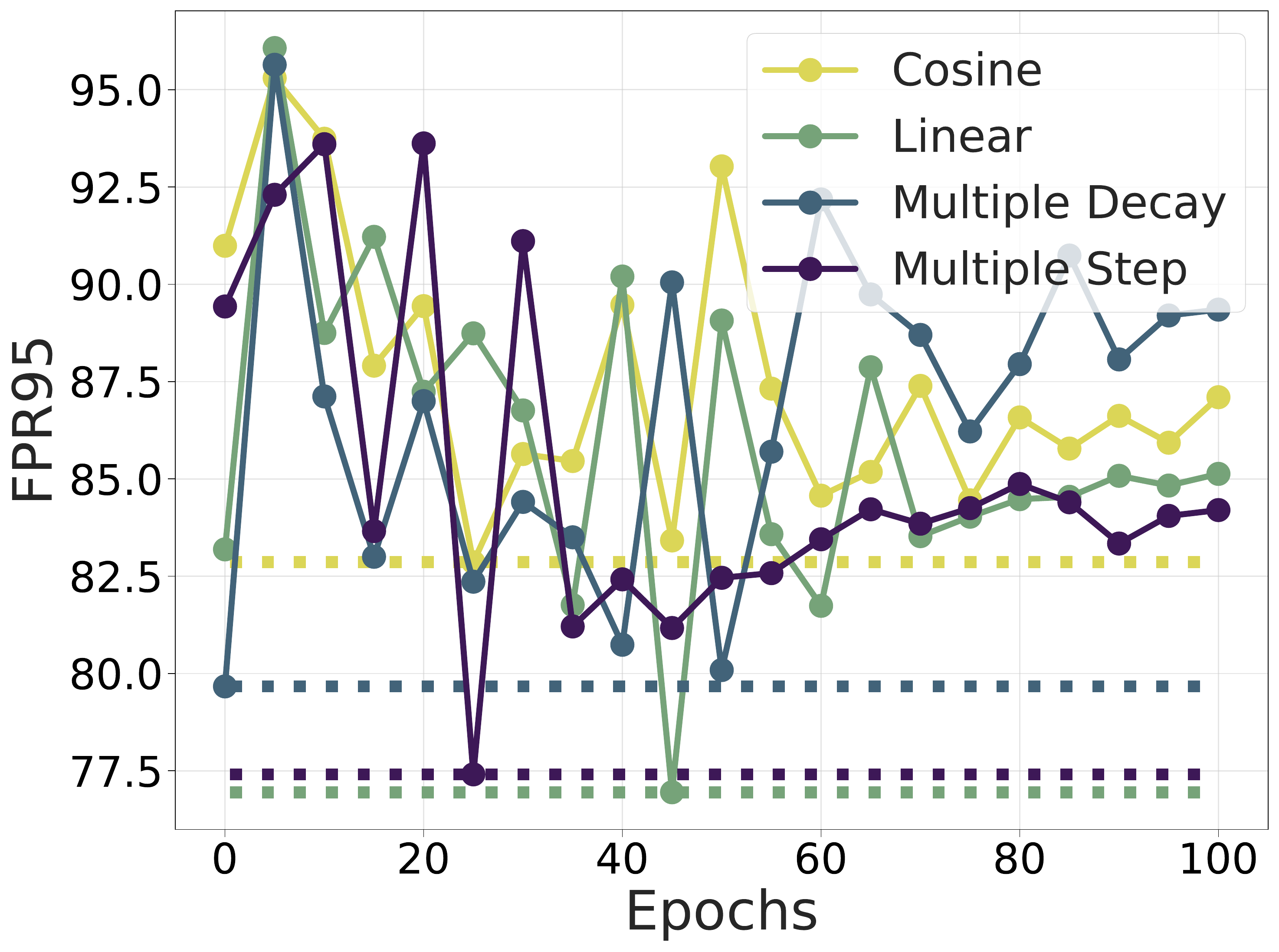}
    \label{fig8:abla_app_2_d}
    }
    \subfigure[CIFAR100 AUROC Curves]{
    \includegraphics[scale=0.18]{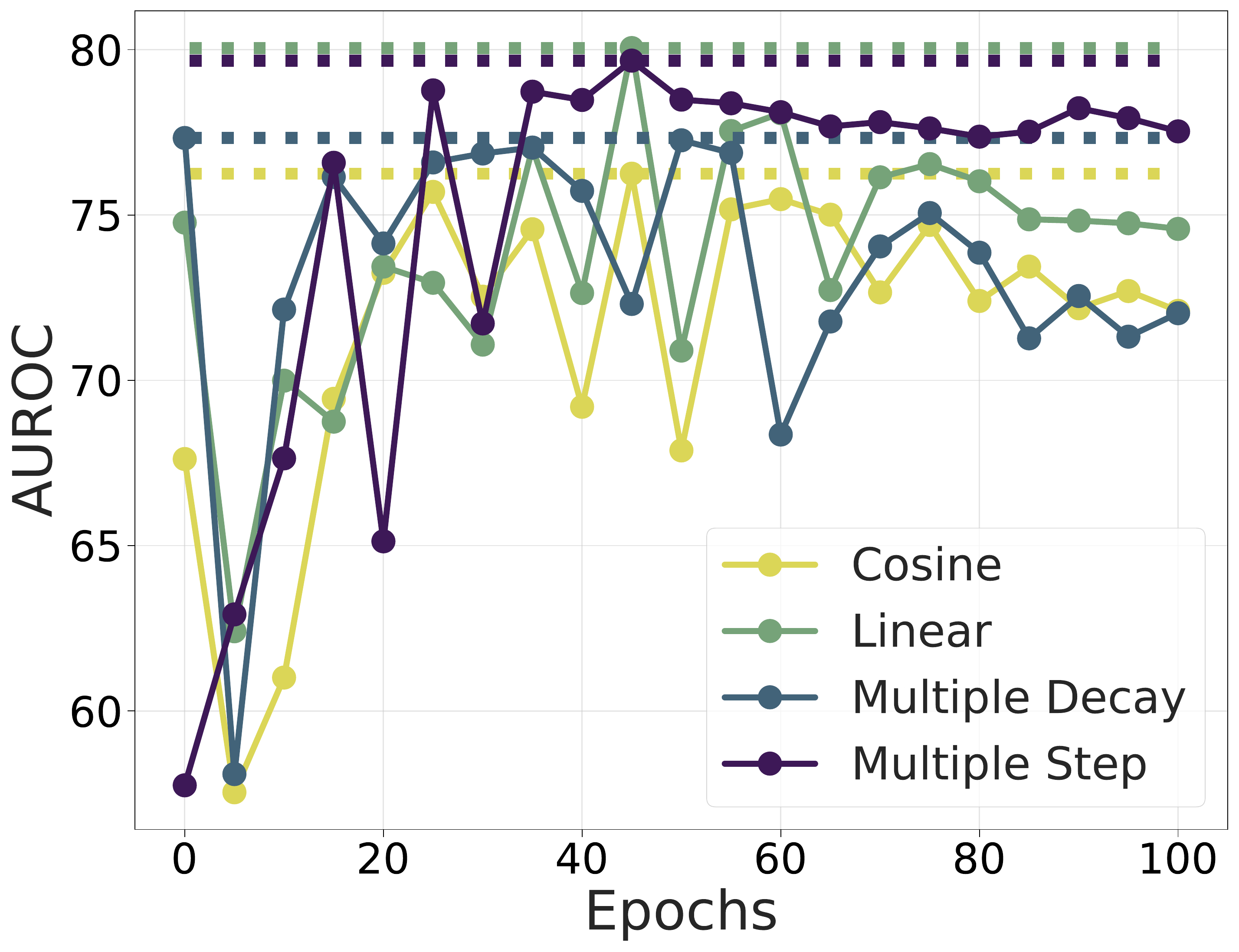}
    \label{fig8:abla_app_2_e}
    }
    \subfigure[CIFAR100 AUPR Curves]{
    \includegraphics[scale=0.18]{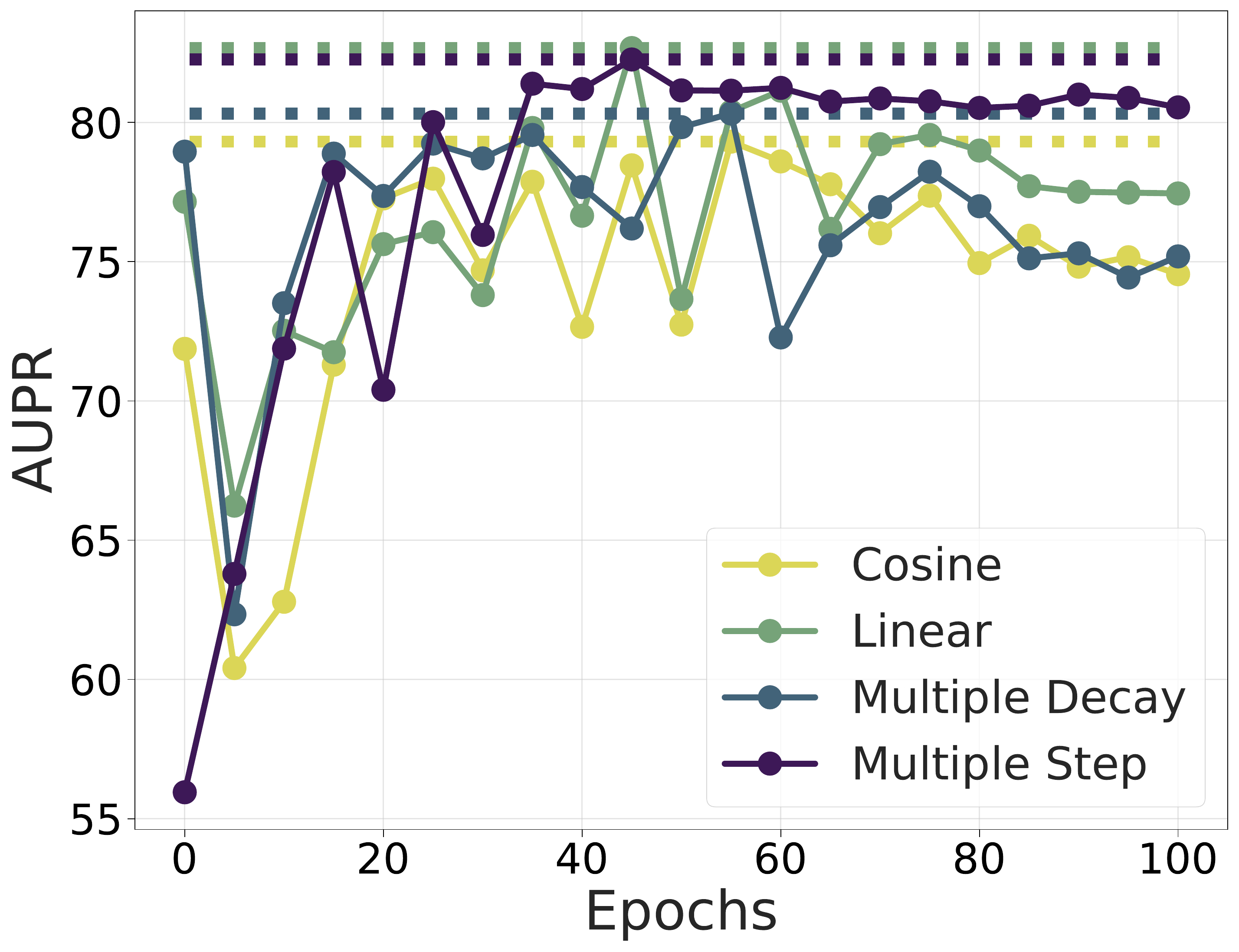}
    \label{fig8:abla_app_2_f}
    }

    \end{center}
    \caption{Ablation studies on three metrics with 4 different learning rate schedules. The model is DenseNet-101 trained on CIFAR-100 with iNaturalist as the OOD dataset. (a) change of FPR95 throughout the pruning phase when training on CIFAR-100; (b) change of AUROC throughout the pruning phase when training on CIFAR-100; (c) change of AUPR throughout the pruning phase when training on CIFAR-100. It demonstrates a better middle stage exists according to the three metrics.}
    \label{fig8:abla_app_2}
\end{figure}

\begin{figure}[t!]
    \begin{center}
    \subfigure[FPR95 Curves]{
    \includegraphics[scale=0.18]{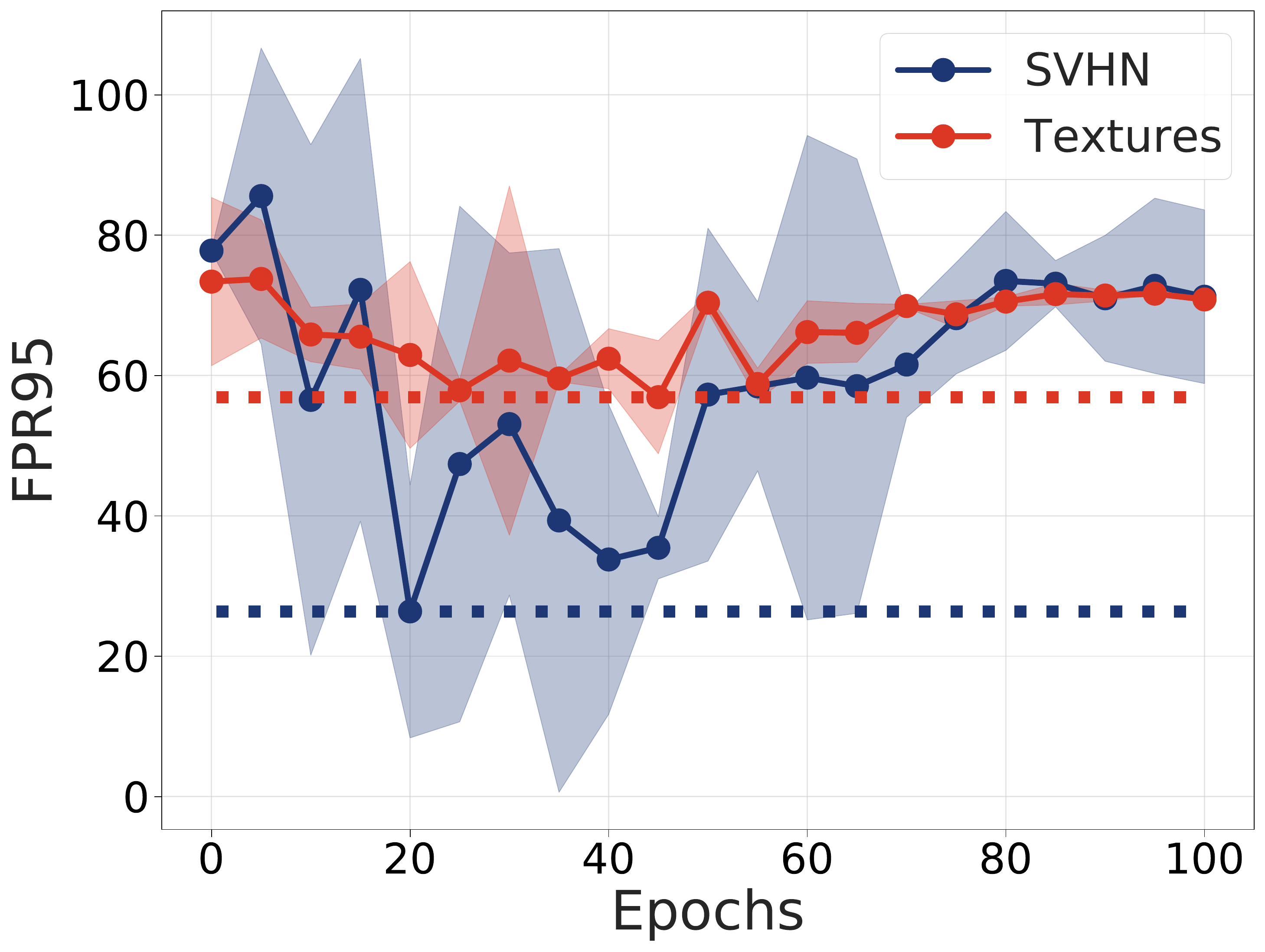}
    \label{fig9:abla_app_3_a}
    }
    \subfigure[AUROC Curves]{
    \includegraphics[scale=0.18]{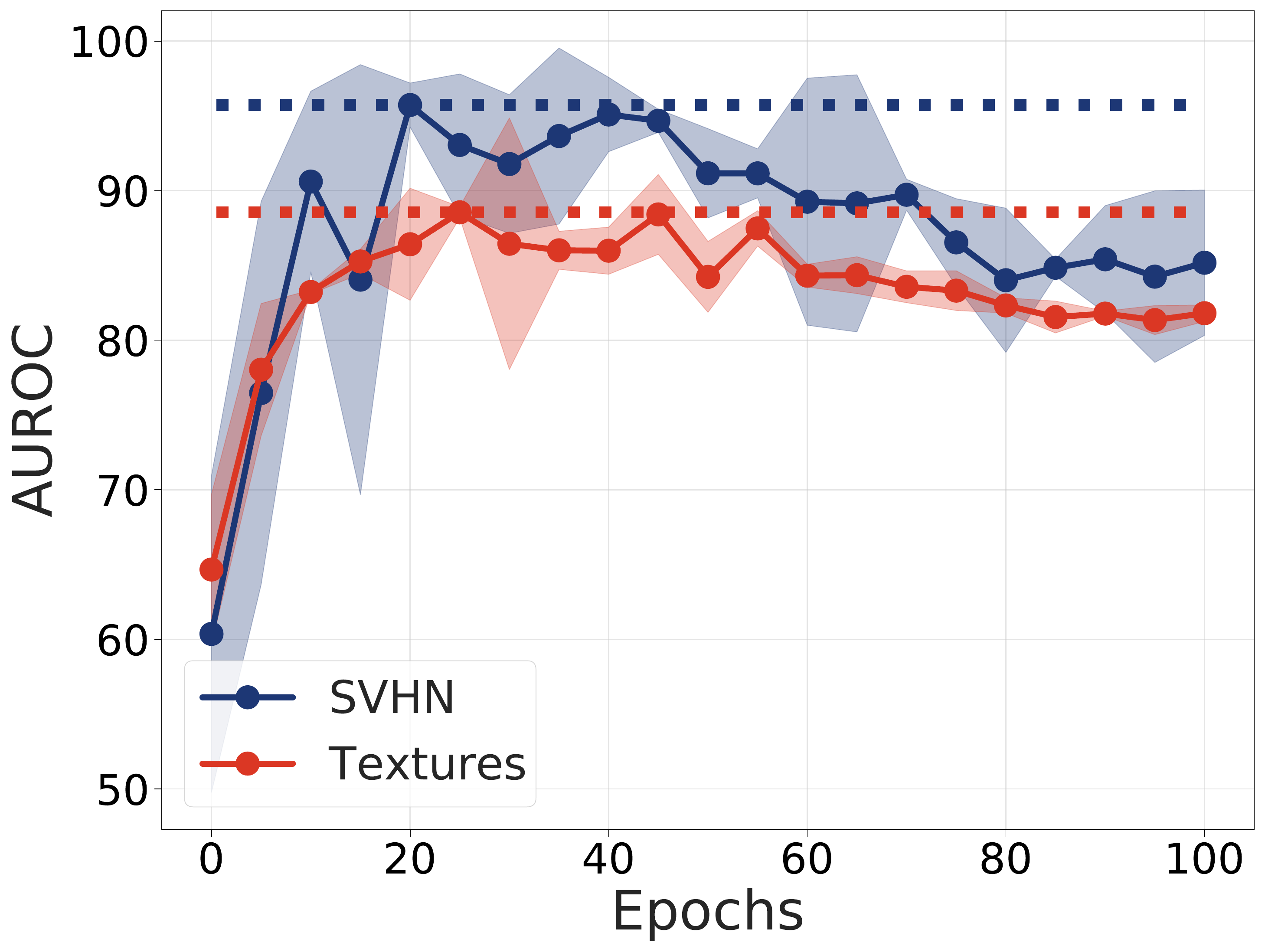}
    \label{fig9:abla_app_3_b}
    }
    \subfigure[AUPR Curves]{
    \includegraphics[scale=0.18]{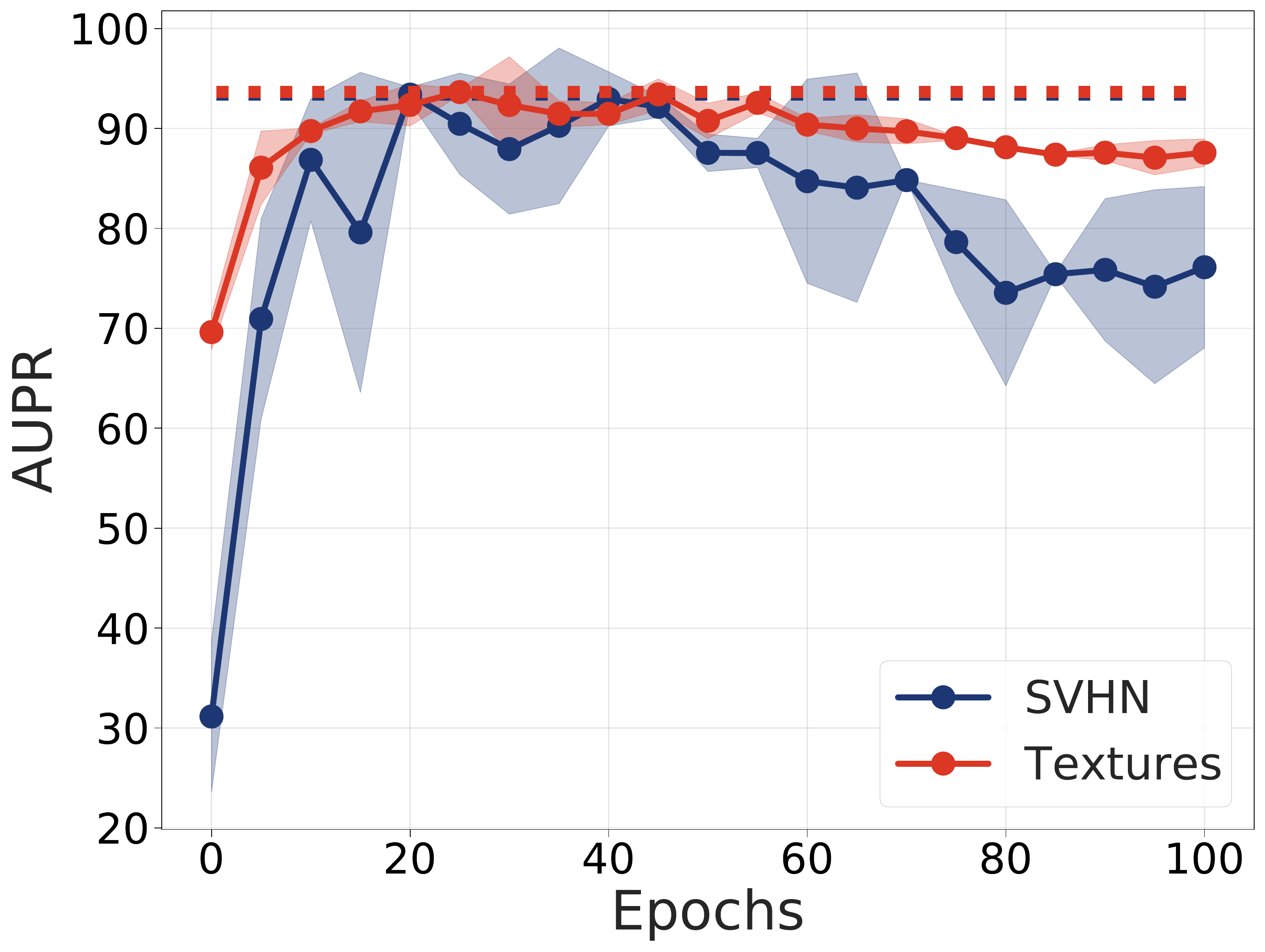}
    \label{fig9:abla_app_3_c}
    }
    \end{center}
    \caption{Ablation studies on three metrics of WRN-40-4 with CIFAR-10 as ID dataset, SVHN, and Textures as OOD datasets. (a) change of FPR95 throughout the pruning phase when training on CIFAR-10; (b) change of AUROC throughout the pruning phase when training on CIFAR-10; (c) change of AUPR throughout the pruning phase when training on CIFAR-10. It demonstrates a better middle stage exists according to the three metrics.}
    \label{fig9:abla_app_3}
\end{figure}

\begin{figure}[t!]
    \begin{center}
    \subfigure[FPR95 Curves]{
    \includegraphics[scale=0.18]{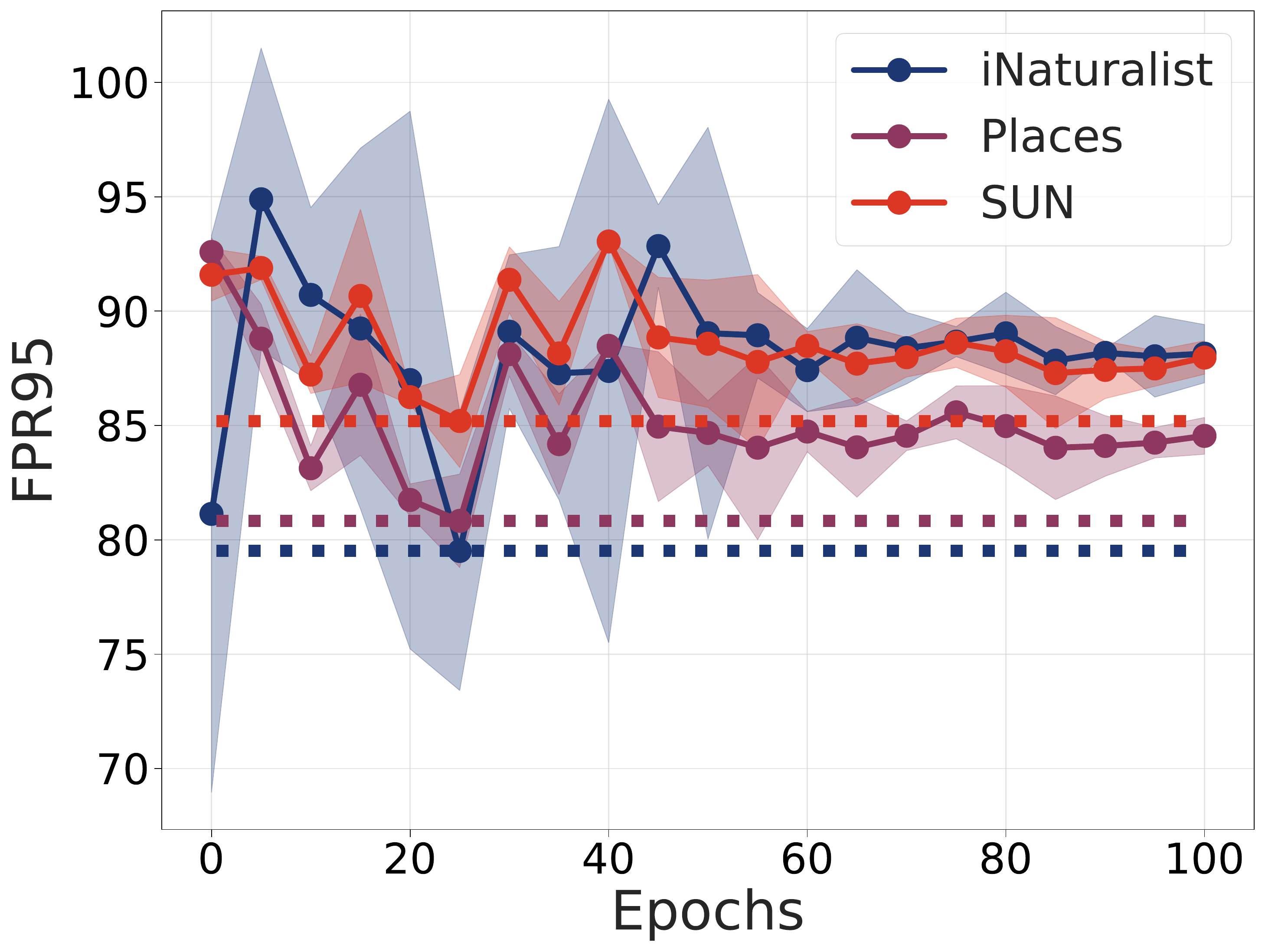}
    \label{fig10:abla_app_4_a}
    }
    \subfigure[AUROC Curves]{
    \includegraphics[scale=0.18]{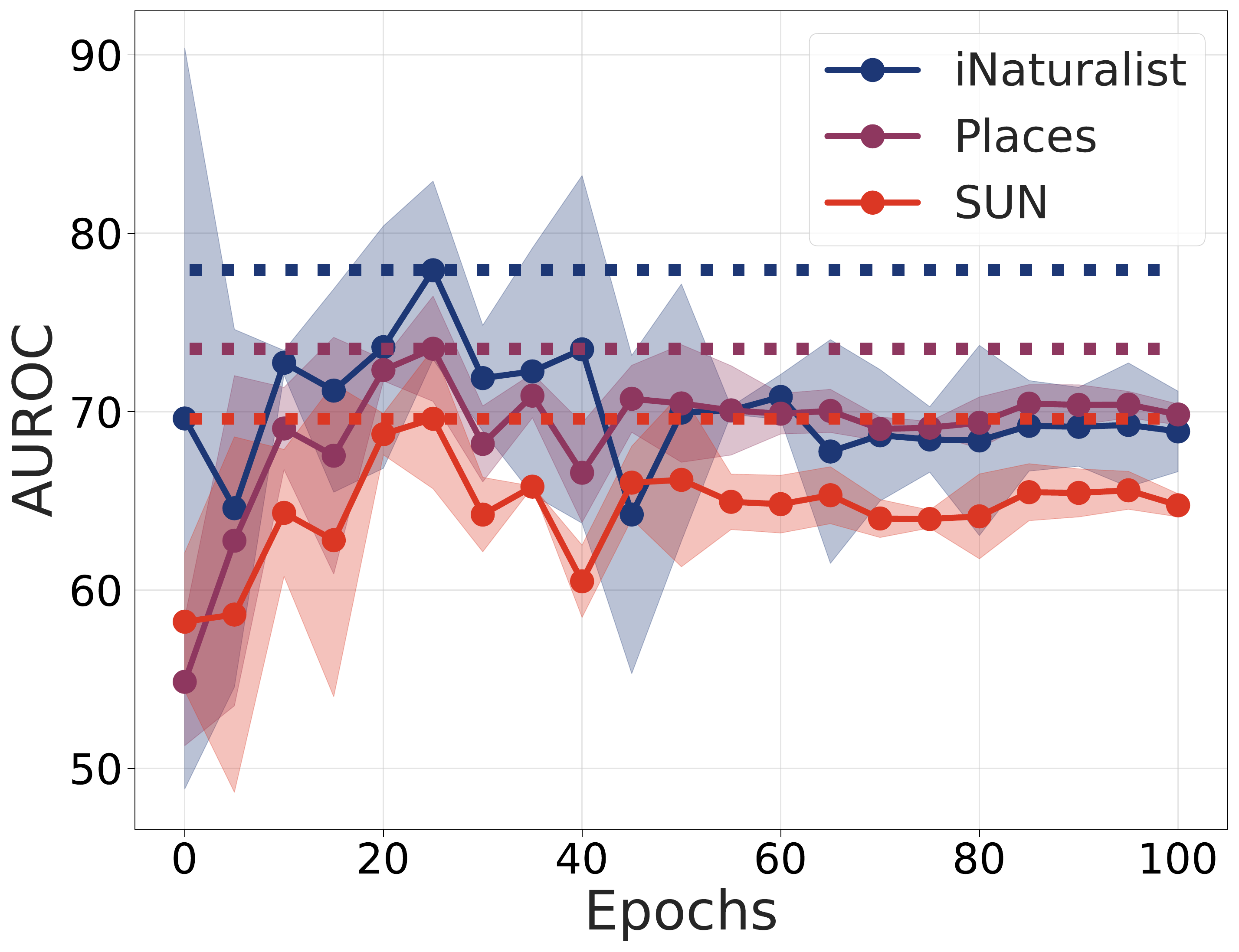}
    \label{fig10:abla_app_4_b}
    }
    \subfigure[AUPR Curves]{
    \includegraphics[scale=0.18]{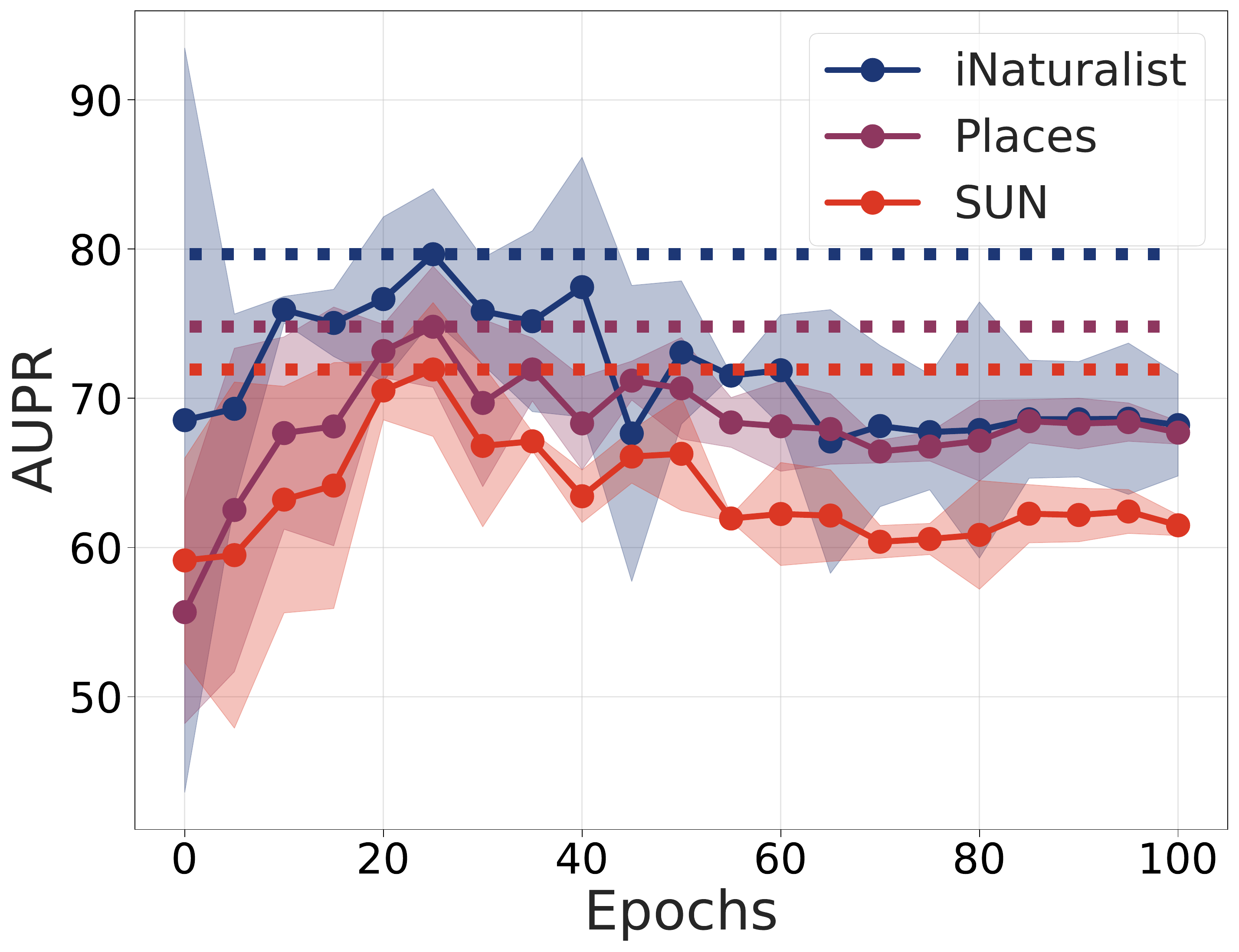}
    \label{fig10:abla_app_4_c}
    }
    \end{center}
    \caption{Ablation studies on three metrics of WRN-40-4 with CIFAR-100 as ID dataset, iNaturalist, Places365, and SUN as OOD datasets. (a) change of FPR95 throughout the pruning phase when training on CIFAR-100; (b) change of AUROC throughout the pruning phase when training on CIFAR-100; (c) change of AUPR throughout the pruning phase when training on CIFAR-100. It demonstrates a better middle stage exists according to the three metrics.}
    \label{fig10:abla_app_4}
\end{figure}

\subsection{Ablation on UMAP which Adopting Pruning on UM.} 
\label{app:abl_UMAP}

We conduct various experiments to see whether pruning has an impact on Unleashing Mask itself. To be specific, we expect the pruning to learn a mask on the given model while not impairing the excellent OOD performance that UM brings. In Figure~\ref{fig11:abla_app_5}, it presents that pruning from a wide range (e.g. $p \in [0.3, 0.9]$) can well maintain the effectiveness of UM while possessing a terrific convergence trend. For simplicity, we use prune to indicate the original pruning approach and UMAP indicate UM with pruning on the mask with our newly designed forgetting objective in Figure~\ref{fig11:abla_app_5}. In Figure~\ref{fig11:abla_app_5_a}, the solid lines represent the proposed UMAP and the dashed lines represent only pruning the well-trained model at prune rates $0.2$, $0.5$, and $0.8$. While the model's OOD performance can't be improved (not better than the baseline) through only pruning, using our proposed forgetting objective for the loss constrain can significantly bring out better OOD performance at a wide range of mask rates (e.g. $p \in [0.5, 0.8]$). In Figure~\ref{fig11:abla_app_5_b}, we intuitively reflect the effect of the estimated loss constraint by the initialized mask which redirects the gradients when the loss reaches the value, while the loss will just approach $0$ when pruning only. In Figure~\ref{fig11:abla_app_5_c}, we can see that ID-ACC for both UMAP and Prune can converge to approximately the same high level ($92\%\sim 94\%$), though we can simply remove the learned mask to recover the original ID-ACC.

\begin{figure}[t!]
    \begin{center}
    \subfigure[FPR95 Curves]{
    \includegraphics[scale=0.18]{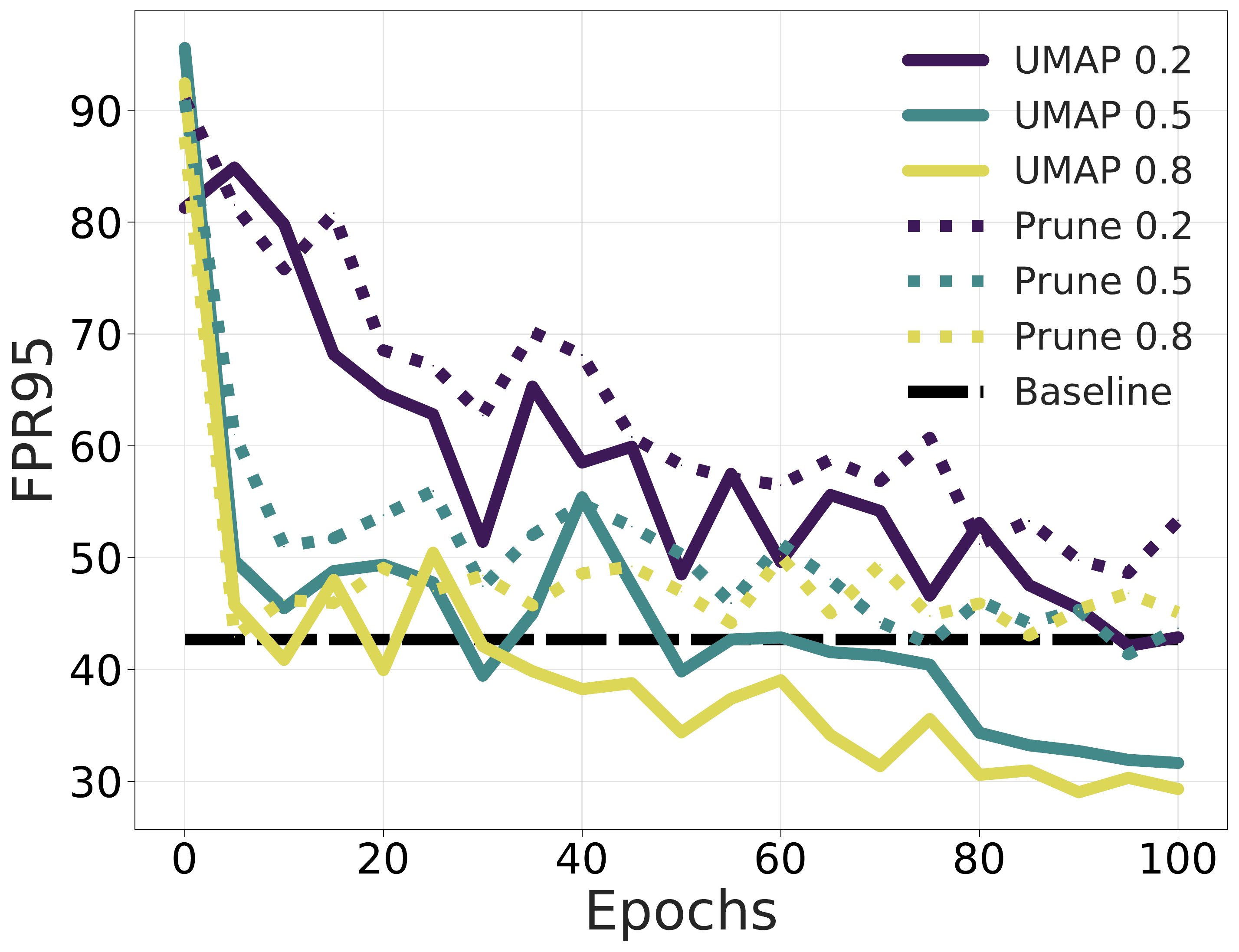}
    \label{fig11:abla_app_5_a}
    }
    \subfigure[Train Loss Curves]{
    \includegraphics[scale=0.18]{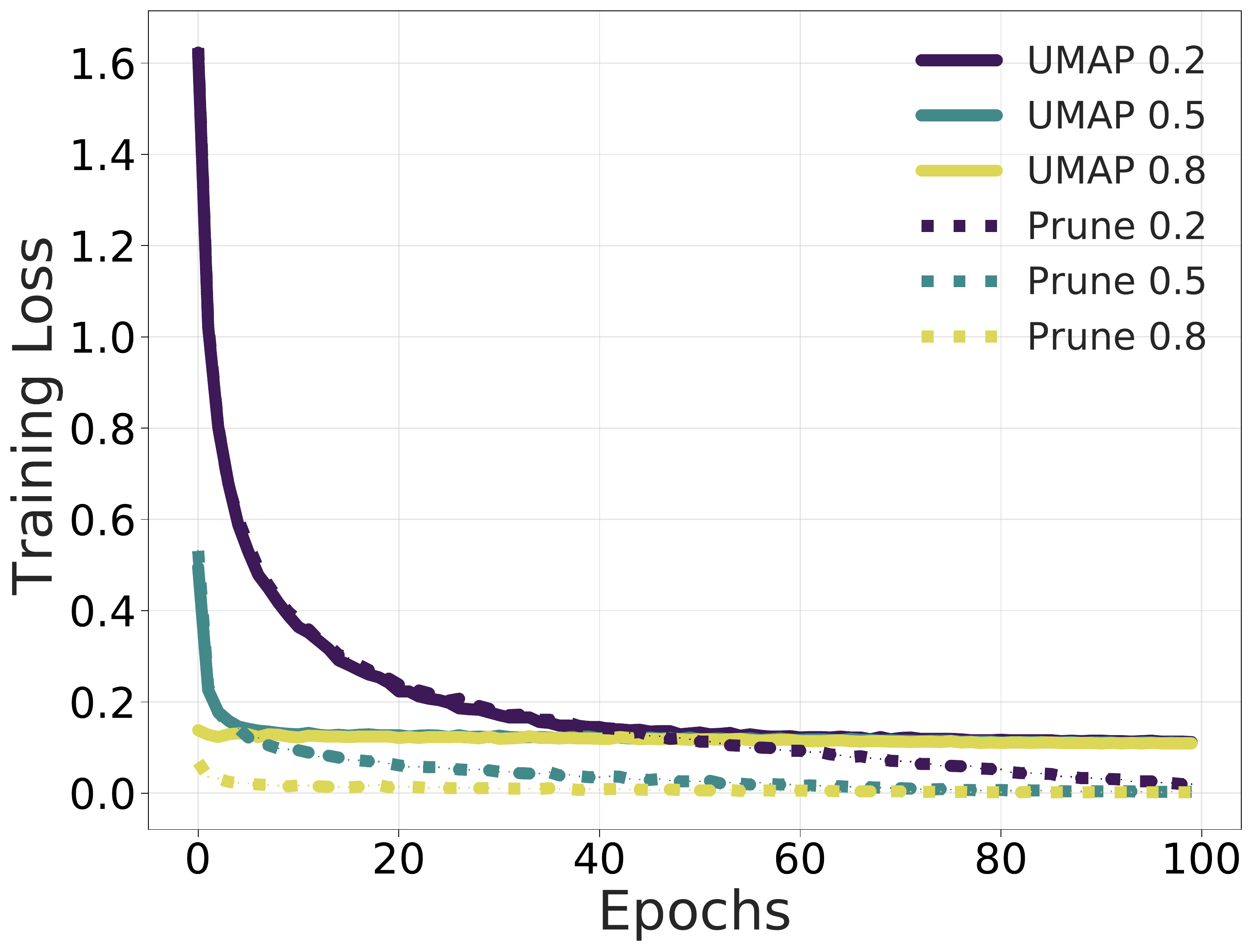}
    \label{fig11:abla_app_5_b}
    }
    \subfigure[Test Acc Curves]{
    \includegraphics[scale=0.18]{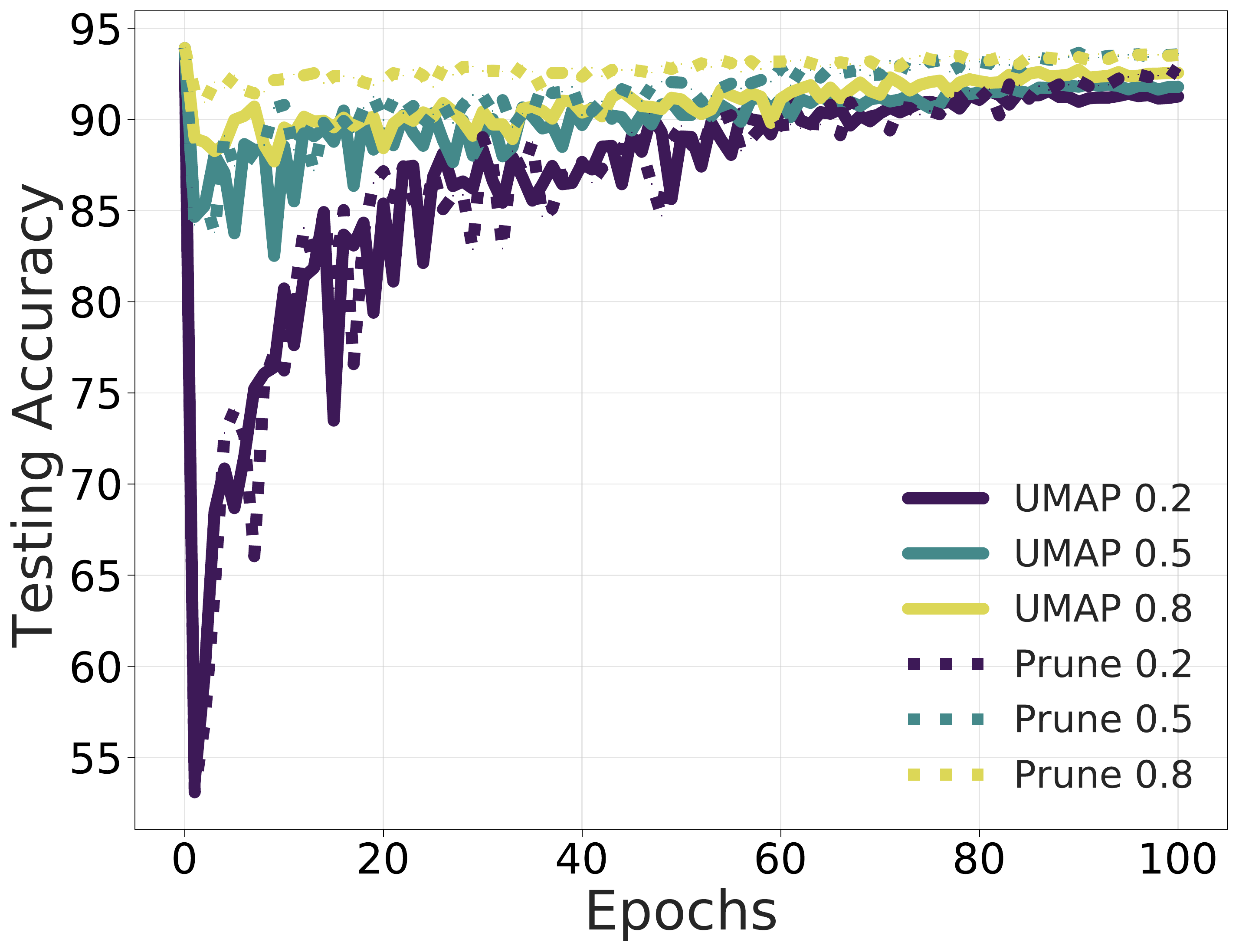}
    \label{fig11:abla_app_5_c}
    }
    \end{center}
    \caption{Ablation studies on Prune Rate of UMAP. (a) change of OOD performance throughout the pruning phase; (b) training loss converges to estimated loss constraint properly; (c) though ID-ACC is not taken into consideration for UMAP, it still rises high after training for 100 epochs.}
    \label{fig11:abla_app_5}
\end{figure}

\begin{figure}[t!]
    \begin{center}
    \subfigure[FPR95 Curves]{
    \includegraphics[scale=0.18]{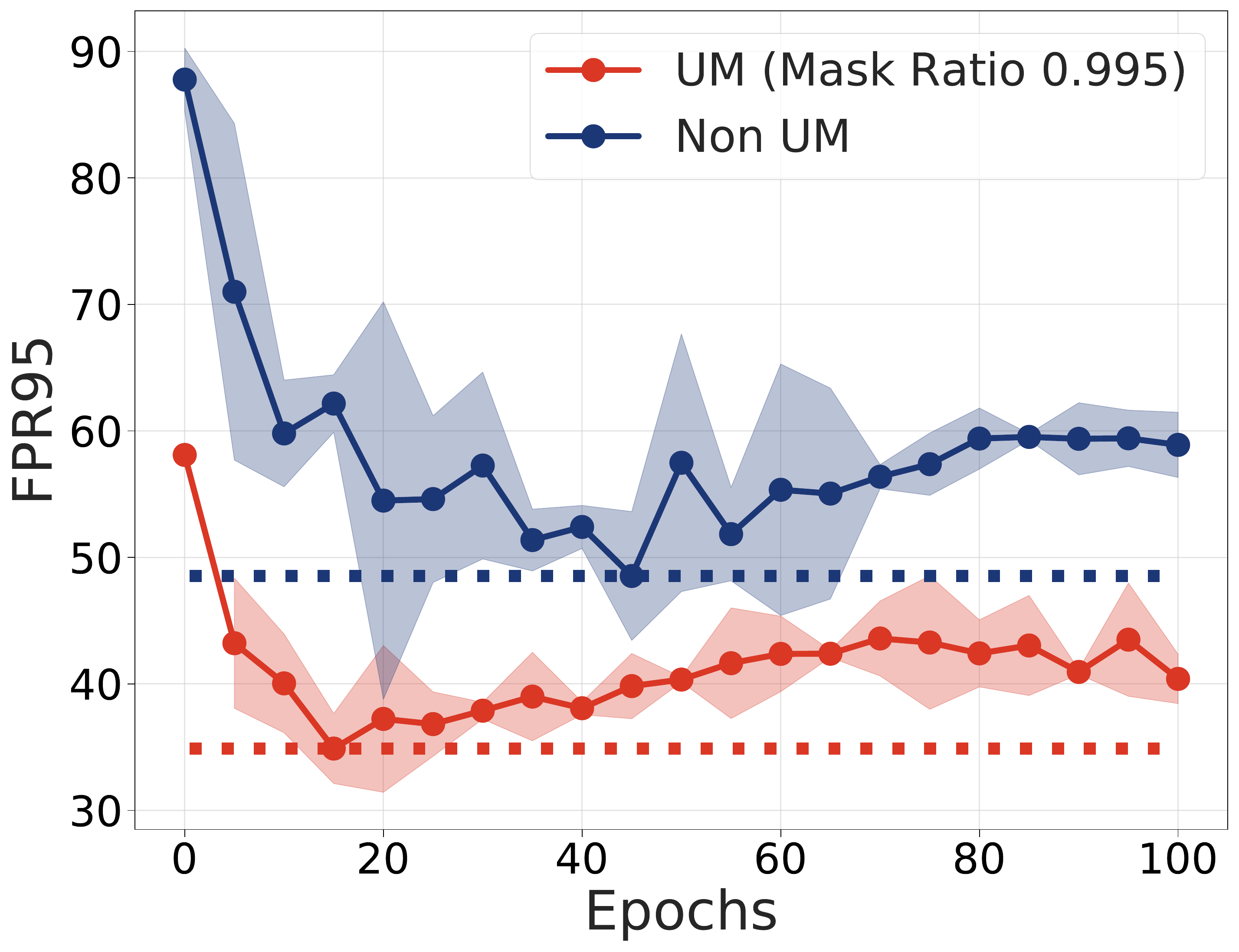}
    \label{fig12:abla_app_6_a}
    }
    \subfigure[AUROC Curves]{
    \includegraphics[scale=0.18]{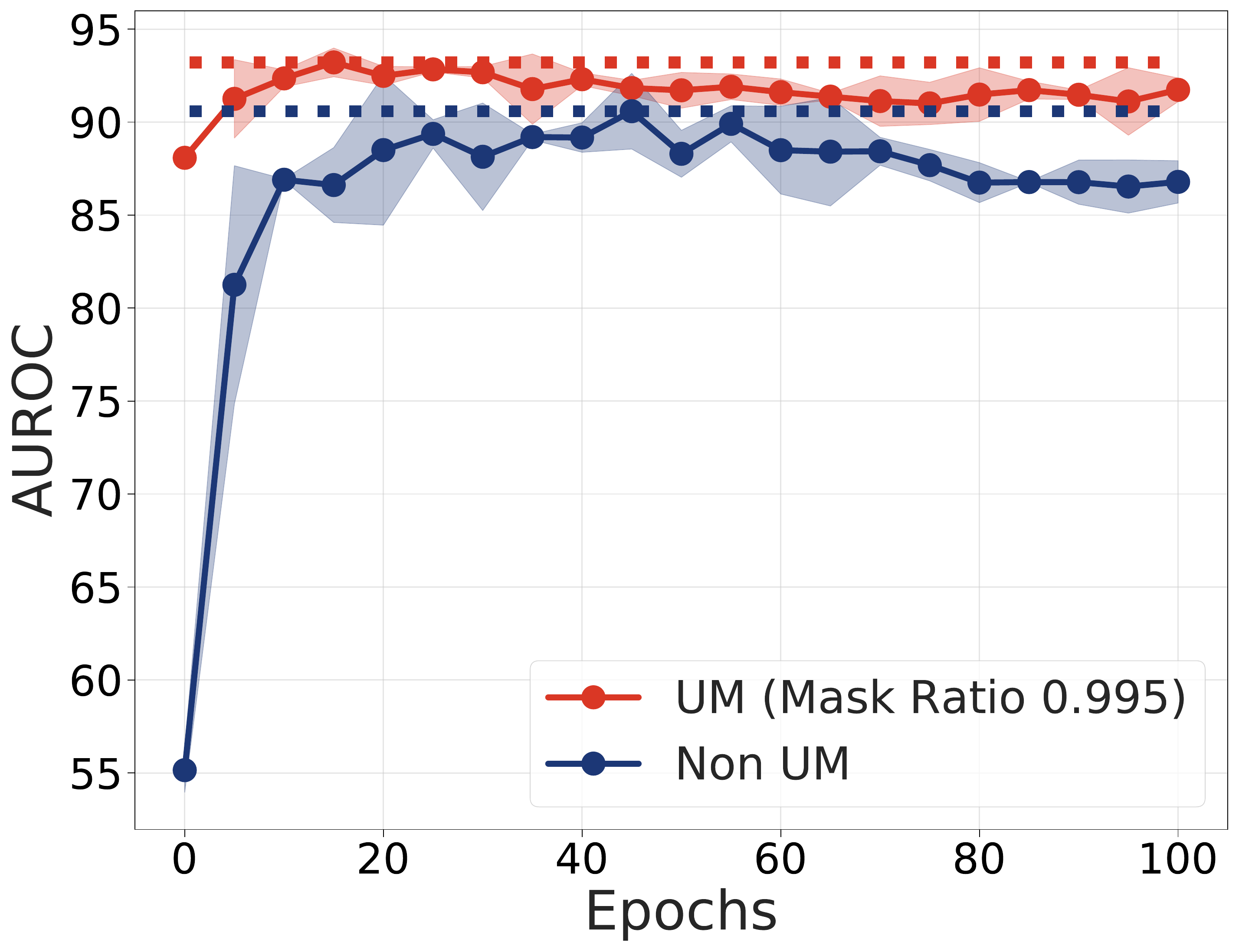}
    \label{fig12:abla_app_6_b}
    }
    \subfigure[AUPR Curves]{
    \includegraphics[scale=0.18]{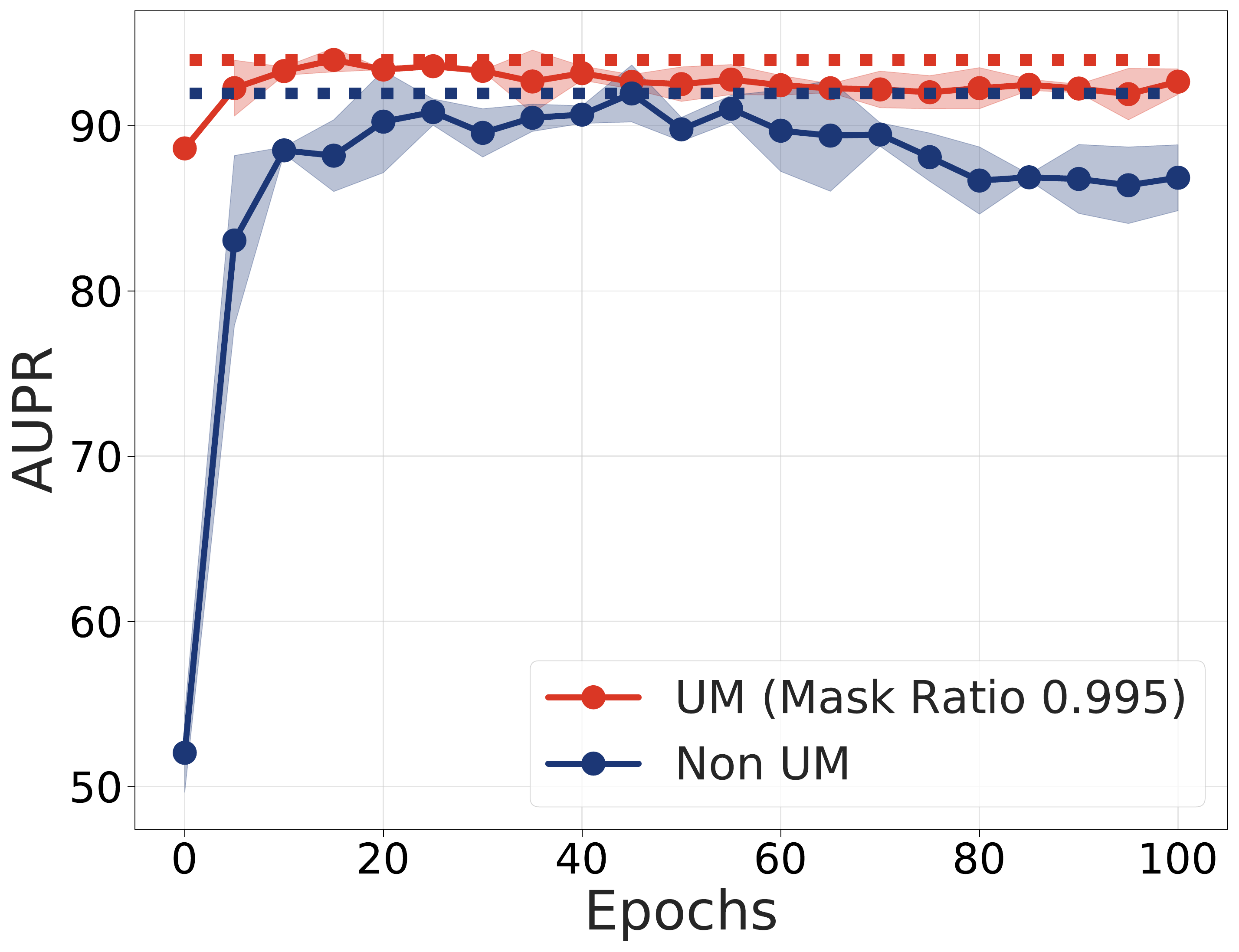}
    \label{fig12:abla_app_6_c}
    }
    \end{center}
    \caption{Ablation studies to reflect the effectiveness of UM. The mask ratio of UM is $99.5\%$. (a) change of FPR95 throughout the training phase on CIFAR-10; (b) change of AUROC throughout the training phase  on CIFAR-10; (c) change of AUPR throughout the training phase on CIFAR-10.}
    \label{fig12:abla_app_6}
\end{figure}

\begin{figure}[t!]
    \begin{center}
    \subfigure[FPR95 Curves]{
    \includegraphics[scale=0.18]{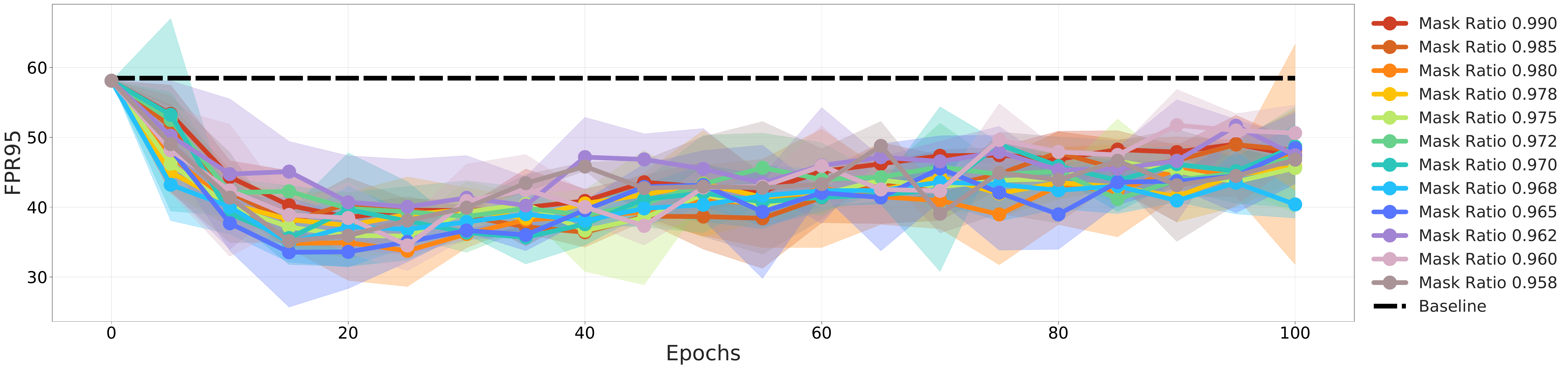}
    \label{fig13:abla_app_7_a}
    }
    \subfigure[AUROC Curves]{
    \includegraphics[scale=0.18]{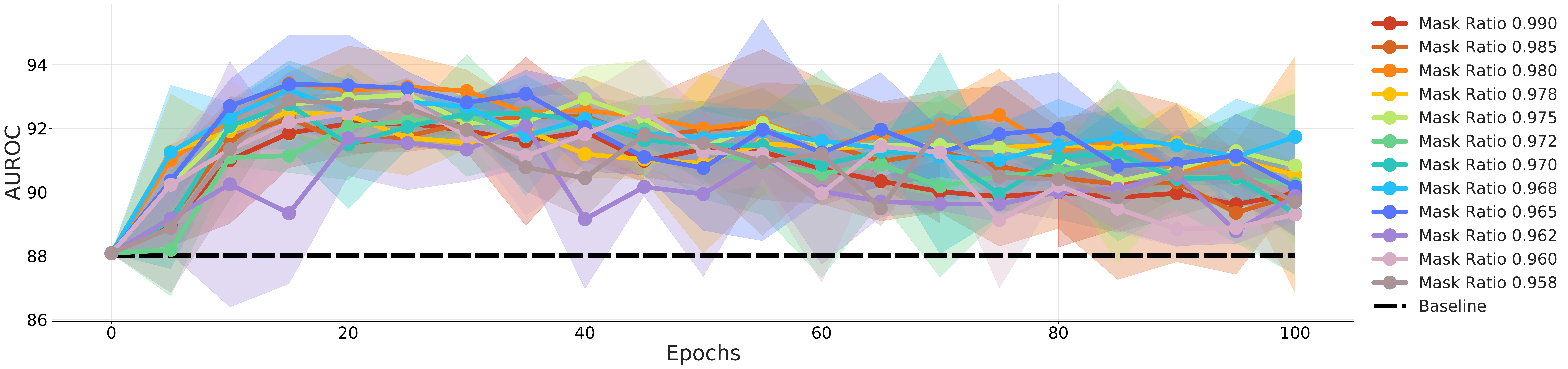}
    \label{fig13:abla_app_7_b}
    }
    \subfigure[AUPR Curves]{
    \includegraphics[scale=0.18]{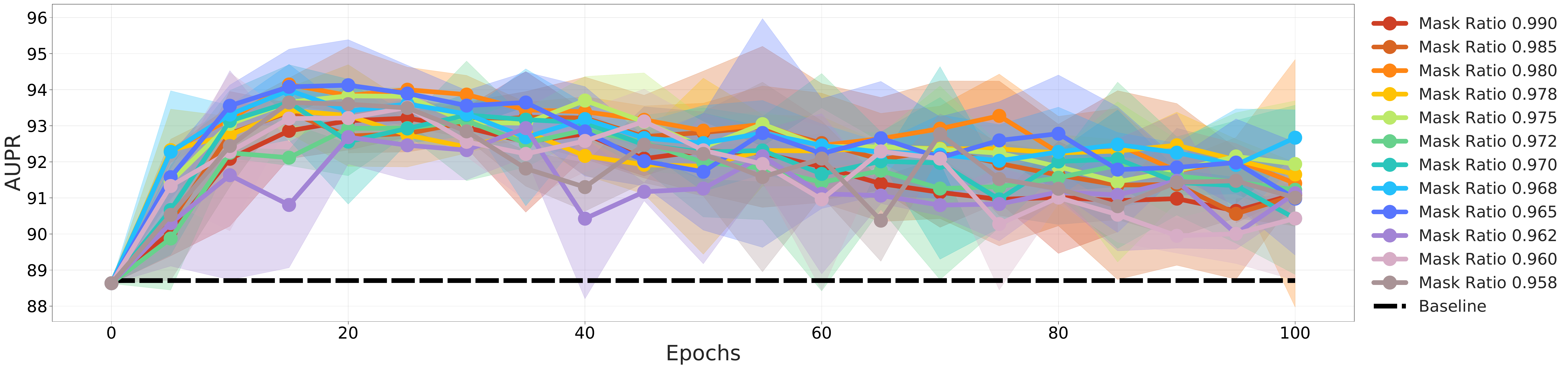}
    \label{fig13:abla_app_7_c}
    }
    \end{center}
    \caption{Ablation studies of WRN-40-4 on various Mask Ratios. The mask rate is from $95.8\%$ to $99.0\%$. (a) change of FPR95 throughout the training on CIFAR-10; (b) change of AUROC throughout the training on CIFAR-10; (c) change of AUPR throughout the training on CIFAR-10.}
    \label{fig13:abla_app_7}
\end{figure}

\subsection{Fine-grained comparison of model weights.}
\label{app:comp_prune_umap}

We display the weights of the original model, pruned model, and the UMAP model respectively in Figure~\ref{fig14:abla_app_8}. The histograms show that the adopted pruning algorithm tends to choose weights far from $0$ for the first convolution layer, shown in Figure~\ref{fig14:abla_app_8_a}. However, for almost all layers (from the 2nd to the 98th), the pruning chooses weights with no respect to the value of weights, shown in Figure~\ref{fig14:abla_app_8_b}. For the fully connected layer, the pruning algorithm itself still keeps its behavior on the first layer, while UMAP forces the pruning algorithm to choose weights near $0$, shown in Figure~\ref{fig14:abla_app_8_c}, indicating that forgetting learned atypical samples doesn't necessarily correspond to larger weights or smaller weights.

\begin{figure}[h!]
    \begin{center}
    \subfigure[First Layer]{
    \includegraphics[scale=0.18]{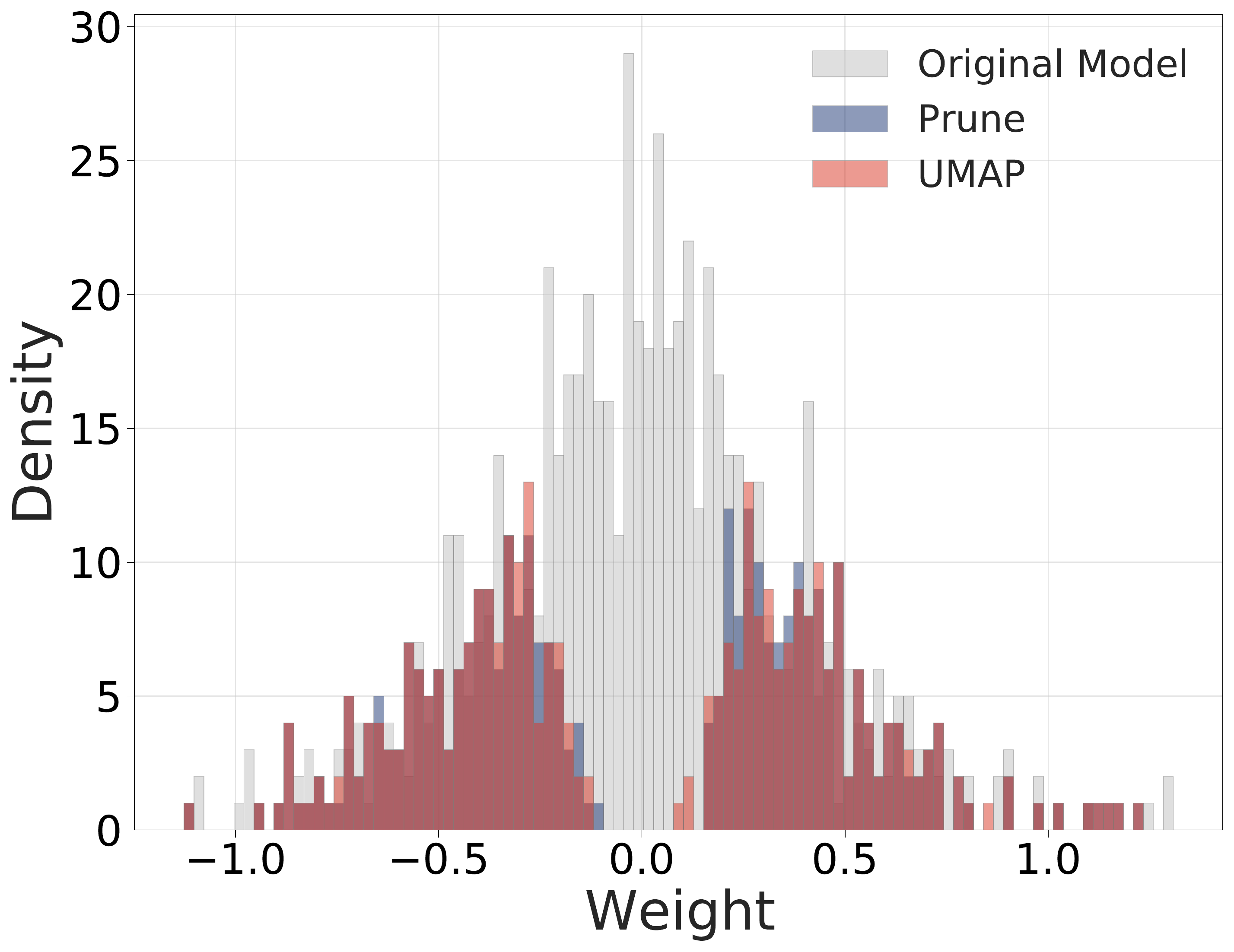}
    \label{fig14:abla_app_8_a}
    }
    \subfigure[50th Layer]{
    \includegraphics[scale=0.18]{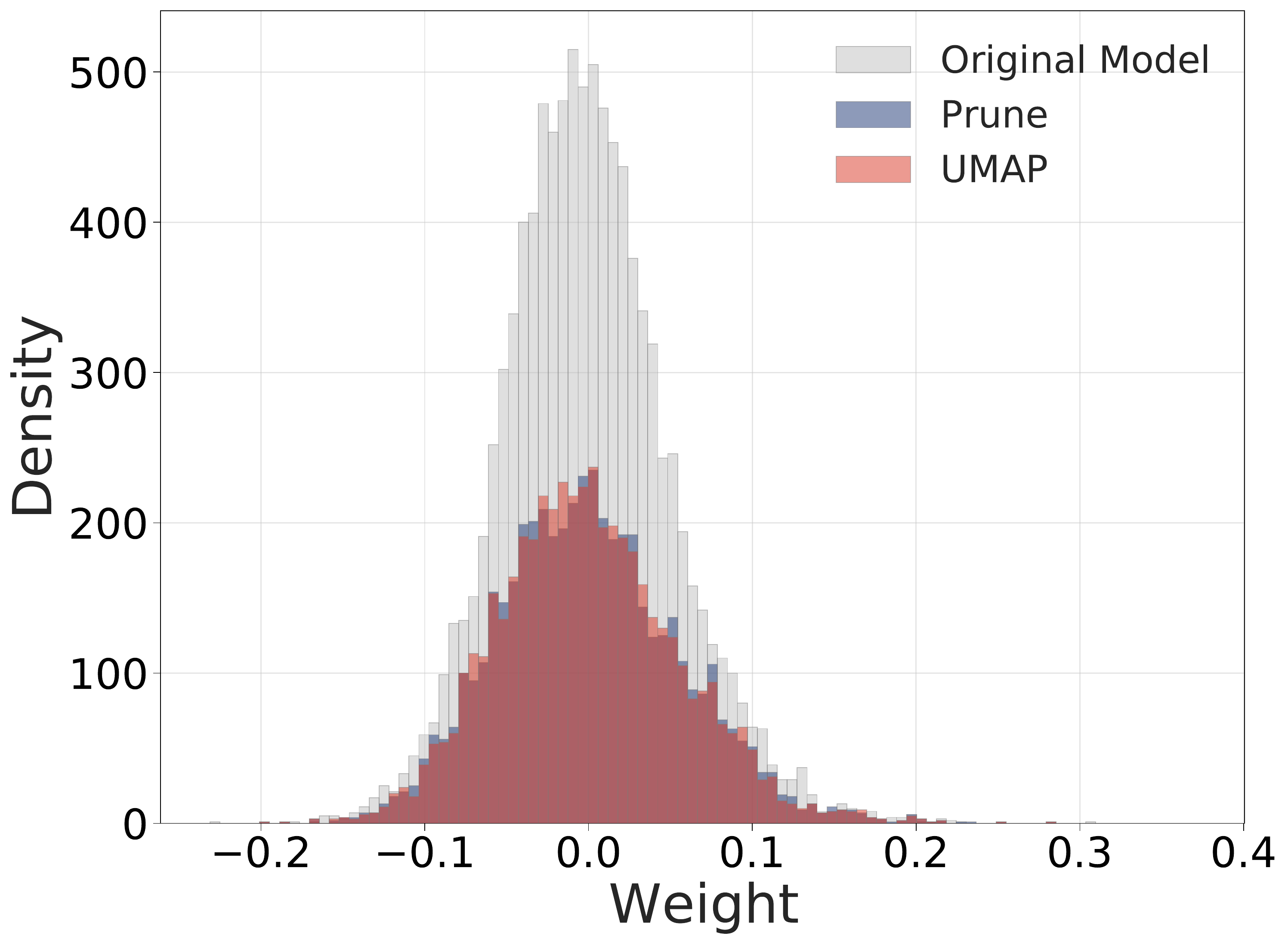}
    \label{fig14:abla_app_8_b}
    }
    \subfigure[Last Layer (FC)]{
    \includegraphics[scale=0.18]{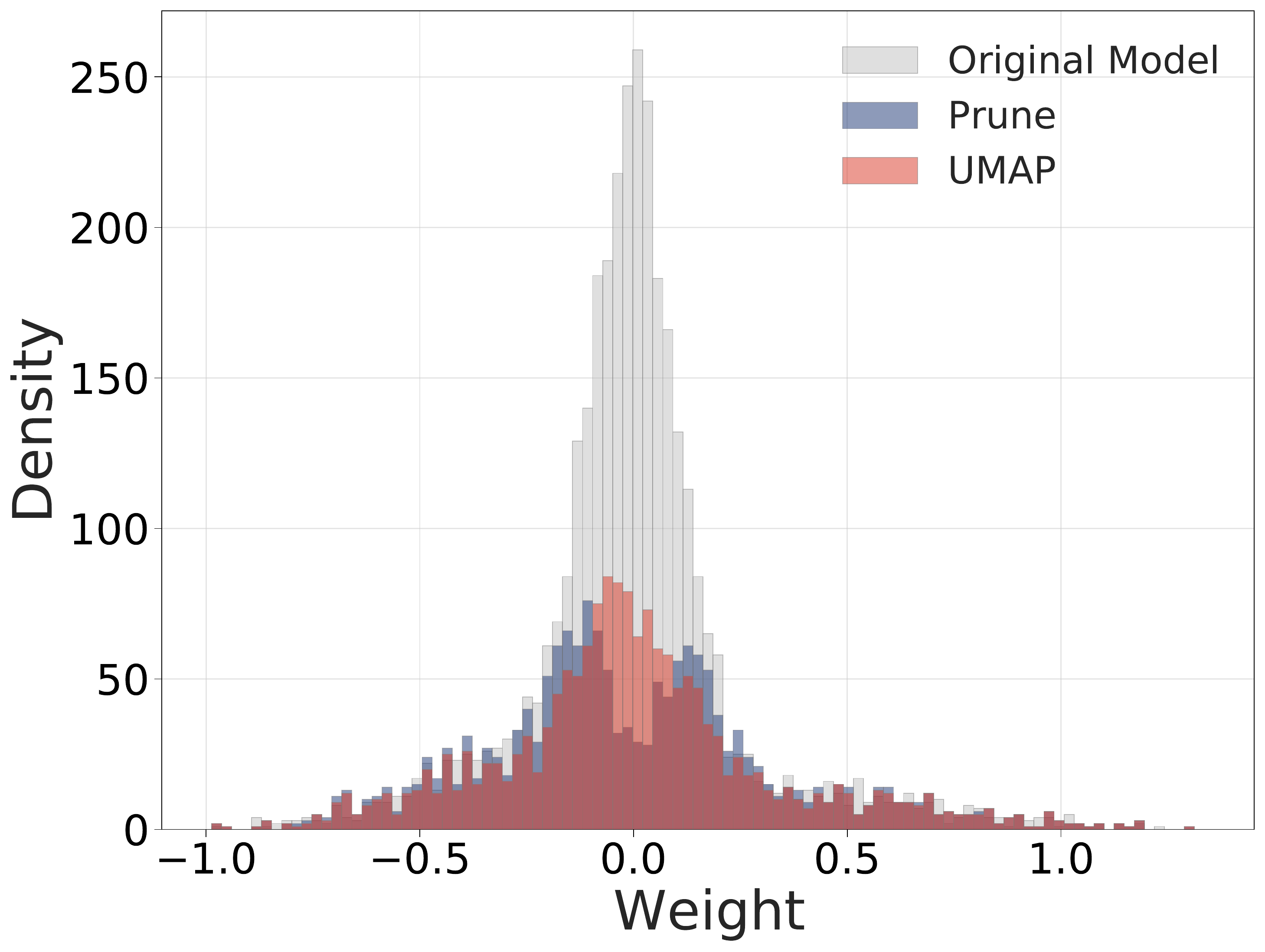}
    \label{fig14:abla_app_8_c}
    }
    \end{center}
    \caption{Histograms of different layers for the original model, pruned model, and UMAP model. The model is DenseNet-101 with a prune rate of $50\%$. (a) the histogram of the first convolution layer; (b) the histogram of the 50th convolution layer; (c) the histogram of the last (fully connected) layer.}
    \label{fig14:abla_app_8}
\end{figure}

\subsection{The effectiveness of UM}
\label{app:eff_um}

In Figure~\ref{fig12:abla_app_6}, we present the FPR95, AUROC, and AUPR curves during training to show the comparison of the original training and our proposed UM on ID data. We observe that training using UM can consistently outperform the vanilla model training, either for the final stage or the middle stage with the best OOD detection performance indicated by the FPR95 curve. In Figure~\ref{fig13:abla_app_7}, we also adopt different mask rates for the initialized loss constraint estimation for forgetting the atypical samples. The results show that a wide range of mask ratios (i.e., from 96\% to 99\%) to estimate the loss constraint used in Eq.~\eqref{eq:obj} can gain better OOD detection performance than the baseline. It shows the mask ratio would be robust to hyper-parameter selection under a certain small value. The principle intuition behind this is our revealed important observation as indicated in Figures \ref{fig1:a}, \ref{fig2:a}, and \ref{fig2:b}. With the guidance of the general mechanism, empirically choosing the hyper-parameter using the validation set is supportable and valuable for excavating better OOD detection capability of the model as conducted by previous literature \citep{hendrycks2018deep,liu2020energy,sun2021react}.

In our experiments, we empirically determine the value of our proposed UM and UMAP by examining the training loss on the masked output. For CIFAR-10 as ID datasets, the value of the mask ratio is $97.5\%$, and the estimated loss constraint for forgetting is $0.10$ for our tuning until the convergence; For CIFAR-100, the value of the mask ratio is $97\%$, and the estimated loss constraint for forgetting is $1.20$ for our tuning until the convergence. To choose the parameters of the estimated loss constraint, we use the TinyImageNet \citep{DBLP:booktitles/corr/abs-2007-06712} dataset as the validation set, which is not seen during training and is not considered in our evaluation of OOD detection performance. Since the core intuition behind our method is to restore the OOD detection performance starting from the well-trained model stage, forgetting a relatively small portion (empirically found around 97\% mask ratio) of atypical samples can be beneficial for the two common benchmarked datasets. In addition, we also verify the effectiveness of UM and UMAP considering the large-scale ImageNet as ID dataset in Table~\ref{tab:my_imagenet} and Appendix~\ref{app:algo_realization}, the loss constraint for forgetting can be $0.6$ which is estimated using the mask ratio as $99.6\%$. To find the optimal parameter for tuning, more advanced searching techniques like AutoML or validation design based on the important observation in our work may be further employed in the future. For the safety concerns, it is affordable and reasonable to gain significant OOD detection performance improvement by investing extra computing resources.

\section{Summarization of the proposed UM/UMAP's advantages}

Regarding the advantages of the proposed method, we kindly interpret them as follows,
\begin{itemize}
    \item \textbf{Novelty.} The proposed UM/UMAP is the first to emphasize the intrinsic OOD detection capability of a given well-trained model during its training phase, better leveraging what has been learned and drawing new insight into the relationship between the OOD detection and the original classification task. This work also shows that ID data is important for a well-trained model's OOD discriminative capability.
    \item \textbf{Simplicity.} Based on the empirical insights that atypical semantics may impair the OOD detection capability, we introduced the easy-to-adopt forgetting objective to weaken the influence of atypical samples on the OOD detection performance. Besides, to maintain the ID performance, we proposed to learning a mask instead of tuning the model directly. Such a design makes UM/UMAP easy to follow and a good starting point to conduct further adjustments. They build on extensive empirical analysis on the point of how to unleash the optimal OOD detection capacity of one given model. Besides, this work explores an orthogonal perspective to previous methods, which shows the consistent improvement combined with previous methods in a range of experiments.
    \item \textbf{Compatibility \& Effectiveness.} UM/UMAP is orthogonal to other competitive methods and can be flexibly combined with them. Extensive experiments demonstrate that UM/UMAP can consistently improve the baselines on average in both benchmarked datasets and large-scale ImageNet (e.g., Tables \ref{tab:my_label},\ref{tab:my_label2},\ref{tab:my_label_complete},\ref{tab:my_imagenet},\ref{tab:my_label3},\ref{tab:my_label4},\ref{tab:label_wrn},\ref{tab:label_wrn_zoom_in_cifar10},\ref{tab:label_wrn_zoom_in_cifar100},\ref{tab:label_zoom_in_SVHN_densenet},\ref{tab:label_zoom_in_SVHN_wrn}; Figures \ref{fig4:a},\ref{fig4:d},\ref{fig4:e},\ref{fig11:abla_app_5},\ref{fig12:abla_app_6},\ref{fig13:abla_app_7},\ref{fig15:um/umap_tsne},\ref{fig16:tsne_UM},\ref{fig16:tsne_UMAP}.
\end{itemize}


\begin{figure*}[t!]
    \begin{center}
    
    \subfigure[UM]{
    \includegraphics[scale=0.24]{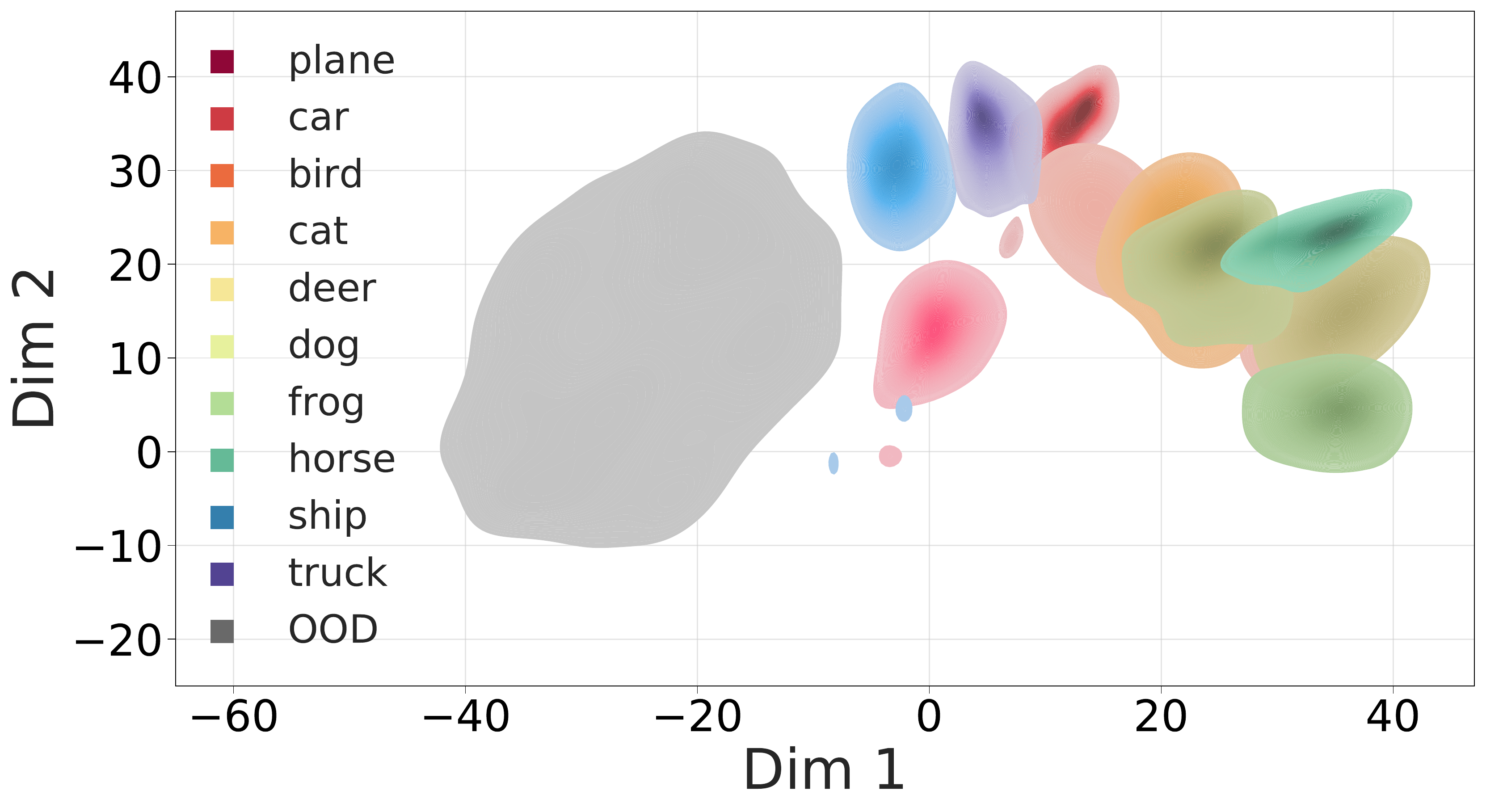}
    \label{fig15:umtsne}
    }
    \subfigure[UMAP]{
    \includegraphics[scale=0.24]{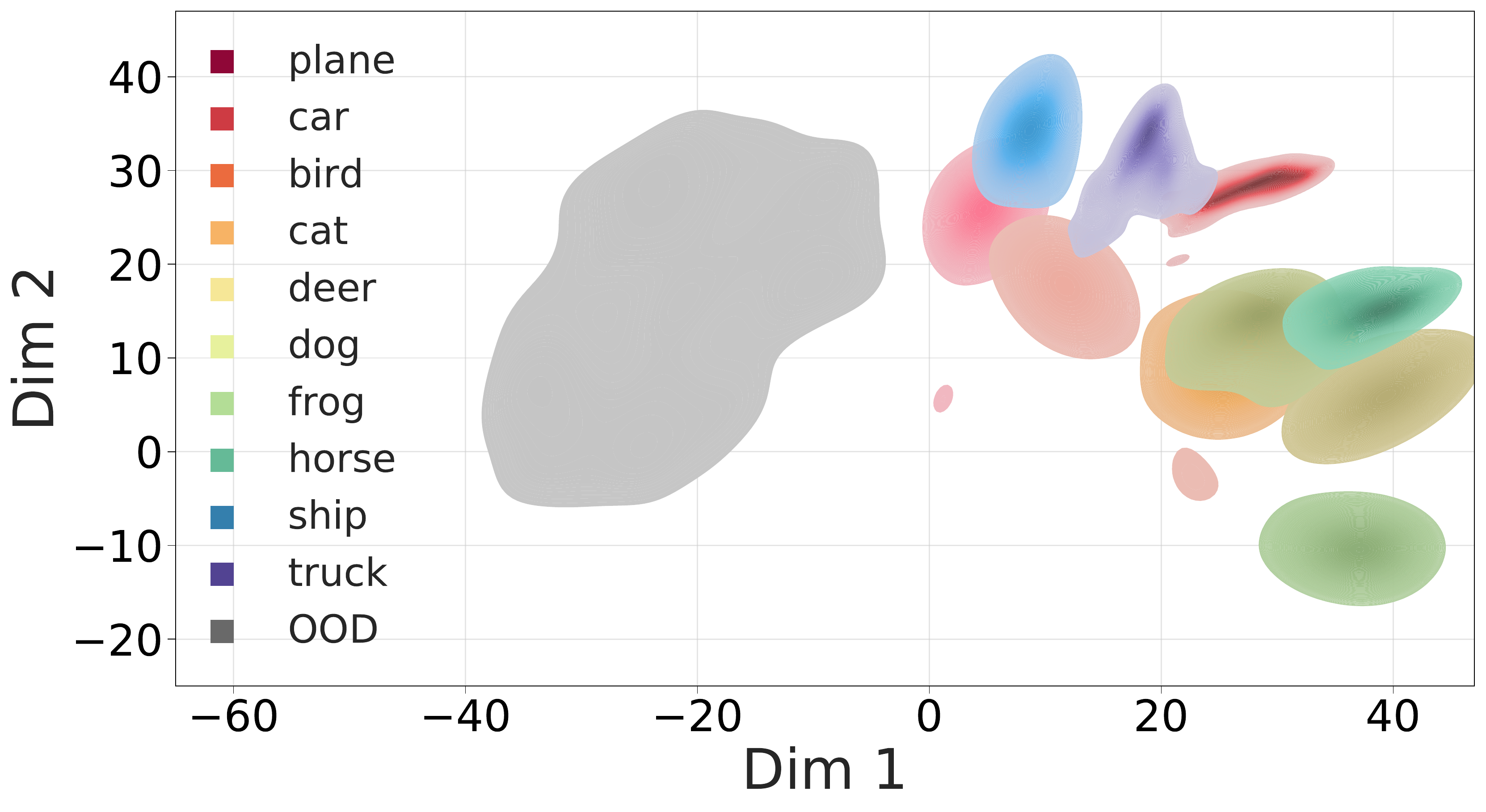}
    \label{fig15:umaptsne}
    }
    \end{center}
    \vspace{-4mm}
    \caption{Similar to previous Figure~\ref{fig2:e}, we visualize the learned feature by UM and UMAP: (a) TSNE visualization of UM; (b) TSNE visualization of UMAP. Compared to the visualization in Figure~\ref{fig2:e}, it is apparent that ID distribution and OOD distribution have larger intervals, which indicates that the UM/UMAP-processed model can better distinguish ID and OOD samples. Moreover, the OOD distributions are more united compared to those in Figure~\ref{fig2:e}, indicating better OOD discriminative Capability.
    }
    \label{fig15:um/umap_tsne}
    \vspace{-2mm}
\end{figure*}

\begin{figure*}[t!]
    \begin{center}
    
    \subfigure[Epoch 60]{
    \includegraphics[scale=0.151]{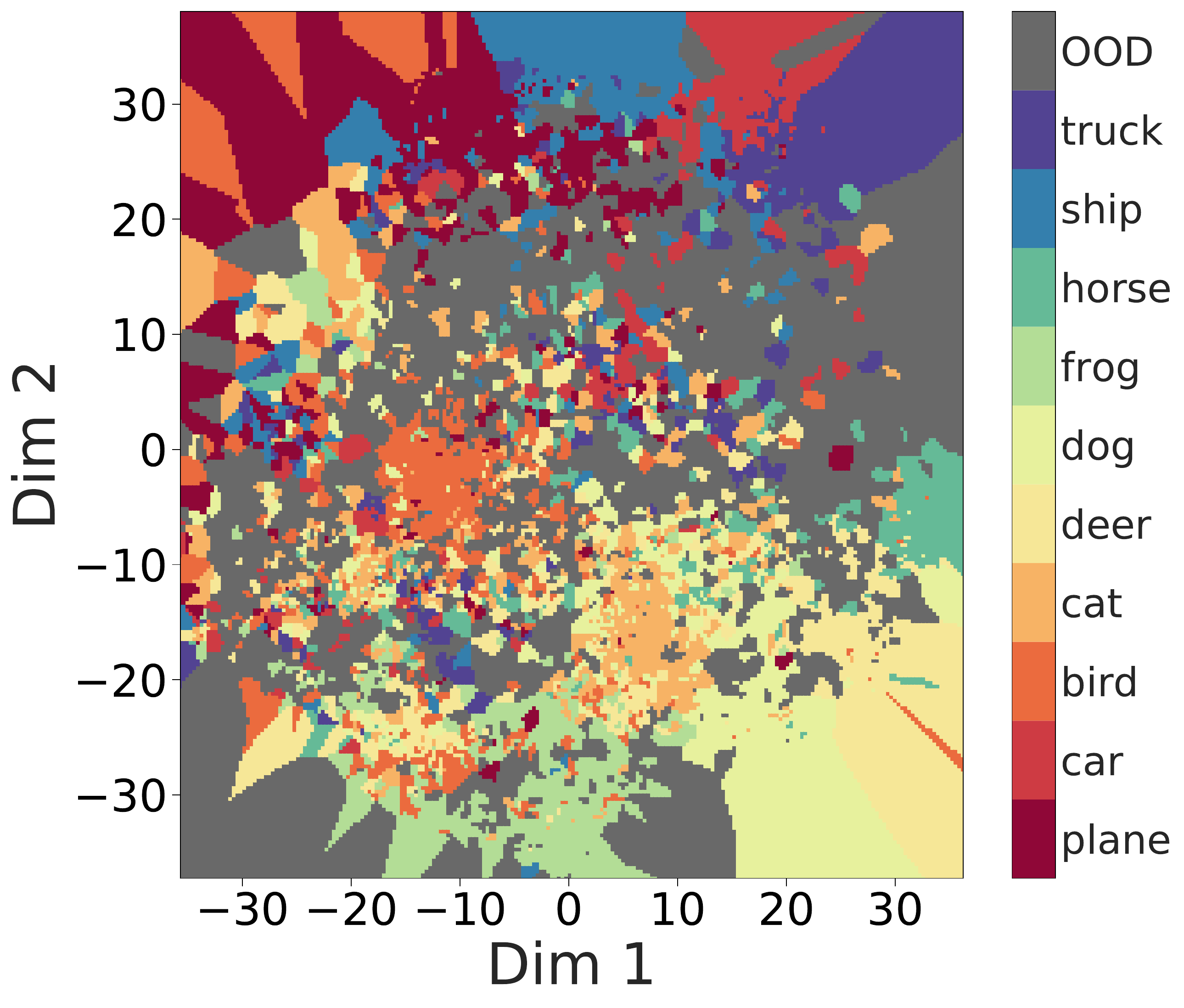}
    \label{fig16:tsne_60}
    }
    \subfigure[Epoch 100]{
    \includegraphics[scale=0.151]{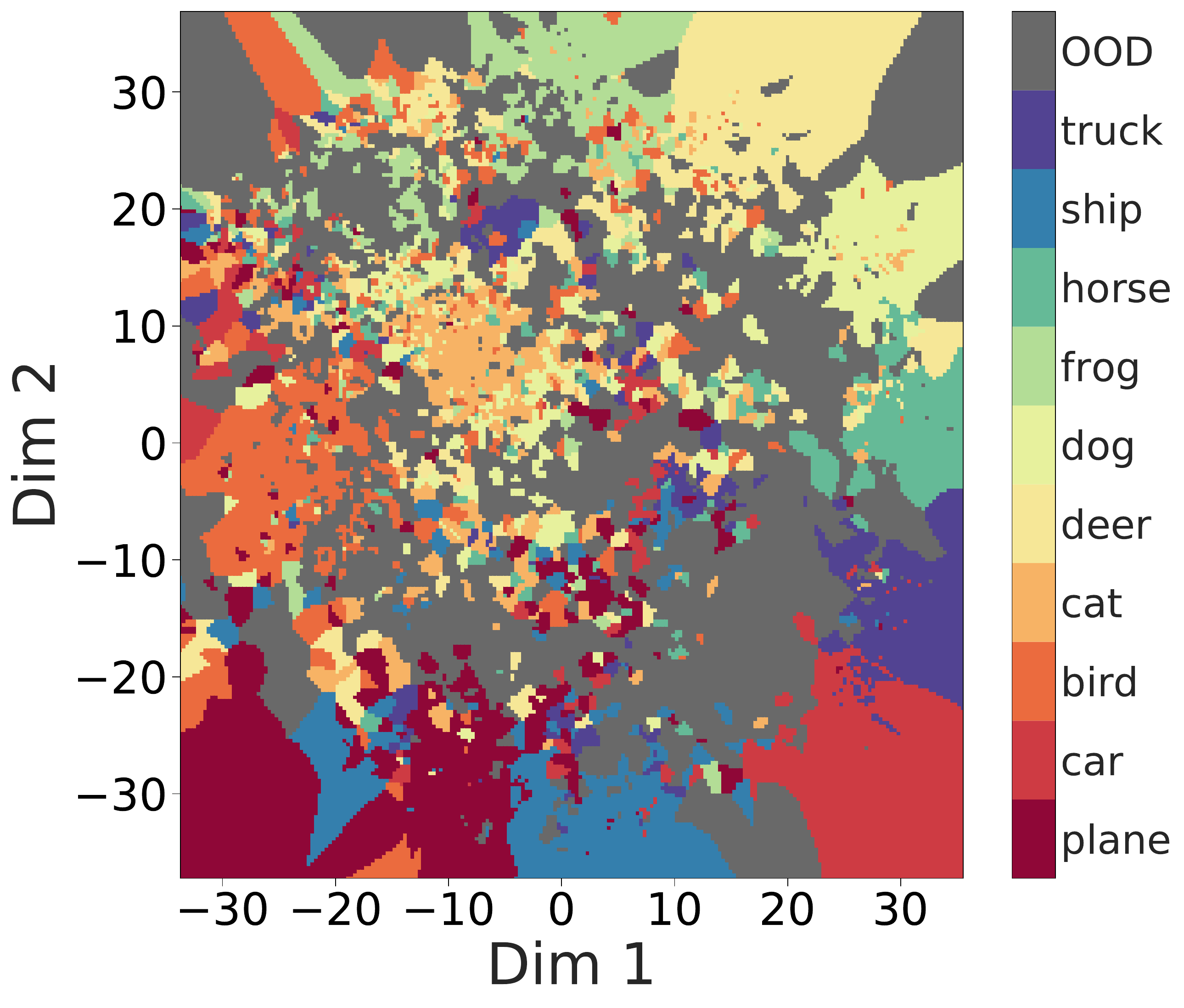}
    \label{fig16:tsne_100}
    }
    \subfigure[UM]{
    \includegraphics[scale=0.151]{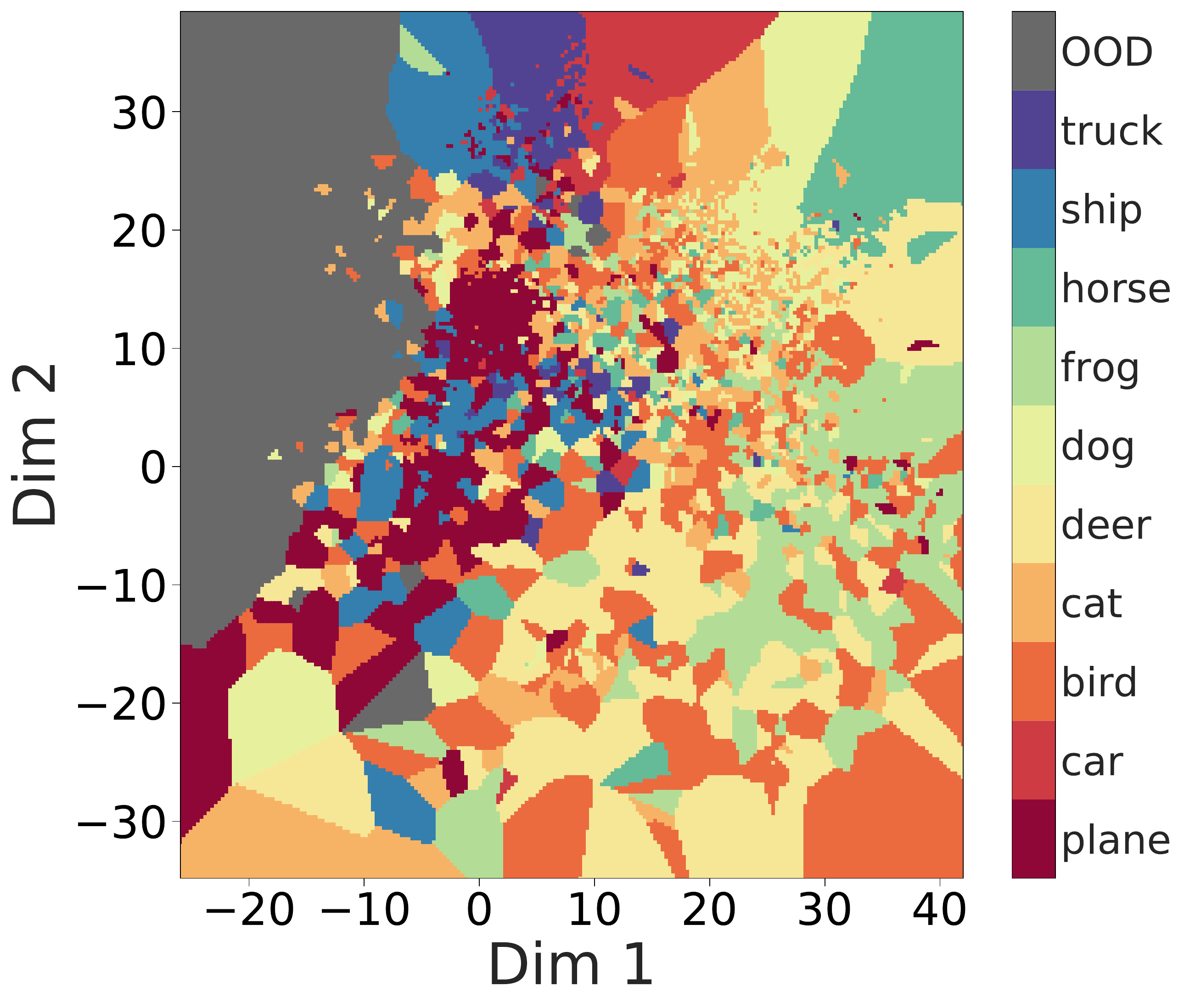}
    \label{fig16:tsne_UM}
    }
    \subfigure[UMAP]{
    \includegraphics[scale=0.151]{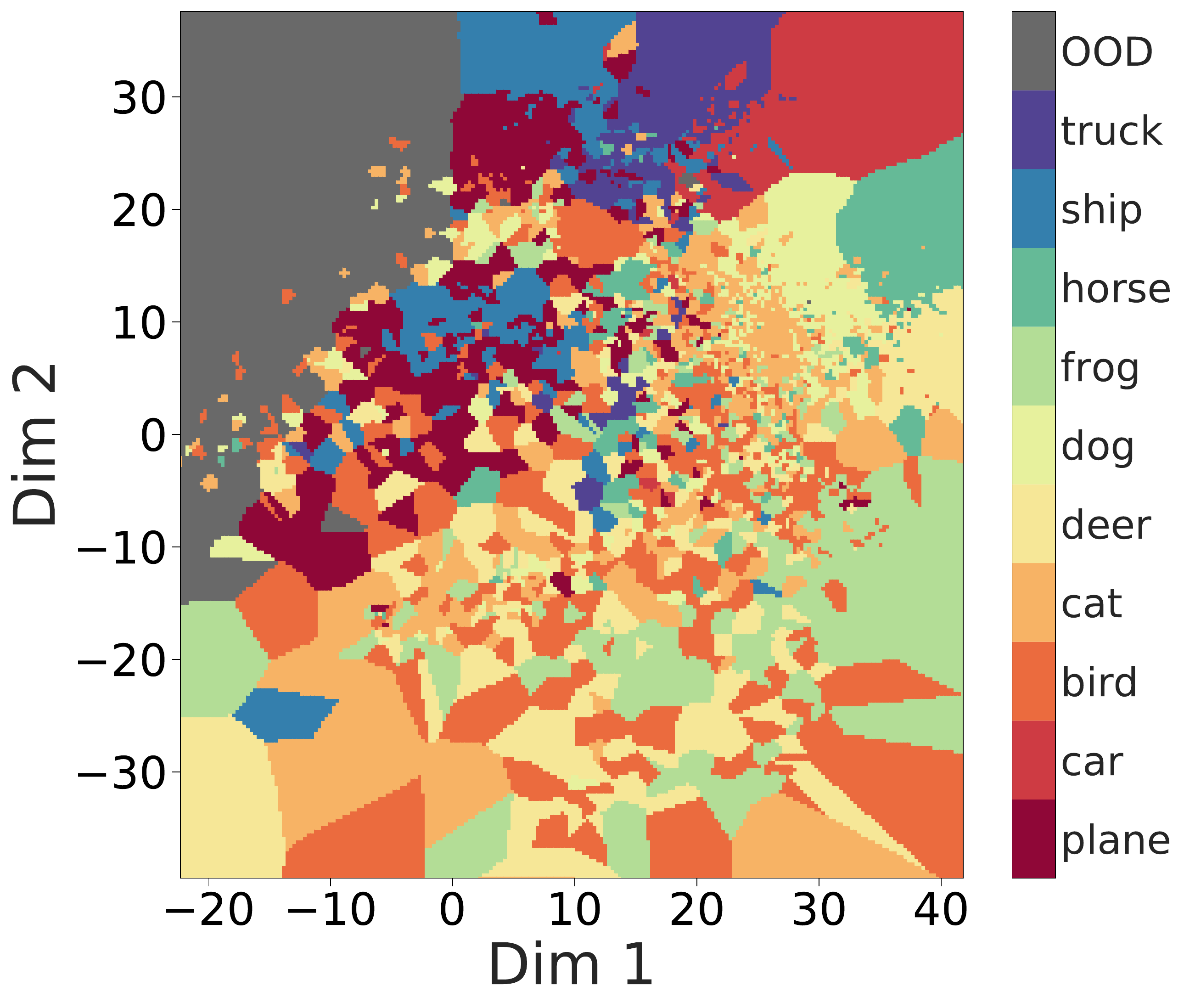}
    \label{fig16:tsne_UMAP}
    }
    \end{center}
    \vspace{-4mm}
    \caption{The visualization of the decision boundary: (a) decision boundary of epoch 60; (b) decision boundary of epoch 100; (c) decision boundary of UM; (d) decision boundary of UMAP. Based on Figures~\ref{fig2:e} and ~\ref{fig15:um/umap_tsne}, we further use the TSNE embeddings to simulate the decision boundaries generated with K-NN \citep{fix1989discriminatory}. The visualization shows that UM/UMAP can help exclude the OOD distribution from the ID distribution, making it easier for the model to distinguish ID/OOD data. While the OOD performance at epoch 60 is better than that at epoch 100, the OOD distribution gets mixed up with the ID distribution in both figures. Though the simulation can't act as a concrete reflection of the model's feature space, it still can intuitively explain the mechanism of UM/UMAP.
    }
    \label{fig16:decision_boundary}
    \vspace{-2mm}
\end{figure*}

\begin{figure}[t!]
    \begin{center}
    \subfigure[Layer-wise masking scores]{
    \includegraphics[scale=0.094]{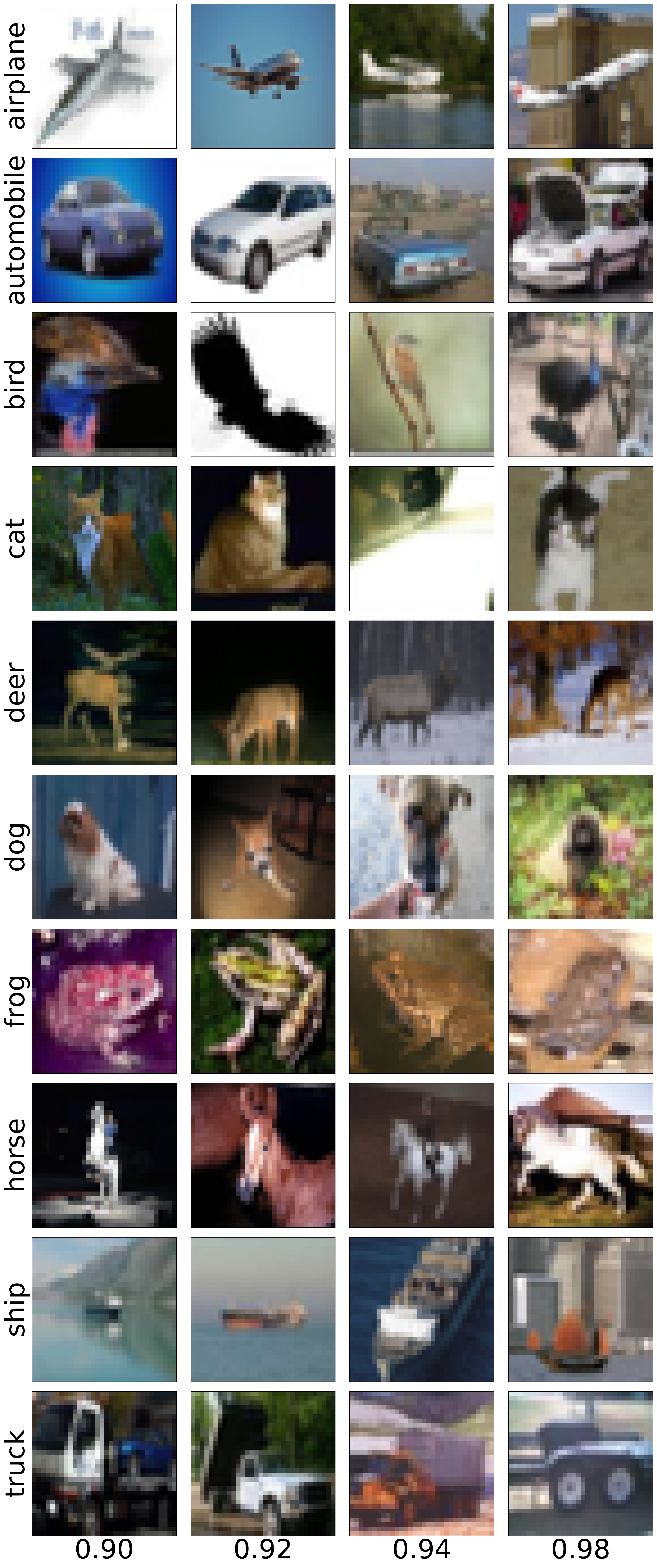}
    \label{fig17:mine_atypical_a}
    }
    \subfigure[Layer-wise masking weights]{
    \includegraphics[scale=0.094]{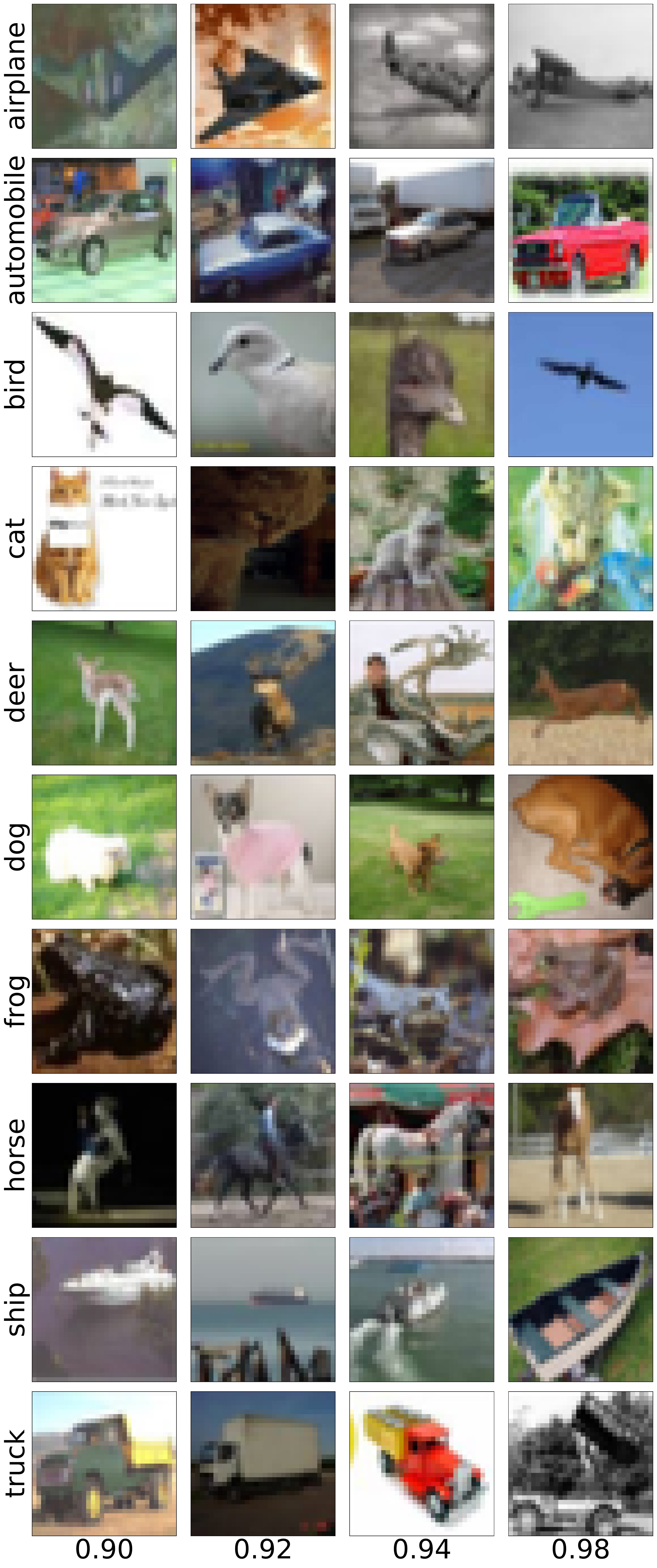}
    \label{fig17:mine_atypical_b}
    }
    \subfigure[Model-wise masking scores]{
    \includegraphics[scale=0.094]{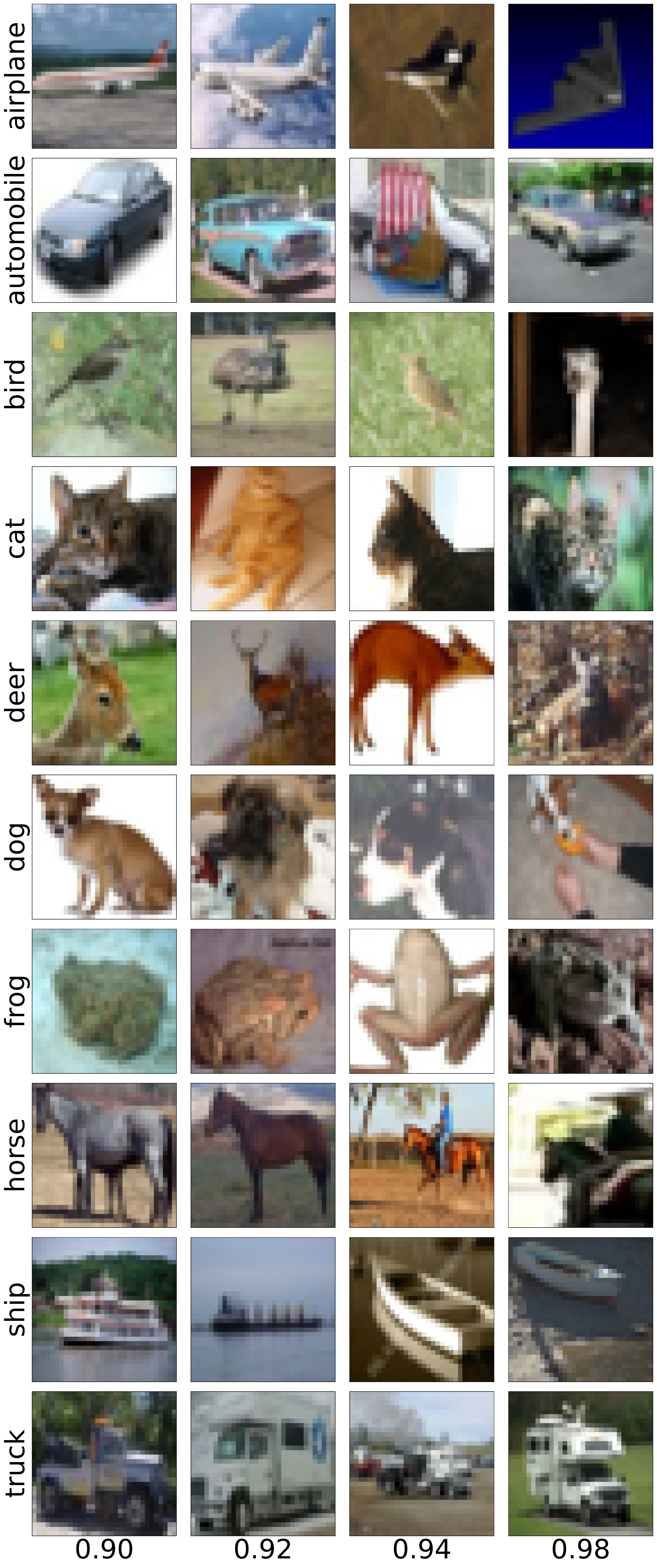}
    \label{fig17:mine_atypical_c}
    }
    \subfigure[Model-wise masking weights]{
    \includegraphics[scale=0.094]{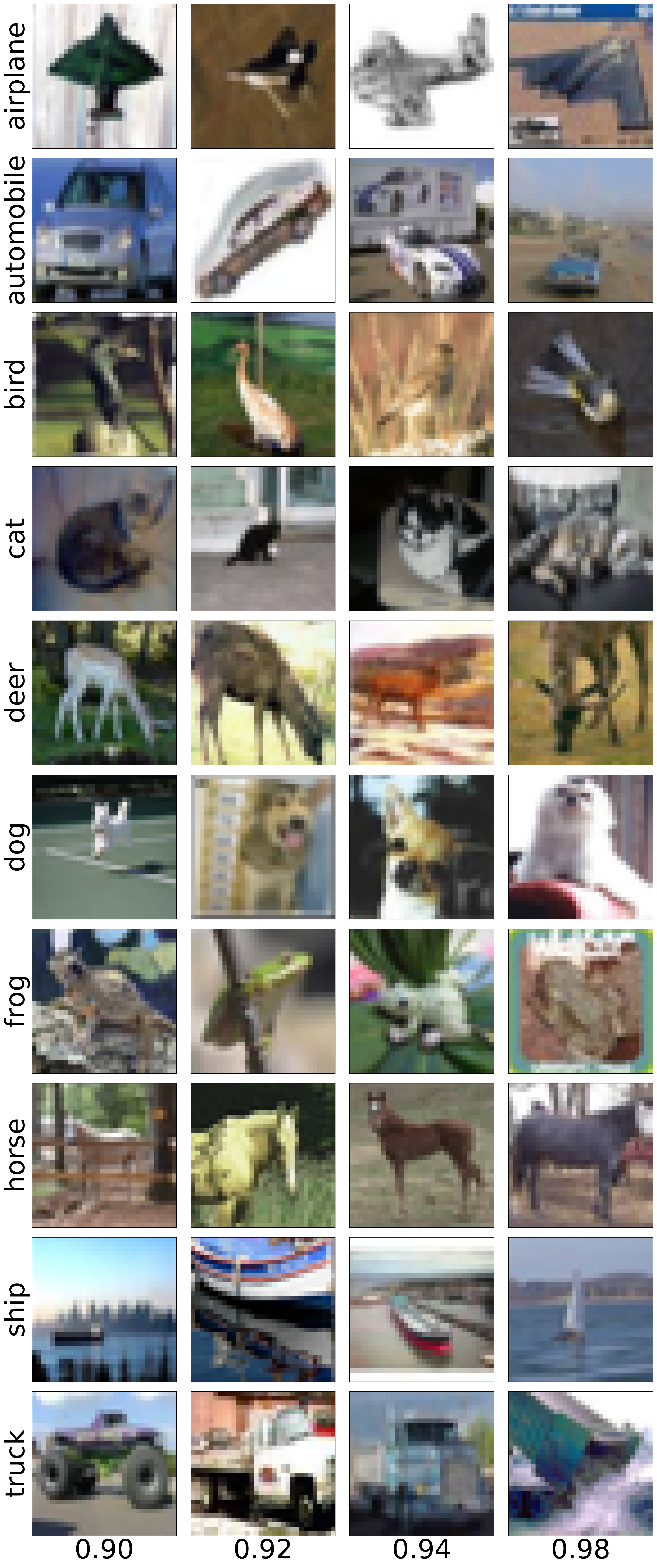}
    \label{fig17:mine_atypical_d}
    }
    \end{center}
    \caption{Examples of misclassified samples after masking the original well-trained model on CIFAR-10. The scores are estimated according to the uniform distribution so that weights are masked out randomly. (a) layer-wise masking scores \citep{ramanujan2020s}; (b) layer-wise masking weights directly; (c) model-wise masking scores; (d) model-wise masking weights directly. For layer-wise masking, a fixed mask ratio is set for every layer of the model, meaning that the same ratio of weights is masked out for every layer; for model-wise masking, the model is considered as a whole with all weights united first and then masked according to the mask ratio. When masking scores, we first generate a score for every weight according to the uniform distribution and then mask out those weights with smaller scores; when masking weights, we directly mask out weights with smaller magnitudes. The images show that masking weights can't efficiently detect atypical samples while masking scores can. For masking scores, masking with a smaller ratio forces the model to misclassify simple samples (clear contours around subjects, single color background) while masking with a larger ratio guide the model to misclassify complex samples (unclear contours, noisy background). This inspection empirically supports our claims that with proper mask ratio, we can detect atypical samples and therefore can force the model to forget them. Besides that, mask weights don't show any particular difference with different mask ratios, and thus it may be inappropriate for mining atypical samples.}
    \label{fig17:mine_atypical}
\end{figure}

\begin{figure}[t!]
    \begin{center}
    \subfigure[Layer-wise masking scores]{
    \includegraphics[scale=0.11]{fig17_largescale_4X15_layer_wise_scores_compressed_compressed.pdf}
    \label{fig18:mine_atypical_a_imagenet}
    }
    \subfigure[Layer-wise masking weights]{
    \includegraphics[scale=0.11]{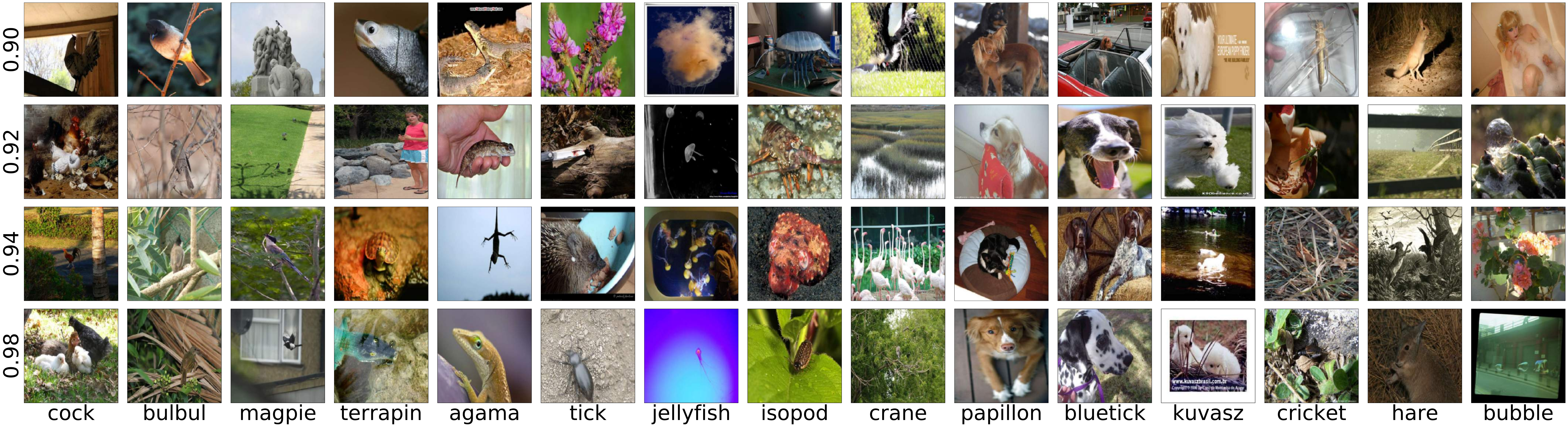}
    \label{fig18:mine_atypical_b_imagenet}
    }
    \subfigure[Model-wise masking scores]{
    \includegraphics[scale=0.11]{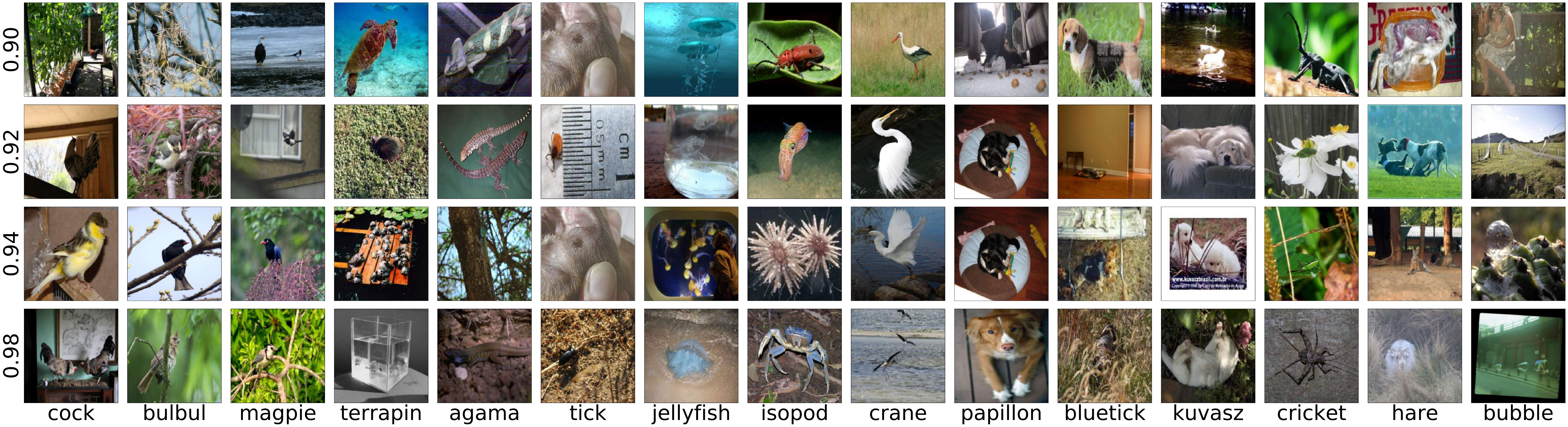}
    \label{fig18:mine_atypical_c_imagenet}
    }
    \subfigure[Model-wise masking weights]{
    \includegraphics[scale=0.11]{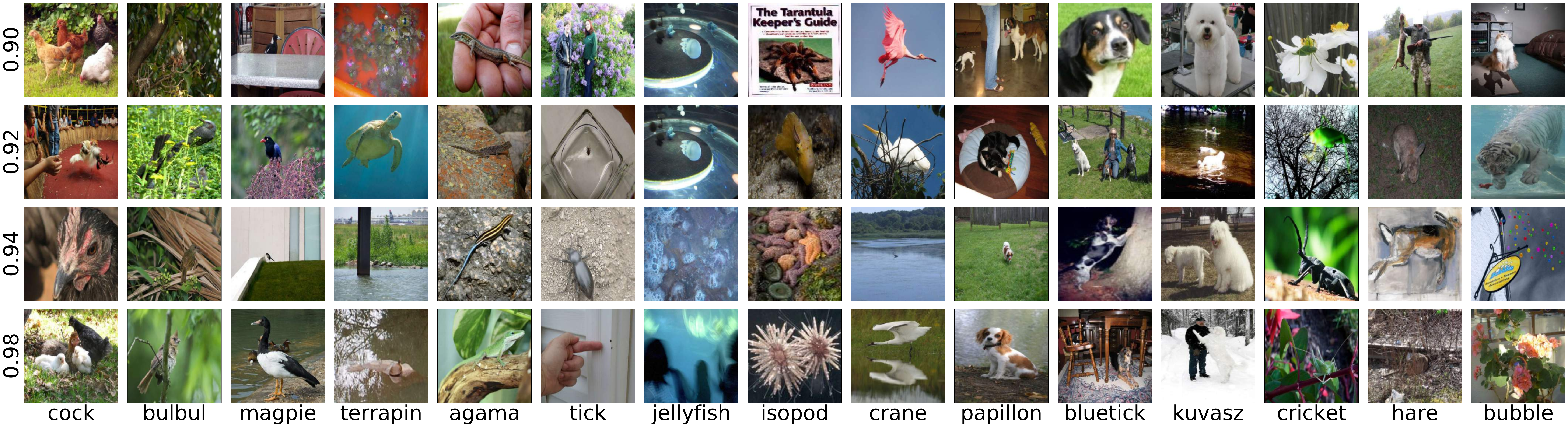}
    \label{fig18:mine_atypical_d_imagenet}
    }
    \end{center}
    \caption{Examples of misclassified samples after masking the original well-trained model on ImageNet.}
    \label{fig18:mine_atypica_imagenetl}
\end{figure}

\end{document}